\documentclass[hidelinks,onefignum,onetabnum]{siamart250211}

\usepackage{lipsum}
\usepackage{amsfonts}
\usepackage{graphicx}
\usepackage{epstopdf}
\usepackage{algorithmic}
\ifpdf
  \DeclareGraphicsExtensions{.eps,.pdf,.png,.jpg}
\else
  \DeclareGraphicsExtensions{.eps}
\fi

\newsiamremark{remark}{Remark}
\newsiamremark{hypothesis}{Hypothesis}
\crefname{hypothesis}{Hypothesis}{Hypotheses}
\newsiamthm{claim}{Claim}
\newsiamremark{fact}{Fact}
\crefname{fact}{Fact}{Facts}

\headers{Generalization Behavior of Deep Residual Networks}{Jinshu Huang, Mingfei Sun, Chunlin Wu}

\title{On Compatible Generalization Bounds of Deep Residual Networks From a Dynamical System Perspective  
\thanks{Submitted to the editors DATE.
\funding{XXX}}}

\author{Jinshu Huang\thanks{School of Mathematical Sciences, Nankai University, Tianjin, China.
		(\email{huangjsh@mail.nankai.edu.cn, sunmf@mail.nankai.edu.cn, wucl@nankai.edu.cn}).}
	\and Mingfei Sun\footnotemark[2]
	\and Chunlin Wu\footnotemark[2] \thanks{Corresponding author.}}

\usepackage{amsopn}

\ifpdf
\hypersetup{
  pdftitle={An Example Article},
  pdfauthor={D. Doe, P. T. Frank, and J. E. Smith}
}
\fi

\usepackage{algorithm}
\usepackage{algorithmic}
\usepackage{xcolor}         
\usepackage{subfigure}
\usepackage{graphicx}
\usepackage{subcaption}
\usepackage{textcomp}
\usepackage{xcolor}
\usepackage{array}
\usepackage{multirow}
\usepackage{longtable}
\usepackage{booktabs}
\usepackage{color}
\usepackage{threeparttable}
\usepackage{geometry}
\usepackage{amssymb}
\usepackage{thmtools}
\usepackage{mathrsfs}
\usepackage{threeparttable}
\usepackage{arydshln}
\usepackage{amsfonts}
\usepackage{array}
\usepackage{ulem}
\usepackage{rotating}
\usepackage{makecell}
\newtheorem{example}{Example}[section]
\allowdisplaybreaks

\begin{document}

\maketitle

\begin{abstract}
Deep neural networks (DNNs) have significantly advanced machine learning, with model depth playing a central role in their successes. Deep residual networks (ResNets) fall into one of the most important classes of DNNs. In this paper, we study the compatible generalization behavior of deep residual networks with a broad class of activation functions in the deep-layer limit, from a dynamical system perspective.
We first show a refined contraction inequality for Rademacher complexity of function classes composed by that type of activation functions. We then conduct an equivalence analysis of component hypothesis sets in Rademacher complexity
and forward convergence analysis of network state variables, which enable us to establish compatible uniform generalization bounds that simultaneously achieve the optimal \(O(1/\sqrt{S})\) sample rate for both discrete- and continuous-time ResNets. 
The bounds also share a common layer-number independent leading term and non-positive refinement terms that capture the decomposition structure of the activation function, revealing a consistent generalization behavior in the deep-layer limit. Compared to previous works considering both discrete- and continuous-time ResNets simultaneously, our theoretical results not only have better sample rate, but also are established under weaker assumptions on the learnable parameters.
We also support our theoretical findings with numerical experiments.
\end{abstract}



\section{Introduction}\label{sec: intro}
Deep learning has achieved remarkable successes across diverse applications in recent years. A key driver behind these successes is the increasing depth of neural networks, which enables the hierarchical extraction of complex features and representations \cite{lecun2015deep, goodfellow2016deep, he2016deep}. 
Despite these empirical advances, the fundamental properties of deep neural networks (DNNs) remain not fully understood due to their inherent complexity.
The expressivity and approximation, generalization, and optimization constitute the essential components of the mathematical mechanisms underlying deep learning \cite{grohs2022mathematical}.
Among them, understanding generalization behavior is one of the most important aspects, as it directly relates to the ability of DNNs to perform well on unseen data beyond the training set with a sample size $S$.

Several prominent frameworks have been developed for the generalization analysis of deep neural networks. These include earlier Vapnik–Chervonenkis (VC) dimension based works \cite{vapnik1971uniform, bartlett1998almost, anthony2009neural}, which characterize generalization via the capacity of the hypothesis class; Rademacher complexity based estimations \cite{neyshabur2015norm, bartlett2017spectrally, truong2022rademacher, shultzman2023generalization}, which often provide data-dependent bounds by measuring the ability of the hypothesis class to fit random noise (or random sign assignments); and the probably approximately correct (PAC) Bayesian framework \cite{mcallester1999pac, langford2002pac, neyshabur2018pac}, which bounds the generalization gap by evaluating the divergence between a chosen prior and a learned posterior distribution over hypotheses, rather than focusing on a single deterministic predictor. Alternative research lines, such as algorithmic stability, further investigate how specific training algorithms influence architecture generalization performance \cite{bousquet2002stability, kuzborskij2018data, sokolic2017generalization, Zhou2025Tighter}.

Deep residual networks \cite{he2016deep} fall in one of the most important types of architectures, and their generalization behavior analysis especially generalization error estimation attracted lots of attention.
Bartlett et al.\cite{bartlett2017spectrally} derived a spectral normalized margin-based bound for DNNs of order $O(1/\sqrt{S})$, which applies to ResNets with ramp risk functions. Li et al.\cite{li2018tighter} derived tighter margin-based bounds for ResNets (and also some other general architectures), explicitly capturing the effects of depth, width, and Lipschitz continuity.
He et al.\cite{he2020resnet} investigated the influence of residual connections on the hypothesis complexity via covering numbers, showing that, when the total number of layers is fixed, residual connections do not increase the hypothesis complexity compared to chain-like networks.
By introducing a weighted path norm, E, Ma, and Wang \cite{e2020rademacher} studied the Rademacher complexity of residual networks with ReLU activation function, and derived generalization bounds of order $O(1/\sqrt{S})$ without explicit dependency on the network depth. These analyses provide insights into why ResNets with wide and multiple layers can still generalize well in practice.

Interestingly, another line for generalization error estimation was developed very recently from the viewpoint of neural ordinary differential equations (ODEs) \cite{chen2018neural} relating to deep residual networks. 
Marion \cite{marion2023generalization} used a covering number reasoning approach to establish a generalization error bound with sample order $O(1/\sqrt[4]{S})$ for a class of continuous-time models with time-varying weights under the Lipschitz continuity assumption, and induced a consistent error bound for ResNet with a scaling factor.
By using similar methods to Marion \cite{marion2023generalization}, Bleistein and Guilloux \cite{bleistein2023generalization} studied the generalization properties of neural controlled differential equations (NCDEs), which model supervised learning from irregularly sampled time series via controlled differential equations. They derived generalization bounds of sample order 
$O(\sqrt{\log(S)}/\sqrt{S})$ under Lipschitz continuity assumption of vector field. More recently, Hanson and Raginsky \cite{hanson2024rademacher} analyzed the generalization properties of neural ODEs through a Chen–Fliess series expansion, which represents continuous-time models as infinitely wide, single-layer networks whose features are given by iterated Lie derivatives. This reformulation enables a compact analysis of the Rademacher complexity of neural ODEs and leads to a generalization bound of order $O(1/\sqrt{S})$.

We note that, despite the above interesting advances, none of them provides compatible generalization error estimations for both discrete-time ResNet and its continuous-time limit simultaneously with matching optimal $O(1/\sqrt{S})$ sample dependency.
Our goal in this work is to study the consistency of the generalization behavior of ResNets in the deep-layer limit, and establish compatible uniform generalization error bounds with optimal $O(1/\sqrt{S})$ sample dependency for both discrete-time ResNet and its continuous-time limit simultaneously.
For this objective, we adopt the recently developed dynamical system modeling strategy \cite{weinan2017proposal, haber2017stable}.
Although in some sense related to neural ODE \cite{chen2018neural}, the term ``dynamical system modeling" refers to a unified modeling and analytical viewpoint that encompasses both discrete-time ResNet and its continuous-time limit, in our opinion. Such a framework offers not only  
theoretical insights like \cite{benning2019deep, zhang2020forward, li2022approximation, li2022deep}, practical architecture design methods \cite{haber2017stable, Lu18Beyond, ruthotto2020deep}, training algorithms \cite{chen2018neural, parpas2019predict, gunther2020layer}, but also asymptotic analysis like the discrete-to-continuum convergence \cite{thorpe2018deep}; see also \cite{ sherstinsky2020fundamentals,huang2024on,huang2024Mathematical,Lin2025DeeplayerLA} for other types of deep neural networks.

In particular, we will study here the generalization behavior of ResNets from a dynamical system perspective \cite{weinan2017proposal, haber2017stable}.
We will consider ResNets with a broad class of activation functions. After firstly generalizing a very recent contraction inequality for the soft-thresholding activation function \cite{shultzman2023generalization} to this broad class of functions, we derive a uniform generalization error bound for discrete-time ResNet, and then develop a discrete-to-continuum convergence to establish a compatible generalization error bound for continuous-time ResNet. 
Our contributions are summarized as follows:
\begin{itemize}
	\item[(i)]
	We extend a very recent contraction inequality on Rademacher complexity of composed function classes for the soft-thresholding operation to a broad type of activation functions, capturing their decomposition structures beyond standard Lipschitz continuity.
	This leads to an interesting refinement term in our generalization error bound analysis for both discrete- and continuous-time ResNets.
	
	\item[(ii)]
	More importantly, we establish uniform generalization error bounds for discrete-time ResNets and continuous-time ResNets respectively from a dynamical system perspective, by carrying out Rademacher complexity analysis, the equivalence of component hypothesis sets and a forward convergence property.
	The resulting bounds are of optimal order $O(1/\sqrt{S})$ in both settings and are compatible across the discrete-to-continuum transition.
	They both consist of a common layer-number independent leading term, and similar non-positive structural refinement terms from the structure of the activation function.
\end{itemize} 
{Different from existing literature that considers both discrete- and continuous-time ResNets simultaneously, our analysis from a dynamical system perspective gives better compatible generalization bounds under milder conditions; see Remark~\ref{remark: compare bound} (i).}
Our analysis reveals a fundamental consistency in the generalization behavior of ResNets viewed as dynamical systems, yielding optimal sample complexity rate in both discrete- and continuous-time settings simultaneously.
We also validate our theoretical results through numerical experiments on SVHN, CIFAR-10 and CIFAR-100 image classification tasks.

The remainder of this paper is organized as follows. 
Section~\ref{Sec: 2} formulates the generalization problems for discrete- and continuous-time ResNets. In Section~\ref{Sec: 3}, we first give our refined contraction inequality for Rademacher complexity of composed function classes, and then present our main theoretical results, i.e., uniform generalization bounds for discrete- and continuous-time ResNets. Section~\ref{Sec: 4} provides detailed proofs. In Section~\ref{Sec: 5}, we present some numerical experiments to validate the results.  Finally, Section~\ref{Sec: 6} concludes the paper.

\section{Notations and problem formulation}
\label{Sec: 2}
In this section, we give some notations and introduce the learning problems for ResNets from the dynamical system viewpoint. Then we describe their generalization problems in neural network supervised learning. 

\subsection{Notations} 
Throughout this paper, we adopt the following notations. Let $\mathbb{N}$ and $\mathbb{R}$ denote the sets of natural and real numbers, respectively. 
Scalars are denoted with italic letters (e.g., $n,N \in \mathbb{N}$), vectors and matrices are denoted with lowercase and capital in straight letters (e.g., $\mathrm{c} \in \mathbb{R}^N$ and $\mathrm{W} \in \mathbb{R}^{N \times N}$). Vector- and matrix-valued functions of time variable $t$ are written in boldface, like $\bold{c}(\cdot)$ and $\bold{W}(\cdot)$.
{For a Lipschitz continuous function $\zeta$ , its Lipschitz constant is denoted by $\mathrm{Lip}_{\zeta}$.}
The $i$-th element of vector $\mathrm{c}$ is written as $\mathrm{c}_i$, and the $i$-th row of matrix $\mathrm{C}$ is $\mathrm{C}_{i,:}$.
The $l^{\infty}$ and $l^2$ norms for vectors or matrices are denoted as $\|\cdot\|_{\infty}$ and $\|\cdot\|_2$, respectively. 
Let $\mathcal{X}$ be a vector space. The notation $\mathcal{X}^N$ repents the Cartesian product space of $\underbrace{\mathcal{X} \times \ldots \times \mathcal{X}}_{N}$. The characteristic function of the set $\Lambda$ is denoted as $\boldsymbol{1}_\Lambda(\cdot)$. The space of functions that are continuous on $\Omega \subset \mathbb{R}^d$ is denoted by $\mathcal{C}(\Omega)$.
The Sobolev space of functions that are $q$-times weakly differentiable and each weak derivative is $\mathcal{L}^2$ integrable on $\Omega$ is denoted by $\mathcal{H}^q(\Omega)$. By convention, \(1/0 = +\infty\). For readers’ convenience, the main notations used throughout the paper are summarized in Table~\ref{tab:notation}.


\begin{table}[htbp]
	\centering
	\caption{Summary of main notations.}
	\label{tab:notation}
	\begin{tabular}{ll}
		\toprule
		\textbf{Symbol} & \textbf{Meaning}\\
		\midrule
		$n_{\rm d}/n$  & The dimension of input data/network output \\
		
		$l/L$ & The index of network layer/ network depth (the number of network layers)  \\
		
		$\tau_L$ & The discretization step \\
		
		$\Theta_L^{\rm total}$ & The learnable parameters of discrete-time ResNet   \\
		
		$\mathrm{x}^{l}(\cdot; {\Theta}^{\rm total}_L)$ & The network state at the $l$-th layer in Eq.(\ref{equation: discrete-time ResNet}) with ${\Theta}^{\rm total}_L$  \\
		
		
		$\mathcal{F}^{l}$ & The hypothesis set of the network state at the $l$-th layer of discrete-time ResNet  \\
		
		
		$T$ & The terminal time of continuous-time ResNet  \\
		
		$\pmb{\Theta}^{\rm total}$ & The learnable parameters of continuous-time ResNet \\
		
		$\bold{x}(T;\cdot; \pmb{\Theta}^{\rm total})$ & The output state of continuous-time ResNet Eq.(\ref{equation: continuous-time ResNet}) with $\pmb{\Theta}^{\rm total} $  \\
		
		${\mathcal{F}}^{T}$ &  
		The hypothesis set of the output state of continuous-time ResNet at time $T$
		\\
		
		
		$\Omega_{{\Theta}^{\rm total}_L}/ \Omega_{\pmb{\Theta}^{\rm total}}$  & The discrete/continuous-time parameter set \\
		
		$\ell$ & The loss function  \\ 
		
		$\mathfrak{D}$ & The (data,label) distribution  \\
		
		$\mathfrak{R}_{{\mathfrak{D}}}(\cdot)$ & The expected loss (risk)  \\
		
		$\mathcal{S}$ & The training sample set \\
		
		$\hat{\mathfrak{R}}_{\mathcal{S}} (\cdot)$ & The empirical loss (risk)  \\
		
		$S$ & The number of training samples  \\
		
		$GE_{\mathcal{S}}(\cdot)$ & The generalization gap  \\
		
		$\mathscr{A}(\mathbb{R})$ & The set of activation functions  \\
		
		$\mathrm{Lip}_{\psi}$ & The Lipschitz constant of function $\psi$  \\
		
		$ \mathscr{R}_{\mathcal{Z}}(\mathcal{F})$ & The rademacher complexity of function class $\mathcal{F}$ \\
		
		$B_{{\mathfrak{D}}}$ & The bound of (data,label) distribution ${\mathfrak{D}}$  \\
		
		$B_{\pmb{\Theta}}$ & The bound of learnable parameters of discrete- and continuous-time ResNets \\
		
		
		$B_{\kappa}$ & The bound of the local Lipschitz function $\kappa$ for $\ell(\cdot, \mathrm{g})$ in Assumption~$(A_4)$ \\
		& on a compact set containing  all outputs of discrete- and continuous-time  ResNets  \\
		& over the (data,label) distribution $\mathfrak{D}$    \\

		$B_{\ell}$ & The bound of the loss function $\ell$ on a compact set containing all outputs of  \\
		
		&discrete- and continuous-time ResNets over the (data, label) distribution $\mathfrak{D}$ \\
		\bottomrule
	\end{tabular}
\end{table}

\subsection{Residual neural network and its dynamical system modeling}
We present the dynamical system modeling of ResNet \cite{weinan2017proposal, haber2017stable}. In particular, we consider the following discrete-time ResNet
\begin{equation} 
	\begin{aligned}
		\mathrm{x}^{l+1} &= \mathrm{x}^{l} + \tau_L \big[ 
		\mathrm{W}^{l}\psi \circ (\mathrm{V}^{l}\mathrm{x}^{l} + \mathrm{b}^{l}) + \mathrm{c}^{l}
		\big], \ 0\leq l \leq L-1 , \\
		\mathrm{x}^{0} &=  \psi \circ (\mathrm{U}\mathrm{d} + \mathrm{a}), 
	\end{aligned}
	\label{equation: discrete-time ResNet}
\end{equation} 
where $\mathrm{d} \in \mathbb{R}^{n_\mathrm{d}}$ is the input data; $\psi: \mathbb{R} \to \mathbb{R}$ is the activation function, and is applied elementwise to vectors; the parameters $(\mathrm{U}, \mathrm{a}) =: {\Theta}^{\rm pre}  \in \mathcal{E}^{\rm pre} =:\mathbb{R}^{n \times n_{\mathrm{d}}} \times \mathbb{R}^{n}$ are learnable parameters {in the preprocessing layer, and the parameters $(\mathrm{V}^{l}, \mathrm{W}^{l}, \mathrm{b}^{l}, \mathrm{c}^{l}) =: {\Theta}^{l}  \in \mathcal{E}:= \mathbb{R}^{m \times n} \times \mathbb{R}^{n \times m} \times \mathbb{R}^{m} \times \mathbb{R}^{n}$ are learnable parameters at the $l$-th layer of ResNet}; $L\in \mathbb{N}$ is the total layer number (the network depth); $\tau_L=T/L$ is the discretization step, which serves as a ``scaling factor".

Denote ${\Theta}_L \!= \!\left({\Theta}^{0}_L, {\Theta}^{1}_L, \ldots, {\Theta}^{L-1}_L \right) \!\in \!\mathcal{E}^{L}$. Let 
${\Theta}^{\rm total}_L\! =\! ({\Theta}^{\rm pre}, {\Theta}_L) \!\in\! \mathcal{E}^{\rm pre} \times \mathcal{E}^{L}$ be the learnable parameters for the discrete-time ResNet (\ref{equation: discrete-time ResNet}) with $L$ layers. {Let $\Omega_{{\Theta}^{\rm total}_L} \subset \mathcal{E}^{\rm pre} \times \mathcal{E}^{L}$ be the parameter set. We denote $\|{\Theta}_L^{l}\|_{\infty} = \max\{ \|\mathrm{V}^{l}\|_{\infty}, \|\mathrm{W}^{l}\|_{\infty}, \|\mathrm{b}^{l}\|_{\infty}, \|\mathrm{c}^{l}\|_{\infty} \}$, $l=0,\ldots, L-1$, and define a norm for ${\Theta}^{\rm total}_L \!\in \!\mathcal{E}^{\rm pre} \!\times\! \mathcal{E}^{L}$ as }
\begin{equation}
	\|{\Theta}^{\rm total}_L\|_{\infty} := \max \{\|\mathrm{U}\|_{\infty}, \|\mathrm{a}\|_{\infty}, \max_{0\leq l\leq L-1}  \|{\Theta}_L^{l}\|_{\infty}\}.
	\label{def: norm}
\end{equation}

We also write the state variable $\mathrm{x}^{l}$ in Eq.\eqref{equation: discrete-time ResNet} as $\mathrm{x}^{l}(\mathrm{d}; {\Theta}^{\rm total}_L)$ ($0\leq l \leq L$) to emphasize its dependency on the input data $\mathrm{d}$ and the learnable parameter ${\Theta}^{\rm total}_L$, and denote $\mathrm{x}_{i}^{l}(\mathrm{d}; {\Theta}^{\rm total}_L)$ as the $i$-th element of $\mathrm{x}^{l}(\mathrm{d}; {\Theta}^{\rm total}_L)$. 
Therefore $\mathrm{x}^{l}(\mathrm{d}; {\Theta}^{\rm total}_L)$ can be seen as functions from $\mathbb{R}^{n_\mathrm{d}}$ to $\mathbb{R}^{n}$ when the parameter ${\Theta}^{\rm total}_L$ is given, and we can define the hypothesis set of discrete-time ResNet at the $l$-th ($0\leq l\leq L$) layer as 
\begin{equation}
	\begin{aligned}
		\mathcal{F}^{l}  :=  \{ \mathrm{x}^{l}(\cdot; {\Theta}^{\rm total}_L)| \ & \mathrm{x}^{l}(\cdot; {\Theta}^{\rm total}_L): \mathbb{R}^{n_\mathrm{d}}\rightarrow \mathbb{R}^{n}  \text{ is the network state at  the  } \\
		& \ \text{ $l$-th layer of the ResNet in Eq.(\ref{equation: discrete-time ResNet}) with parameter }  {\Theta}^{\rm total}_L \in \Omega_{{\Theta}^{\rm total}_L} \}.
	\end{aligned}
	\label{def-equa: hypothesis set vector function}
\end{equation}
{The class of scalar functions representing the $i$-th component ($1 \leq i \leq n$) of the state variable at layer $l$ is defined as}
\begin{equation}
	\begin{aligned}
		\mathcal{F}^{l}_i  :=  \{ \mathrm{x}_{i}^{l}(\cdot; {\Theta}^{\rm total}_L)| \ &\mathrm{x}_{i}^{l}(\cdot; {\Theta}^{\rm total}_L) :  \mathbb{R}^{n_\mathrm{d}}\rightarrow \mathbb{R} \text{ is the $i$-th component of }  \\
		&  \ \mathrm{x}^{l}(\cdot; {\Theta}^{\rm total}_L) \text{ determined by Eq.\eqref{equation: discrete-time ResNet} with parameter } {\Theta}^{\rm total}_L  \in \Omega_{{\Theta}^{\rm total}_L} \}.
	\end{aligned}
	\label{def-equa: hypothesis set scalar function}
\end{equation}
Note that the hypothesis classes $\mathcal{F}^l$ and $\mathcal{F}_i^l$ ($0 \leq l \leq L$, $1\leq i \leq n$) consist of vector-valued and scalar-valued functions induced by the ResNet given in Eq.\eqref{equation: discrete-time ResNet}, and parameterized by parameters in $\Omega_{{\Theta}^{\rm total}_L}$. We omit the explicit dependency of the hypothesis sets on the parameter set for notational simplicity, as $\Omega_{{\Theta}^{\rm total}_L}$ is fixed throughout the paper.

The dynamical system approach models the discrete-time ResNet Eq.\eqref{equation: discrete-time ResNet} 
as a continuous-time ODE \cite{weinan2017proposal, haber2017stable}. In particular, we consider the following continuous-time counterpart of Eq.\eqref{equation: discrete-time ResNet}
\begin{equation} 
	\begin{aligned}
		\frac{\mathrm{d}{\bold{x}}(t)}{\mathrm{d}t} &=  
		\bold{W}(t)\psi \circ (\bold{V}(t)\bold{x}(t) + \bold{b}(t)) + \bold{c}(t), t \in [0, T], \\
		\bold{x}(0) & =  \psi \circ (\mathrm{U}\mathrm{d} + \mathrm{a}).
	\end{aligned}
	\label{equation: continuous-time ResNet}
\end{equation}
Similar to the discrete-time case, the learnable parameters of (\ref{equation: continuous-time ResNet}) are packaged as $\pmb{\Theta}^{\rm total} =  ({\Theta}^{\rm pre}, \pmb{\Theta})$, where ${\Theta}^{\rm pre} = (\mathrm{U}, \mathrm{a})$,  $\pmb{\Theta}(t) = (\bold{V}(t), \bold{W}(t), \bold{b}(t), $ $\bold{c}(t)) \in \mathcal{E}, t \in [0, T]$. We then consider the parameter space of $\pmb{\Theta}$ as $\mathcal{C}([0, T]; \mathcal{E}) \cap \mathcal{H}^1(0, T; \mathcal{E})
$ and denote
$$
\begin{aligned}
	\|\pmb{\Theta}\|_{\mathcal{C}} & := \max \{ \max_{0\leq t \leq T}  \|\bold{V}(t)\|_{\infty}, \max_{0\leq t \leq T}  \|\bold{W}(t)\|_{\infty},
	\max_{0\leq t \leq T}  \|\bold{b}(t)\|_{\infty},
	\max_{0\leq t \leq T}  \|\bold{c}(t)\|_{\infty}\}; 
	\\
	\|\pmb{\Theta}\|_{\mathcal{H}^1} & := \max \{\|\bold{V}\|_{\mathcal{H}^1}, \|\bold{W}\|_{\mathcal{H}^1}, \|\bold{b}\|_{\mathcal{H}^1}, \|\bold{c}\|_{\mathcal{H}^1} \}.
\end{aligned}
$$
For $\pmb{\Theta}^{\rm total} =  ({\Theta}^{\rm pre} , \pmb{\Theta} ) \in \mathcal{E}^{\rm pre} \times (\mathcal{C}([0, T]; \mathcal{E}) \cap \mathcal{H}^1(0, T; \mathcal{E}))$, we define its norm as 
\begin{equation}
	\|\pmb{\Theta}^{\rm total}\| = \max\{\|\mathrm{U}\|_{\infty}, \|\mathrm{a}\|_{\infty}, 
	\|\pmb{\Theta}\|_{\mathcal{C}} ,{\|\pmb{\Theta}\|_{\mathcal{H}^1}}
	\}.
	\label{def: con-norm of para}
\end{equation}
The set of learnable parameters is $\Omega_{\pmb{\Theta}^{\rm total}} \subset \mathcal{E}^{\rm pre} \times (\mathcal{C}([0, T]; \mathcal{E}) \cap \mathcal{H}^1(0, T; \mathcal{E}))$.

We also write the states $\bold{x}(t)$ in Eq.\eqref{equation: continuous-time ResNet} as the flow map notation $\bold{x}(t; \mathrm{d}; \pmb{\Theta}^{\rm total})$ to emphasize its dependency on the input data $\mathrm{d}$ and the learnable parameter $\pmb{\Theta}^{\rm total}$. 
Then the hypothesis set of the output state of continuous-time ResNet \eqref{equation: continuous-time ResNet} is given as 
\begin{equation}
	\begin{aligned}
		{\mathcal{F}}^{T}  =  \{\bold{x}(T;\cdot; \pmb{\Theta}^{\rm total})| & \ \bold{x}(T;\cdot; \pmb{\Theta}^{\rm total}): \mathbb{R}^{n_\mathrm{d}}\rightarrow \mathbb{R}^{n} \text{ is the output  state of the } \\
		&\text{  continuous-time ResNet in Eq.(\ref{equation: continuous-time ResNet})  with parameter }  \pmb{\Theta}^{\rm total} \in \Omega_{\pmb{\Theta}^{\rm total}} \};
	\end{aligned}
	\label{def-equa: continuous-time vector function hypothesis set}
\end{equation}
{and the class of scalar functions that represent the $i$-th component ($1 \leq i \leq n$) of the output state of the continuous-time ResNet is defined as}
\begin{equation}
	\begin{aligned}
		{\mathcal{F}}_i^{T}  =  \{\bold{x}_i(T; \cdot; \pmb{\Theta}^{\rm total})| \ & \bold{x}_i(T; \cdot; \pmb{\Theta}^{\rm total}): \mathbb{R}^{n_\mathrm{d}}\rightarrow \mathbb{R} \text{ is the $i$-th  component of }   \\
		&  \bold{x}(T;\cdot; \pmb{\Theta}^{\rm total})  \text{ determined by Eq.(\ref{equation: continuous-time ResNet}) with parameter } \pmb{\Theta}^{\rm total} \in \Omega_{\pmb{\Theta}^{\rm total}} \}.
	\end{aligned}
	\label{def-equa: continuous-time scalar function hypothesis set}
\end{equation}
Similar to the discrete-time case, we omit the dependency of the hypothesis sets on the parameter set, as $\Omega_{\pmb{\Theta}^{\rm total}}$ is fixed throughout the paper.

\subsection{The generalization problems of discrete- and continuous-time ResNets} 
Let \(\mathfrak{D}\) be an (often unknown) probability distribution supported on a $\sigma$-algebra of $ \mathbb{R}^{n_\mathrm{d}} \times \mathbb{R}^{n}$.  In supervised learning, the objective is to approximate an $\mathbb{R}^{n_\mathrm{d}} $ to $\mathbb{R}^{n}$ target function that effectively ``characterizes" a relationship from the data to the label in \(\mathfrak{D}\). 

The target function is determined by solving the following optimal control problem
\begin{equation}
	\begin{aligned}
		&\inf \limits_{{\Theta}^{\rm total}_L \in \Omega_{{\Theta}^{\rm total}_L} }  
		\Big\{ \mathfrak{R}_{{\mathfrak{D}}}(\mathrm{x}^{L}(\cdot; {\Theta}^{\rm total}_L)) := \mathbb{E}_{(\mathrm{d}, \mathrm{g}) \sim {\mathfrak{D}}} \big[{\ell} (\mathrm{x}^{L}(\mathrm{d}; {\Theta}^{\rm total}_L), \mathrm{g}) \big]\Big\} \\
		&{\rm subject \ to: } \ 	 
		\mathrm{x}^{L}(\mathrm{d}; {\Theta}^{\rm total}_L) \text{ is the output state of the ResNet in Eq.(\ref{equation: discrete-time ResNet}) with input } \mathrm{d},
		\\
	\end{aligned}
	\label{equation: discrete-time-control problem P}
\end{equation}
where $\ell: \mathbb{R}^{n} \times \mathbb{R}^{n} \rightarrow[0, \infty)$ is a loss function, and $\mathfrak{R}_{{\mathfrak{D}}}(\mathrm{x}^{L}(\cdot; {\Theta}^{\rm total}_L)) $ is the so-called expected loss or risk. Since the distribution \(\mathfrak{D}\) is unknown, this quantity cannot be computed directly. Instead, we approximate it using a training dataset $\mathcal{S} := \{(\mathrm{d}_{(s)}, \mathrm{g}_{(s)})\}_{s=1}^{S}$ consisting of independently and identically distributed (i.i.d.) samples drawn from $\mathfrak{D}$. The goal is then to find a hypothesis $\mathrm{x}^{L}(\cdot; {\Theta}^{\rm total}_L)$ by solving the following sampled optimal control problem
\begin{equation}
	\begin{aligned}
		&\inf \limits_{{\Theta}^{\rm total}_L \in \Omega_{{\Theta}^{\rm total}_L} }  
		\Big\{ \hat{\mathfrak{R}}_{\mathcal{S}}(\mathrm{x}^{L}(\cdot; {\Theta}^{\rm total}_L)) := \frac{1}{S}\sum_{s=1}^{S}{\ell} (\mathrm{x}^{L}(\mathrm{d}_{(s)}; {\Theta}^{\rm total}_L), \mathrm{g}_{(s)}) \Big\} \\
		&{\rm subject \ to: } \ 	 
		\mathrm{x}^{L}(\mathrm{d}_{(s)}; {\Theta}^{\rm total}_L)  \text{ is the output state of the ResNet in Eq.(\ref{equation: discrete-time ResNet}) with input } \mathrm{d}_{(s)},
		\\
	\end{aligned}
	\label{equation: discrete-time-control problem P_S}
\end{equation}
where $\hat{\mathfrak{R}}_{\mathcal{S}}(\mathrm{x}^{L}(\cdot; {\Theta}^{\rm total}_L))$ is the empirical risk of the hypothesis.

A fundamental question in generalization theory is to understand how minimizing the empirical risk $\hat{\mathfrak{R}}_{\mathcal{S}}(\mathrm{x}^{L}(\cdot; {\Theta}^{\rm total}_L))$ provides a meaningful approximation to minimizing the expected risk
$\mathfrak{R}_{{\mathfrak{D}}}(\mathrm{x}^{L}(\cdot; {\Theta}^{\rm total}_L)) $. This can be characterized by the so-called generalization gap
\begin{equation}
	GE_{\mathcal{S}}(\mathrm{x}^{L}(\cdot; {\Theta}^{\rm total}_L)) := \mathfrak{R}_{{\mathfrak{D}}}(\mathrm{x}^{L}(\cdot; {\Theta}^{\rm total}_L)) - \hat{\mathfrak{R}}_{\mathcal{S}}(\mathrm{x}^{L}(\cdot; {\Theta}^{\rm total}_L)), \ \mathrm{x}^{L}(\cdot; {\Theta}^{\rm total}_L) \in {\mathcal{F}}^{L}.
	\label{equation: discrete-time generalization error}
\end{equation}
The generalization gap of continuous-time ResNet is similarly defined, i.e.,
\begin{equation}
	GE_{\mathcal{S}}(\bold{x}(T;\cdot; \pmb{\Theta}^{\rm total})) = \mathfrak{R}_{{\mathfrak{D}}}(\bold{x}(T;\cdot; \pmb{\Theta}^{\rm total})) - \hat{\mathfrak{R}}_{\mathcal{S}}(\bold{x}(T;\cdot; \pmb{\Theta}^{\rm total})), \ \bold{x}(T;\cdot; \pmb{\Theta}^{\rm total}) \in {\mathcal{F}}^{T}.
	\label{equation: continuous-time generalization error}
\end{equation}

In this work, we aim to investigate the asymptotic generalization behavior of ResNets in the deep-layer limit by establishing compatible generalization error bounds simultaneously for discrete-time ResNet and its continuous-time counterpart.

\section{Main results}
\label{Sec: 3}
Motivated by the dynamical system modeling approach, which offers valuable insights into the approximation behavior, structural stability, and depth variation behavior of DNNs,
we establish generalization error bounds for discrete- and continuous-time ResNets. This captures their generalization behavior in the deep-layer limit. 
We begin by introducing a class of commonly used activation functions and develop a new contraction-type inequality that reflects their structural properties. This inequality plays a key role in controlling the Rademacher complexity of the hypothesis class, leading to potentially better generalization error bounds.

\begin{definition}
	Let $\Gamma(\mathbb{R})$ be a set of functions such that for any ${\phi} \in \Gamma(\mathbb{R})$, ${\phi}: \mathbb{R} \to \mathbb{R}$ is increasing, $\mathrm{Lip}_{\phi}$-Lipschitz, and ${\phi}(x)=0$ for all $x \leq 0$.		
	Then the set of activation functions $\mathscr{A}(\mathbb{R})$ considered in this paper is given as 
	$\mathscr{A}(\mathbb{R}) = \{\psi(\cdot) = \phi_1(\cdot-\alpha) - \phi_2(-\cdot-\beta)| \ \phi_1, \phi_2 \in  \Gamma(\mathbb{R}), \ \alpha,\beta \ge 0 \} .$
	\label{definition: activation func}
\end{definition}
\begin{table}[h]
	\begin{center}
		\caption{Examples of activation functions in $\mathscr{A}(\mathbb{R})$. The parameters $a_1 \in [0,1]$, $a_2 , \lambda, \lambda_1, \lambda_2 >0$. 
		}
		\label{tab: activation function}
		\begin{tabular}{lllll}
			\toprule
			\textbf{Activation  functions} & \textbf{$\qquad \psi$} & \textbf{$\phi_1$} & \textbf{$\phi_2$} & $(\alpha, \beta)$\\
			\hline
			{ReLU \cite{glorot2011deep}}& \(\left\{ \begin{array}{ll}
				x, \  x \geq 0 \\
				0, \  x < 0 
			\end{array} \right. \)   & $x \boldsymbol{1}_{[0,\infty)}(x)$  & $0$ & $(0,\beta)$ \\
			\hline
			{LeakyReLU \cite{maas2013rectifier}} & \(\left\{ \begin{array}{ll}
				x, \  x \geq 0 \\
				a_1x, \  x < 0  
			\end{array} \right. \)   & $x \boldsymbol{1}_{[0,\infty)}(x)$  & $a_1x\boldsymbol{1}_{[0,\infty)}(x)$ & $(0,0)$ \\
			\hline
			TReLU& \(\left\{ \begin{array}{ll}
				x-\lambda, \  x \geq \lambda \\
				0, \  x < \lambda 
			\end{array} \right. \)   & $x \boldsymbol{1}_{[0,\infty)}(x)$  & $0$ & $(\lambda,\beta)$ \\
			\hline
			{ELU} \cite{clevert2015fast} & \(\left\{ \begin{array}{ll}
				x, \ x \geq 0 \\
				a_2(e^{x} -1) , \  x < 0 
			\end{array} \right. \) & $x  \boldsymbol{1}_{[0,\infty)}(x)$ & $\begin{array}{ll}
				-a_2 (e^{-x} -1) \\
				\ \  \cdot \boldsymbol{1}_{[0,\infty)}(x)
			\end{array}$& (0,0)\\
			\hline
			{TEReLU} \cite{pandey2020overcoming} & \(\left\{ \begin{array}{ll}
				x-\lambda_1, \  x \geq \lambda_1 \\
				0, \  -\lambda_2 < x < \lambda_1 \\
				a_2(e^{x+\lambda_2}\!-\!1), \! x \!\leq \!-\lambda_2 
			\end{array} \right. \)   & $x \boldsymbol{1}_{[0,\infty)}(x)$  &$\begin{array}{ll}
				-a_2 (e^{-x} -1)  \\
				\ \ \cdot \boldsymbol{1}_{[0,\infty)}(x)
			\end{array}$& $(\lambda_1,\lambda_2)$ \\
			\hline
			$\mathrm{ST}_{\lambda}$ \cite{beck2017first} & \(\left\{ \begin{array}{ll}
				x - \lambda,\  x \ge \lambda \\
				0 ,\  -\lambda < x < \lambda \\
				x+ \lambda, \  x \leq  -\lambda \ & 
			\end{array} \right. \) & $x \boldsymbol{1}_{[0,\infty)}(x)$ & $x \boldsymbol{1}_{[0,\infty)}(x)$ & ($\lambda,\lambda$)\\
			\hline
			$\mathrm{ST}_{\lambda_1, \lambda_2}$ & \(\left\{ \begin{array}{ll}
				x - \lambda_1,\  x \ge \lambda_1 \\
				0 ,\  -\lambda_2 < x < \lambda_1 \\
				x+ \lambda_2, \  x \leq  -\lambda_2 \ & 
			\end{array} \right. \)& $x \boldsymbol{1}_{[0,\infty)}(x)$ & $x \boldsymbol{1}_{[0,\infty)}(x)$ &($\lambda_1,\! \lambda_2$)\\
			\hline
			{Tanh \cite{goodfellow2016deep}} &\(\frac{e^x - e^{-x}}{e^x+e^{-x}}\) & $\frac{e^x - e^{-x}}{e^x+e^{-x}}  \boldsymbol{1}_{[0,\infty)}(x)$ &  $\frac{e^x - e^{-x}}{e^x+e^{-x}}  \boldsymbol{1}_{[0,\infty)}(x)$& (0,0) \\
			\hline
		\end{tabular}
	\end{center}
\end{table}
The set $\mathscr{A}(\mathbb{R})$ includes many commonly used activation functions, as summarized in Table~\ref{tab: activation function}. 
Any function $\psi \in \mathscr{A}(\mathbb{R})$ that can be written in the form $\psi = \phi_1(\cdot - \alpha) - \phi_2(-\cdot - \beta)$ with $\phi_1, \phi_2 \in \Gamma(\mathbb{R})$, is  $\mathrm{Lip}_\psi$-Lipschitz continuous, where $\mathrm{Lip}_\psi = \max\{\mathrm{Lip}_{\phi_1}, \mathrm{Lip}_{\phi_2}\}$. Recalling the definition of empirical Rademacher complexity, we derive a refined contraction inequality by using the decomposition structure of activation functions.
\begin{proposition}
	\label{proposition: Rademacher complexity property of ac function}
	Let $\mathcal{G}$ be a set of functions mapping $\mathbb{R}^N$ to $\mathbb{R}$. Consider a sample set $\mathcal{Z}=\left\{\mathrm{z}_{(1)}, \mathrm{z}_{(2)}, \cdots, \mathrm{z}_{(S)}\right\}$ from some distribution $\mathfrak{B}$ and an activation function $\psi \in \mathscr{A}(\mathbb{R})$ of the form $\psi(\cdot):=\phi_1(\cdot-\alpha) - \phi_2(-\cdot-\beta)$, where $\phi_1, \phi_2 \in \Gamma(\mathbb{R})$ and $\alpha, \beta \ge 0$. Then there exists a constant $C_{\mathcal{Z}}$ dependent on $\mathcal{Z}$, $0 \leq C_\mathcal{Z} \leq \min \{S, \frac{ \mathrm{Lip}_{\psi}S}{ \mathrm{Lip}_{\phi_1}\alpha +  \mathrm{Lip}_{\phi_2}\beta} \mathscr{R}_{\mathcal{Z}}(\mathcal{G}) \}$ such that   
	\begin{equation}
		\mathscr{R}_{\mathcal{Z}}(\psi \circ \mathcal{G}) 
		\leq   
		\mathrm{Lip}_{\psi} \cdot \mathscr{R}_{\mathcal{Z}}(\mathcal{G})- \frac{  \mathrm{Lip}_{\phi_1}\alpha +  \mathrm{Lip}_{\phi_2}\beta}{S} C_{\mathcal{Z}},
		\label{equation: property of Rademacher of activation}
	\end{equation}
	where $ \mathrm{Lip}_{\psi} = \max\{\mathrm{Lip}_{\phi_1}, \mathrm{Lip}_{\phi_2}\}$. 
\end{proposition}
\begin{proof}
	The proof is given in Section~\ref{sec: Proof for Proposition}.
\end{proof}
\begin{remark}
	Proposition~\ref{proposition: Rademacher complexity property of ac function} provides a refined characterization of the Rademacher complexity of composed function classes compared to the classical contraction inequality in \cite[Lemma 26.9]{shalev2014understanding}, which states that
	\begin{equation}
		\mathscr{R}_{\mathcal{Z}}(\psi \circ \mathcal{G}) \leq  \mathrm{Lip}_{\psi} \cdot \mathscr{R}_{\mathcal{Z}}(\mathcal{G}).
		\label{equation: lip property of Rademacher of activation old}
	\end{equation}
	Moreover, the result in \cite[Lemma~1]{shultzman2023generalization} is a special case of the above proposition with $\psi=\mathrm{ST}_{\lambda}$.
\end{remark}

The following example shows that Proposition~\ref{proposition: Rademacher complexity property of ac function} does give a tighter bound than (\ref{equation: lip property of Rademacher of activation old}).
\begin{example}
	Given constants $\alpha, \beta, \gamma, \eta >0$ such that $0 < \max\{ \alpha, \beta \} < \gamma$ and any sample set $\mathcal{Z} = \{\mathrm{z}_{(s)} | \  \mathrm{z}_{(s)}\in \mathbb{R}^{N}, \|\mathrm{z}_{(s)}\|_2 = 1, s= 1,\ldots,S \}$, we consider the function class $$\mathcal{G} := \{c_1\| \cdot\|_2 + c_2: \mathbb{R}^{N} \to \mathbb{R} \ | 
	(c_1, c_2) \in [0, \eta] \times [\alpha, \gamma] \text{ or } 
	(c_1, c_2) \in [-\eta, 0] \times [-\gamma, -\beta] \},$$ 
	and function 
	$\psi(\cdot) = (\cdot - \alpha)  \boldsymbol{1}_{[0,\infty)}(\cdot - \alpha) - (-\cdot - \beta)  \boldsymbol{1}_{[0,\infty)}(-\cdot-\beta)$. Then
	\begin{equation}
		\begin{aligned}
			\mathscr{R}_{\mathcal{Z}}(\mathcal{G}) & =  \frac{1}{2^{S-1}} (\eta+\gamma)\binom{S - 1}{\lfloor \frac{S}{2} \rfloor}  , \\
			\mathscr{R}_{\mathcal{Z}}(\psi \circ \mathcal{G}) & = \frac{1}{2^{S}} (2\eta+2\gamma-\beta-\alpha)\binom{S - 1}{\lfloor \frac{S}{2} \rfloor},
		\end{aligned}
	\end{equation}
	and therefore $ \mathscr{R}_{\mathcal{Z}}(\psi \circ \mathcal{G}) \leq \mathrm{Lip}_{\psi} \mathscr{R}_{\mathcal{Z}}(\mathcal{G}) - \frac{\alpha+\beta}{S} \cdot \frac{S}{2^S}\binom{S - 1}{\lfloor \frac{S}{2} \rfloor}.$
	\label{proposition: example of RC}
\end{example}
\begin{proof}
	The proof is basic and is supplied in Appendix~\ref{sec: appendix-1}.
\end{proof}

Before showing our main results, we list the following assumptions required for our analysis:
\begin{itemize}
	\item [(${A}_1$)] 
	{(Bounded (data, label) distribution)} The (data, label) distribution $\mathfrak{D}$ is supported in a bounded set, i.e., for all $(\mathrm{d}, \mathrm{g}) \sim \mathfrak{D}, \ \|\mathrm{d}\|_2 + \|\mathrm{g}\|_2 \leq B_{{\mathfrak{D}}}$ for some $B_{{\mathfrak{D}}}>0$.
	
	\item [(${A}_2$)] 
	{(Bounded parameter sets)} For some $B_{\pmb{\Theta}} > 0 $, the parameter sets are given by 
	$$
	\begin{aligned}
		\Omega_{{\Theta}^{\rm total}_L} &= \{{\Theta}^{\rm total}_L: {\Theta}^{\rm total}_L \in \mathcal{E}^{\rm pre} \times \mathcal{E}^L,  	\|{\Theta}^{\rm total}_L\|_{\infty} \leq B_{\pmb{\Theta}} \}, \\
		\Omega_{\pmb{\Theta}^{\rm total}} &= \{\pmb{\Theta}^{\rm total}: \pmb{\Theta}^{\rm total} \in \mathcal{E}^{\rm pre} \times 
		(\mathcal{C}([0, T]; \mathcal{E}) \cap \mathcal{H}^1(0, T; \mathcal{E})), 
		\|\pmb{\Theta}^{\rm total}\| \! \leq \!B_{\pmb{\Theta}}\}.
	\end{aligned}
	$$
	
	\item [(${A}_3$)] {(Activation function)}  The activation function $\psi \in  \mathscr{A}(\mathbb{R})$.
	
	\item [(${A}_4$)]
	{(Loss function)} The loss function ${\ell}: \mathbb{R}^{n} \times \mathbb{R}^{n} \rightarrow [0, \infty)$ is continuous and satisfies a local Lipschitz condition with respect to the first argument in the following sense. There exists a continuous function ${\kappa}: \mathbb{R}^{n} \times \mathbb{R}^{n} \times \mathbb{R}^{n} \rightarrow [0, \infty)$ such that for all $\mathrm{x}, \tilde{\mathrm{x}}, \mathrm{g} \in \mathbb{R}^{n}$,
	$$
	|{\ell} (\mathrm{x}, \mathrm{g}) - {\ell} (\tilde{\mathrm{x}}, \mathrm{g})| \leq {\kappa}(\mathrm{x}, \tilde{\mathrm{x}},  \mathrm{g}) \|\mathrm{x} -  \tilde{\mathrm{x}}\|_{2}.
	$$
\end{itemize}
\begin{remark}
	These assumptions are mild and commonly used in deep learning theory. In particular,
	the boundedness assumptions (${A}_1$) and (${A}_2$) are standard \cite{chen2020understanding, kobler2020total, schnoor2023generalization, huang2024on}.
	(${A}_3$) is satisfied by lots of activation functions, like those in Table~\ref{tab: activation function}.
	Some widely used loss functions meet (${A}_4$);  see the following examples: 
	\begin{itemize}
		\item In image inverse problems, the loss function is usually ${\ell}(\mathrm{x},\mathrm{g}) = \|\mathrm{x}-\mathrm{g}\|_2^2$. Then we can take ${\kappa}(\mathrm{x}, \tilde{\mathrm{x}},  \mathrm{g}) = \|\mathrm{x}-\mathrm{g}\|_2 + \|\tilde{\mathrm{x}}-\mathrm{g}\|_2$.
		\item  In image classification tasks, the loss function is usually the ramp loss \cite{shalev2014understanding}, which is $\mathrm{Lip}_{\ell}$-Lipschitz. Then we can take ${\kappa}(\mathrm{x}, \tilde{\mathrm{x}},  \mathrm{g}) \equiv \mathrm{Lip}_{\ell}$.
	\end{itemize}
\end{remark}

We now present our two main results. One is the generalization error bound for discrete-time ResNet in (\ref{equation: discrete-time ResNet}), and the other is a compatible one for continuous-time ResNet in (\ref{equation: continuous-time ResNet}).
\begin{theorem} 
	\label{theorem: Uniform generalization error bound for discrete-time ResNet}
	(Generalization error bound for discrete-time ResNet)
	Assume assumptions (${A}_1$)-(${A}_4$) hold. Let  $\psi \in \mathscr{A}(\mathbb{R})$ with $\phi_1, \phi_2 \in \Gamma(\mathbb{R})$.
	Then there is a constant ${C}_{\mathcal{S},L}$ dependent on data $\mathcal{S}$ and layer number $L$,
	$0 \leq {C}_{\mathcal{S},L} \leq \min \{ \sqrt{S} \frac{1 +  2\mathrm{Lip}_{\psi} B_{\pmb{\Theta}}}{ 2(\mathrm{Lip}_{\phi_1}\alpha +  \mathrm{Lip}_{\phi_2}\beta)}, S \}$  such that for any $\delta \in (0,1)$, with probability at least $1-\delta$, 
	we have
	\begin{equation}
		\begin{aligned}
			& \sup_{\mathrm{x}^{L}(\cdot; {\Theta}^{\rm total}_L) \in {\mathcal{F}}^{L}} GE_{\mathcal{S}}(\mathrm{x}^{L}(\cdot; {\Theta}^{\rm total}_L) ) \\
			\leq  & 2 \sqrt{2} n  B_{\kappa}B_{\pmb{\Theta}} \frac{M(T, \!  \mathrm{Lip}_{\psi},  \! n_{\mathrm{d}}, \! B_{{\mathfrak{D}}},  \!B_{\pmb{\Theta}})}{\sqrt{S}}   
			\! + \! 4B_{\ell} \sqrt{\frac{2\! \log(4/\delta)}{S}}  \\
			& + (- 2 \sqrt{2}) n  B_{\kappa}B_{\pmb{\Theta}} \frac{( \mathrm{Lip}_{\phi_1}\alpha\! +\!  \mathrm{Lip}_{\phi_2}\beta)\exp(T  \mathrm{Lip}_{\psi} B_{\pmb{\Theta}}^2)}{S}  T   {C}_{\mathcal{S},L},
		\end{aligned}
		\label{equation: uniform generalization error bound for discrete ResNet}
	\end{equation}
	where  $M(T,  \mathrm{Lip}_{\psi}, n_{\mathrm{d}}, B_{{\mathfrak{D}}}, B_{\pmb{\Theta}})=\Big(\mathrm{Lip}_{\psi}B_{{\mathfrak{D}}}\sqrt{2\log(2n_{\mathrm{d}})} + 1 + T (1 +  2\mathrm{Lip}_{\psi} B_{\pmb{\Theta}}) \Big)\exp \big(2T  \mathrm{Lip}_{\psi} B_{\pmb{\Theta}}^2 \big)$, and $B_{\kappa}$, $B_{\ell}$ are constants dependent on $T, B_{{\mathfrak{D}}}, B_{\pmb{\Theta}},  \mathrm{Lip}_{\psi}$ and independent of $L$. 
\end{theorem}
\begin{proof}
	The proof is given in Section~\ref{subsec: Proof for Theorem-1}.
\end{proof}

\begin{remark}
	{The bound in \eqref{equation: uniform generalization error bound for discrete ResNet} is of $O(1/\sqrt{S})$ sample order, which consists of a leading positive term that is independent of the network depth $L$, and a non-positive depth-dependent structural refinement term of the form
		$$
		-2 \sqrt{2} n  B_{\kappa}B_{\pmb{\Theta}} \frac{( \mathrm{Lip}_{\phi_1}\alpha\! +\!  \mathrm{Lip}_{\phi_2}\beta)\exp(T  \mathrm{Lip}_{\psi} B_{\pmb{\Theta}}^2)}{S} T {C}_{\mathcal{S},L}.
		$$
		Although this structural term appears with a negative sign in
		\eqref{equation: uniform generalization error bound for discrete ResNet}, the resulting right-hand side of \eqref{equation: uniform generalization error bound for discrete ResNet} remains non-negative.
		This is guaranteed by $ {C}_{\mathcal{S},L} \leq \min \{ \sqrt{S} \frac{1 +  2\mathrm{Lip}_{\psi} B_{\pmb{\Theta}}}{ 2(\mathrm{Lip}_{\phi_1}\alpha +  \mathrm{Lip}_{\phi_2}\beta)}, S \}$ in the
		Theorem~\ref{theorem: Uniform generalization error bound for discrete-time ResNet}. By discarding this refinement term, one immediately obtains a depth-uniform upper bound, i.e., an upper bound independent of the layer number $L$. }
\end{remark}

\begin{remark}
	\label{remark: ac func}
	The non-positive structural term in (\ref{equation: uniform generalization error bound for discrete ResNet}) provides a refinement to the generalization bound. This term depends on several activation-related constants like \(\mathrm{Lip}_{\phi_1}\), \(\mathrm{Lip}_{\phi_2}\), \(\alpha\), and \(\beta\).
	Note that $ {C}_{\mathcal{S},L}$ therein is given in an existential manner and is hard to compute precisely, as can be seen in the later proof. 
	Even for a specific activation function and data distribution, evaluating these constants remains nontrivial (see Example~\ref{proposition: example of RC}). 
	Consequently, it is generally impractical to derive explicit rules for choosing the parameters \(\mathrm{Lip}_{\phi_1}\), \(\mathrm{Lip}_{\phi_2}\), \(\alpha\), and \(\beta\) from the theoretical bound for activation function design.
	Nevertheless, this non-positive structural term reveals the potential role of the activation function structure in improving generalization performance. For instance, one may treat structural parameters in the activation function (like \(\alpha\) and \(\beta\)) as learnable parameters in practice, allowing the network to adaptively learn activation configurations that lead to improved generalization performance in a data-driven manner.
	We will demonstrate this numerically in the experimental section.
	
\end{remark}


\begin{theorem}
	\label{theorem: Uniform generalization error bound for continuous-time ResNet} 
	(Generalization error bound for continuous-time ResNet)
	Assume assumptions (${A}_1$)-(${A}_4$) hold. Let $\psi \in \mathscr{A}(\mathbb{R})$ with $\phi_1, \phi_2 \in \Gamma(\mathbb{R})$. Then there is a constant $C_{\mathcal{S}}$ dependent on ${\mathcal{S}}$, $0 \leq {C}_{\mathcal{S}} \leq \min \{ \sqrt{S} \frac{1 +  2\mathrm{Lip}_{\psi} B_{\pmb{\Theta}}}{2 (\mathrm{Lip}_{\phi_1}\alpha +  \mathrm{Lip}_{\phi_2}\beta)}, S \}$ such that for any $\delta\in(0,1)$, with probability at least $1-\delta$, we have 
	\begin{equation}
		\begin{aligned}
			& \sup_{\bold{x}(T;\cdot; \pmb{\Theta}^{\rm total}) \in {\mathcal{F}}^{T}} GE_{\mathcal{S}}(\bold{x}(T;\cdot; \pmb{\Theta}^{\rm total})) \\
			\leq & 2 \sqrt{2} n  B_{\kappa}B_{\pmb{\Theta}} \frac{M(T,  \mathrm{Lip}_{\psi}, n_{\mathrm{d}}, B_{{\mathfrak{D}}}, B_{\pmb{\Theta}})}{\sqrt{S}}    + 4B_{\ell}\sqrt{\frac{2\log(4/\delta)}{S}} \\
			&  + (- 2 \sqrt{2}) n  B_{\kappa}B_{\pmb{\Theta}}  \frac{( \mathrm{Lip}_{\phi_1}\alpha +  \mathrm{Lip}_{\phi_2}\beta)\exp(T  \mathrm{Lip}_{\psi} B_{\pmb{\Theta}}^2) }{S} TC_{\mathcal{S}},  
		\end{aligned}
		\label{equation: generalization error bound for continuous-time ResNet networks-1}
	\end{equation}
	where $M(T,  \mathrm{Lip}_{\psi}, n_{\mathrm{d}}, B_{{\mathfrak{D}}}, B_{\pmb{\Theta}})=\Big(\mathrm{Lip}_{\psi}B_{{\mathfrak{D}}}\sqrt{2\log(2n_{\mathrm{d}})} + 1 + T (1 + 2 \mathrm{Lip}_{\psi} B_{\pmb{\Theta}}) \Big)\exp \big(2T  \mathrm{Lip}_{\psi} B_{\pmb{\Theta}}^2 \big)$, and $B_{\kappa}$, $B_{\ell}$ are constants dependent on $T, B_{{\mathfrak{D}}}, B_{\pmb{\Theta}},  \mathrm{Lip}_{\psi}$ and independent of $L$.
\end{theorem}
\begin{proof}
	The proof is given in Section~\ref{subsec: Proof for Theorem-2}.
\end{proof}

\begin{remark}
	The bound in \eqref{equation: generalization error bound for continuous-time ResNet networks-1} admits the same structural decomposition as the bound \eqref{equation: uniform generalization error bound for discrete ResNet} for discrete-time ResNet, i.e., a leading term of order $O(1/\sqrt{S})$ and a non-positive structural refinement term. Similarly to \eqref{equation: uniform generalization error bound for discrete ResNet}, the bound remains well-defined and non-negative. Combined with \eqref{equation: uniform generalization error bound for discrete ResNet}, this result shows that the generalization performance of ResNet remains uniformly bounded as the depth tending to infinity, providing theoretical support for the asymptotic consistency of deep residual architectures.
	Moreover, a special case of this bound \eqref{equation: generalization error bound for continuous-time ResNet networks-1} for constant-in-time learnable parameters gives a generalization estimate for neural ODEs~\cite{chen2018neural} with an optimal sample order $O(1/\sqrt{S})$ leading term and a non-positive refinement term.
\end{remark}

\begin{remark}
	Here we would like to compare our estimation results to existing ones for the generalization error bound of ResNet. Given the focus of this paper, we will consider existing results for ResNet with general loss functions, with particular attention to those considering both discrete- and continuous-time ResNets; see Table~3. The similarities and differences between our results and these existing ones are as follows. 
	\begin{itemize}
		\item[(i)]  In \cite{marion2023generalization, bleistein2023generalization}, the authors provide generalization gap estimations for both discrete- and continuous-time networks, and the relationships between these results and ours are threefold.  Firstly, their sample complexity orders are respectively \(O(S^{-1/4})\) and \(O(\sqrt{\log(S)}/\sqrt{S})\) derived via covering number arguments, which are weaker than our \(O(1/\sqrt{S})\) order obtained by dynamical system perspective with the equivalence of component hypothesis sets (Lemma~\ref{lemma: equivalence hypothesis set with scalar function}) and a refined Rademacher complexity estimate. 
		{Secondly, their results rely on Lipschitz-type continuity assumptions for continuous-time learnable parameters, while our analysis uses a bound condition in \(\mathcal{C}\cap\mathcal{H}^1\), which is milder.}
		Thirdly, these results and ours all exhibit exponential dependency on the terminal time \(T\), the Lipschitz constant of the activation functions, and the bound on the learnable parameters.
		
		\item[(ii)] In \cite{e2020rademacher}, Corollary~2.2 provides a posteriori generalization gap estimation closely related to the problem studied in our paper, and the relationships between this result and ours are fivefold. Firstly, this result and ours all have a sample order of \(O(1/\sqrt{S})\) (an optimal order under the Rademacher complexity analysis framework), but ours include an additional non-positive structural refinement term. Secondly, this result considers only the ReLU activation function, but ours apply to a broader class of activation functions. Thirdly, this result uses weighted path-norm of learnable parameters and is thus in some sense an instance-dependent estimate (the bound depends on the parameter values), but ours are uniform over the domain of learnable parameters. Fourthly, this result considers the case of scalar network output and depends on the input dimension $n_{\mathrm d}$, but ours consider general network output dimension $n$ and thus depend on $n_{\mathrm d}$ and $n$. Fifthly, this result considers only discrete-time ResNet, while ours provide compatible estimations for both discrete- and continuous-time ResNets.

		\item[(iii)] 
		In \cite{hanson2024rademacher}, the authors derive a generalization bound for neural ODEs via a Chen–Fliess series representation, and the relationships between this result and ours are twofold.  
		Firstly, this result and ours both achieve a sample complexity order of \(O(1/\sqrt{S})\), but ours include an additional non-positive structural refinement term. 
		Secondly, this result considers only the continuous-time setting, whereas ours provide compatible estimates for both discrete- and continuous-time ResNets.
	\end{itemize}
	\label{remark: compare bound}
\end{remark}

\begin{table}[htbp]
	\centering
	\label{tab: compare bound}
	\centering
	\scriptsize  
	\caption{Comparison of some existing bounds with ours of generalization gap for discrete- and continuous-time ResNets, where only results for general (Lipschitz-continuous) loss $\ell$ are listed. For ease of illustration, we present them using the symbols in this paper as closely as possible. 
		In \cite{marion2023generalization}, $B=6 \mathrm{Lip}_{\ell} \mathrm{Lip}_f \exp \left(\mathrm{Lip}_f B_{\Theta}\right)\left(B_{\mathfrak{D}}+M B_{\Theta} \exp \left(\mathrm{Lip}_f B_{\Theta}\right)+B_{\mathfrak{D}}\right)$, and $K_{\mathcal{W}}$ is $L$ times of an upper bound on the differences between the parameters of adjacent layers. In \cite{bleistein2023generalization}, $M_{\Theta}=B_{\Theta} \mathrm{Lip}_{\psi} \exp \left(\mathrm{Lip}_{\mathbf{A}} \mathrm{Lip}_{\psi} \mathrm{Lip}_{\mathrm{x}}\right)\left(B_{\Theta} B_{\mathfrak{D}}+B_{\Theta}+B_{\Theta} \mathrm{Lip}_{\mathrm{x}}\right)$, $U_1=\log (20 \sqrt{S m} K_1)$ and $U_2=\log (20 \sqrt{S n_{\rm d} m} K_2)$ with $K_1, K_2$ dependent on $\Theta, \mathrm{Lip}_x, B_{\mathfrak{D}}$ and $\mathrm{Lip}_{\phi}$. In \cite{e2020rademacher}, the norm $ \|\boldsymbol{\theta}\|_{\mathrm{P}}=\left\||\mathrm{u}|^T\left(\mathrm{I}+2\left|\mathrm{V}^L \| \mathrm{W}^L\right|\right) \cdots\left(\mathrm{I}+2\left|\mathrm{V}^1 \| \mathrm{W}^1\right|\right) 2|\mathrm{U}|\right\|_1,
		$ with
		$\mathrm{u}$,$\mathrm{U}$,$\mathrm{V}^l$,$\mathrm{W}^l$ $(l=1,\ldots,L)$ being learnable parameters. In \cite{hanson2024rademacher}, $M$ is the bound of fixed control input of networks; $\Lambda_k$, $k\ge 0$ are the Lie derivatives, whose values need to be further estimated according to the specific model.
		In ours, $C_{\rm dis} = n  B_{\kappa}B_{\pmb{\Theta}} \frac{( \mathrm{Lip}_{\phi_1}\alpha\! +\!  \mathrm{Lip}_{\phi_2}\beta)\exp(T  \mathrm{Lip}_{\psi} B_{\pmb{\Theta}}^2)}{S} T C_{\mathcal{S},L}$, $C_{\rm con} = n  B_{\kappa}B_{\pmb{\Theta}}  \frac{( \mathrm{Lip}_{\phi_1}\alpha +  \mathrm{Lip}_{\phi_2}\beta)\exp(T  \mathrm{Lip}_{\psi} B_{\pmb{\Theta}}^2) T}{S} C_{\mathcal{S}}$. See Remark~\ref{remark: compare bound} for a detailed discussion.}
	\begin{tabular}{p{1.9cm}p{3.8cm}p{3.8cm}p{1.0cm}p{2.2cm}p{1.3cm}}
		\toprule
		\textbf{} & \multicolumn{2}{c}{\textbf{Bounds}} & \textbf{Order of sample} & \textbf{Hyper-parameters} & \textbf{Non-positive } \\
		\cmidrule(lr){2-3}
		& \textbf{Discrete-time} & \textbf{Continuous-time} &  \textbf{size} & & \textbf{structural term} \\
		\midrule
		\cite{marion2023generalization} 
		& \begin{tabular}{@{}p{4cm}@{}}
			$B(n_{\rm d}+1) \sqrt{\dfrac{\log \left(B_{\Theta} n_{\rm d} S\right)}{S}}$ \\[2pt]
			$+ B \dfrac{n_{\rm d}^2 \sqrt{K_{\mathcal{W}}}}{S^{1/4}} + \dfrac{B}{\sqrt{S}} \sqrt{\log (1/\delta)}$
		\end{tabular}
		& \begin{tabular}{@{}p{4cm}@{}}
			$B \sqrt{\dfrac{(m+1) \log \left(B_{\Theta} m S\right)}{S}}$ \\[2pt]
			$+ B \dfrac{m \sqrt{\mathrm{Lip}_{\pmb{\Theta}}}}{S^{1/4}} + \dfrac{B}{\sqrt{S}} \sqrt{\log( 1/\delta)}$
		\end{tabular}
		& $S^{-1/4}$ 
		& $n_{\rm d}; n\!=\!n_{\rm d}; m$; \  $T=1$ & $\times$ \\
		\midrule
		\cite{bleistein2023generalization} 
		& \begin{tabular}{@{}p{4cm}@{}}
			$\frac{24M_{\Theta}\mathrm{Lip}_\ell}{\sqrt{S}} \sqrt{2mU_1 + n_{\rm d} m(m+2)U_2} $ \\[2pt]
			$ + B_\ell \sqrt{\frac{\log (1/\delta)}{2S}}$
		\end{tabular}
		& \begin{tabular}{@{}p{4cm}@{}}
			$\frac{24M_{\Theta}\mathrm{Lip}_\ell}{\sqrt{S}} \sqrt{2mU_1 + n_{\rm d} m(m+2)U_2} $ \\[2pt]
			$ + B_\ell \sqrt{\frac{\log (1/\delta)}{2S}}$
		\end{tabular}
		& $\frac{\sqrt{\log(S)}}{S^{1/2}}$ 
		& $n_{\rm d}; n\!=\!n_{\rm d}; m; $ \ $T=1$ & $\times$ \\
		\midrule
		\cite{e2020rademacher} 
		& \begin{tabular}{@{}p{4cm}@{}}
			$2\left(\|\boldsymbol{\theta}\|_{\mathrm{P}}+1\right) \dfrac{2 \sqrt{2 \log (2 n_{\rm d})}+1}{\sqrt{S}}$ \\
			$+ 2\sqrt{\dfrac{2 \log (14 / \delta)}{S}}$
		\end{tabular}
		& $\times$
		& $S^{-1/2}$  & $n_{\rm d}; n=1$ & $\times$ \\
		\midrule
		\cite{hanson2024rademacher} 
		& $\times$
		& \begin{tabular}{@{}p{3.9cm}@{}}
			$  \frac{8B_{\mathfrak{D}}}{\sqrt{S}} + \frac{8}{\sqrt{S}} \sum_{k \geq 0} \frac{(m M T)^k}{k!} \Lambda_k $   $+B_{\ell} \sqrt{\dfrac{2\log(1/\delta)}{S}}$ 
		\end{tabular}
		& $S^{-1/2}$ & $n_{\rm d}; n\!=\!n_{\rm d}; m ;$ \ $T$ & $\times$ \\
		\midrule
		\textbf{Ours}
		& \begin{tabular}{@{}p{4cm}@{}}
			$2\sqrt{2} n B_{\kappa}B_{\pmb{\Theta}} \dfrac{M(T, \mathrm{Lip}_{\psi}, n_{\mathrm{d}}, B_{{\mathfrak{D}}}, B_{\pmb{\Theta}})}{\sqrt{S}}$ \\[2pt]
			$+ 4B_{\ell} \sqrt{\dfrac{2\log(4/\delta)}{S}} + (-2\sqrt{2}) C_{\rm dis}$
		\end{tabular}
		& \begin{tabular}{@{}p{4cm}@{}}
			$2\sqrt{2} n B_{\kappa}B_{\pmb{\Theta}} \dfrac{M(T, \mathrm{Lip}_{\psi}, n_{\mathrm{d}}, B_{{\mathfrak{D}}}, B_{\pmb{\Theta}})}{\sqrt{S}}$ \\[2pt]
			$+ 4B_{\ell} \sqrt{\dfrac{2\log(4/\delta)}{S}} + (-2\sqrt{2}) C_{\rm con}$
		\end{tabular}
		& $S^{-1/2}$ & $n_{\rm d}; n; T$ & $\checkmark$ \\
		\bottomrule
	\end{tabular}
\end{table}

\section{Proofs}
\label{Sec: 4}
In this section, we present detailed proofs for the main results. We start with the proof for Proposition~\ref{proposition: Rademacher complexity property of ac function}.

\subsection{Proof of Proposition~\ref{proposition: Rademacher complexity property of ac function}}
\label{sec: Proof for Proposition}

\begin{proof}
	(Proof of Proposition~\ref{proposition: Rademacher complexity property of ac function})
	\label{sec: proof of lemma Rademacher complexity property of ac function}
	We need only to consider the general case with $\alpha>0$ or $\beta>0$, since the case with $\alpha=\beta=0$ reduces to the result of  existing well-known contraction inequality \cite[Lemma 26.9]{shalev2014understanding}.
	The proof follows the chaining method framework \cite{vershynin2018high, shultzman2023generalization} and is divided into three steps. The key idea is to carefully track how the piecewise structure of the activation function interacts with the Rademacher variables, allowing us to extract a non-positive structural refinement term beyond the standard Lipschitz-based bound.
	
	Step 1: We use the chaining method to decouple $\varepsilon_{(S)}$. Let $\pmb{\varepsilon}_s:= \{\varepsilon_{(1)}, \ldots, \varepsilon_{(s)} \}$ for $1 \leq s \leq S$, and $R :=  \mathbb{E}_{\pmb{\varepsilon}_S} \sup_{{g} \in \mathcal{G}}\Big[ \sum_{s=1}^{S} \varepsilon_{(s)} \psi (g(\mathrm{z}_{(s)}))\Big]$. Since $\varepsilon_{(s)}$ $(1 \leq s \leq S)$ are independently and uniformly distributed over $\{-1, 1\}$, we have
	\begin{equation}
		\begin{aligned}
			R =	& \frac{1}{2} \mathbb{E}_{\pmb{\varepsilon}_{S-1}} \sup_{{g} \in \mathcal{G}} \Big[ \psi(g(\mathrm{z}_{(S)})) +\sum_{s=1}^{S-1} \varepsilon_{(s)} \psi(g(\mathrm{z}_{(s)}))\Big] \\
			& + \frac{1}{2}	 \mathbb{E}_{\pmb{\varepsilon}_{S-1}} \sup_{{g} \in \mathcal{G}} \Big[ -\psi(g(\mathrm{z}_{(S)})) +\sum_{s=1}^{S-1} \varepsilon_{(s)} \psi(g(\mathrm{z}_{(s)}))\Big] \\
			= &  \frac{1}{2} \mathbb{E}_{\pmb{\varepsilon}_{S-1}} \sup_{g, \hat{g} \in \mathcal{G}} \Big\{ \big[\psi(g(\mathrm{z}_{(S)})) - \psi(\hat{g}(\mathrm{z}_{(S)}))\big]  +\sum_{s=1}^{S-1} \varepsilon_{(s)} \big[\psi(g(\mathrm{z}_{(s)})) + \psi(\hat{g}(\mathrm{z}_{(s)}))\big]\Big\}.
			\label{equation-proof: property of Ac-1}
		\end{aligned}
	\end{equation}
	Let $u_{\pmb{\varepsilon}_{s}}(g, \hat{g}) := \sum_{i=1}^{s} \varepsilon_{(i)} \big[\psi(g(\mathrm{z}_{(i)})) + \psi(\hat{g}(\mathrm{z}_{(i)}))\big]$ for $s \ge 1$ and $u_{0}(g, \hat{g}):=0$. We observe that $u_{\pmb{\varepsilon}_{s}}(g, \hat{g})$ is symmetric, i.e., $u_{\pmb{\varepsilon}_{s}}(g, \hat{g}) = u_{\pmb{\varepsilon}_{s}}(\hat{g} ,g)$.
	
	Step 2:	
	We estimate the first term in the right-hand side of Eq.~\eqref{equation-proof: property of Ac-1}. Because an activation function in $\mathscr{A}(\mathbb{R})$ has a piecewise structure in the regions $(-\infty, -\beta)$, $[-\beta, \alpha]$, and $(\alpha, +\infty)$, we divide $\mathcal{G}$ into three categories
	\begin{equation}
		\left\{
		\begin{aligned}
			{\mathcal{G}}^{-1} & = \{ {g} \in \mathcal{G} : \ g(\mathrm{z}_{(S)}) < - \beta \}, \\
			{\mathcal{G}}^{0} & = \{ {g} \in \mathcal{G} : \ - \beta \leq g(\mathrm{z}_{(S)}) \leq \alpha \}, \\
			{\mathcal{G}}^{1} & = \{ {g} \in \mathcal{G} : \ g(\mathrm{z}_{(S)}) > \alpha \}.
		\end{aligned}
		\right.
		\nonumber
	\end{equation}
	Define $\gamma_{a,b} = \sup_{{g} \in \mathcal{G}^a, \hat{g} \in \mathcal{G}^b} \Big\{ \big[\psi(g(\mathrm{z}_{(S)})) - \psi(\hat{g}(\mathrm{z}_{(S)}))\big] + u_{\pmb{\varepsilon}_{S-1}}(g, \hat{g}) \Big\}$, where $a,b \in \{-1, 0, 1\}$. Next, we estimate these $\gamma_{a,b}$.
	
	\begin{itemize}
		\item[(i)] 
		For $(g, \hat{g}) \in \mathcal{G}^a \times \mathcal{G}^b$ with $a = b$ (i.e., $(a,b) = (-1,-1)$, $(a,b) = (0,0)$, or $(a,b) = (1,1)$), we have 
		$ \psi(g(\mathrm{z}_{(S)})) - \psi(\hat{g}(\mathrm{z}_{(S)})) \leq  \mathrm{Lip}_{\psi} \cdot | g(\mathrm{z}_{(S)}) - \hat{g}(\mathrm{z}_{(S)}) |$,  since $\psi$ is $ \mathrm{Lip}_{\psi}$-Lipschitz continuous. Therefore, we obtain
		$$
		\begin{aligned}
			\gamma_{a,b} &\leq \sup_{{g} \in \mathcal{G}^a, \hat{g} \in \mathcal{G}^b} \Big\{  \mathrm{Lip}_{\psi} | g(\mathrm{z}_{(S)}) - \hat{g}(\mathrm{z}_{(S)}) | + u_{\pmb{\varepsilon}_{S-1}}(g, \hat{g}) \Big\} \\
			& = \sup_{{g} \in \mathcal{G}^a, \hat{g} \in \mathcal{G}^b} \Big\{  \mathrm{Lip}_{\psi} ( g(\mathrm{z}_{(S)}) - \hat{g}(\mathrm{z}_{(S)}) ) + u_{\pmb{\varepsilon}_{S-1}}(g, \hat{g}) \Big\},
		\end{aligned}
		$$
		where the last equality is due to the symmetry of $u_{\pmb{\varepsilon}_{(S-1)}}$.
		\item[(ii)] For $(g, \hat{g}) \in \mathcal{G}^a \times \mathcal{G}^b$ with $a - b = 1$ (i.e., $(a,b) = (0,-1)$ or $(a,b) = (1,0)$), we obtain
		$$ \psi(g(\mathrm{z}_{(S)})) - \psi(\hat{g}(\mathrm{z}_{(S)})) \leq  \mathrm{Lip}_{\psi} | g(\mathrm{z}_{(S)}) - \hat{g}(\mathrm{z}_{(S)}) | =  \mathrm{Lip}_{\psi} \big(g(\mathrm{z}_{(S)}) - \hat{g}(\mathrm{z}_{(S)}) \big). $$ 
		Therefore, we get
		$$
		\begin{aligned}
			\gamma_{a,b} & \leq \sup_{{g} \in \mathcal{G}^a, \hat{g} \in \mathcal{G}^b} \Big\{  \mathrm{Lip}_{\psi} \big(g(\mathrm{z}_{(S)}) - \hat{g}(\mathrm{z}_{(S)}) \big) + u_{\pmb{\varepsilon}_{S-1}}(g, \hat{g}) \Big\}.
		\end{aligned}
		$$
		\item[(iii)] For $(g, \hat{g}) \in \mathcal{G}^a \times \mathcal{G}^b$ with $a-b=2$ (i.e., $(a,b)=(1,-1)$), we note that $-\hat{g}(\mathrm{z}_{(S)}) - \beta > 0$, and
		$$
		\begin{aligned}
			\psi(g(\mathrm{z}_{(S)})) - \psi(\hat{g}(\mathrm{z}_{(S)})) & = \phi_1(g(\mathrm{z}_{(S)}) - \alpha) + \phi_2(-\hat{g}(\mathrm{z}_{(S)}) - \beta)  \\
			& \leq   \mathrm{Lip}_{\phi_1}(g(\mathrm{z}_{(S)}) - \alpha) +    \mathrm{Lip}_{\phi_2}(- \hat{g}(\mathrm{z}_{(S)}) - \beta) \\
			& \leq  \mathrm{Lip}_{\psi} \big(g(\mathrm{z}_{(S)}) - \hat{g}(\mathrm{z}_{(S)}) \big) - \big( \mathrm{Lip}_{\phi_1}\alpha +  \mathrm{Lip}_{\phi_2}\beta \big),
		\end{aligned}
		$$
		where the second inequality uses $\mathrm{Lip}_{\phi_1}, \mathrm{Lip}_{\phi_2} \leq \mathrm{Lip}_{\psi}$.
		It then follows that
		$$
		\begin{aligned}
			\gamma_{a,b} & \leq \sup_{{g} \in \mathcal{G}^a, \hat{g} \in \mathcal{G}^b} \Big\{  \mathrm{Lip}_{\psi} ( g(\mathrm{z}_{(S)}) - \hat{g}(\mathrm{z}_{(S)}) ) + u_{\pmb{\varepsilon}_{S-1}}(g, \hat{g}) -\big( \mathrm{Lip}_{\phi_1}\alpha +  \mathrm{Lip}_{\phi_2}\beta \big)  \Big\}.
		\end{aligned}
		$$
		\item[(iv)] For $(g, \hat{g}) \in \mathcal{G}^a \times \mathcal{G}^b$ with $a-b<0$ (i.e., $(a,b)=(-1, 0)$ or $(a,b)=(0, 1)$ or $(a,b)=(-1, 1)$), we have 
		$$
		\begin{aligned}
			\psi(g(\mathrm{z}_{(S)})) - \psi(\hat{g}(\mathrm{z}_{(S)})) \leq  \mathrm{Lip}_{\psi} | g(\mathrm{z}_{(S)}) - \hat{g}(\mathrm{z}_{(S)}) |  =  \mathrm{Lip}_{\psi} (\hat{g}(\mathrm{z}_{(S)}) - g(\mathrm{z}_{(S)})).
		\end{aligned}
		$$
		Therefore, we get
		$$
		\begin{aligned}
			\gamma_{a,b} & \leq \sup_{{g} \in \mathcal{G}^a, \hat{g} \in \mathcal{G}^b} \Big\{  \mathrm{Lip}_{\psi} (\hat{g}(\mathrm{z}_{(S)}) - g(\mathrm{z}_{(S)})) + u_{\pmb{\varepsilon}_{S-1}}(g, \hat{g}) \Big\} \\
			& = \sup_{{g} \in \mathcal{G}^a, \hat{g} \in \mathcal{G}^b} \Big\{  \mathrm{Lip}_{\psi} (\hat{g}(\mathrm{z}_{(S)}) - g(\mathrm{z}_{(S)})) + u_{\pmb{\varepsilon}_{S-1}}(\hat{g}, g) \Big\} = \gamma_{b, a}.
		\end{aligned}
		$$
	\end{itemize}
	
	Combining the above discussions, we obtain
	\begin{equation}
		\begin{aligned}
			\max_{a,b \in \{-1, 0, 1\}} \gamma_{a,b} 
			\leq  \sup_{g, \hat{g} \in \mathcal{G}} \Big\{  \mathrm{Lip}_{\psi} ( g(\mathrm{z}_{(S)}) - \hat{g}(\mathrm{z}_{(S)}) ) + u_{\pmb{\varepsilon}_{S-1}}(g, \hat{g})  - \big( \mathrm{Lip}_{\phi_1}\alpha +  \mathrm{Lip}_{\phi_2}\beta \big) \cdot \bold{1}_{{\mathcal{G}}^{1} \times {\mathcal{G}}^{-1}}((g, \hat{g})) \Big\},
		\end{aligned}
		\label{equation-proof: property of Ac-2}
	\end{equation}
	with
	\begin{equation}
		\bold{1}_{{\mathcal{G}}^{1} \times {\mathcal{G}}^{-1}}(g, \hat{g}) =
		\left \{
		\begin{aligned}
			& 1, \ (g, \hat{g}) \in {\mathcal{G}}^{1} \times {\mathcal{G}}^{-1}, \\
			& 0, \ (g, \hat{g}) \not \in {\mathcal{G}}^{1} \times {\mathcal{G}}^{-1}.
		\end{aligned}
		\right.
		\nonumber
	\end{equation}
	The indicator term explicitly encodes the frequency across the sample with which function pairs fall into the extreme regions of the activation.
	
	Step 3: We now aggregate the non-positive terms accumulated over all Rademacher configurations, yielding a data-dependent constant in Eq.~\eqref{equation: property of Rademacher of activation} that captures the overall strength of the structural refinement. 
	For any $\eta > 0$ and any given $\pmb{\varepsilon}_S \setminus {\varepsilon}_{(S)} \in \{-1,1\}^{S-1}$, there exist functions $g^*_{\pmb{\varepsilon}_S\setminus {\varepsilon}_{(S)}}, \hat{g}^*_{\pmb{\varepsilon}_S\setminus {\varepsilon}_{(S)}} \in \mathcal{G}$  such that
	\begin{equation}
		\begin{aligned}
			& \sup_{g, \hat{g} \in \mathcal{G}} \Big\{  \mathrm{Lip}_{\psi} ( g(\mathrm{z}_{(S)}) - \hat{g}(\mathrm{z}_{(S)}) ) + u_{\pmb{\varepsilon}_{S-1}}(g, \hat{g})  - \big( \mathrm{Lip}_{\phi_1}\alpha +  \mathrm{Lip}_{\phi_2}\beta \big) \cdot \bold{1}_{{\mathcal{G}}^{1} \times {\mathcal{G}}^{-1}}(g, \hat{g}) \Big\} \\
			\leq &  \mathrm{Lip}_{\psi} \big( g^*_{\pmb{\varepsilon}_S\setminus {\varepsilon}_{(S)}}(\mathrm{z}_{(S)}) - \hat{g}^*_{\pmb{\varepsilon}_S\setminus {\varepsilon}_{(S)}}(\mathrm{z}_{(S)}) \big) + u_{\pmb{\varepsilon}_{S-1}}(g^*_{\pmb{\varepsilon}_S\setminus {\varepsilon}_{(S)}}, \hat{g}^*_{\pmb{\varepsilon}_S\setminus {\varepsilon}_{(S)}}) \\
			&  - \big( \mathrm{Lip}_{\phi_1}\alpha +  \mathrm{Lip}_{\phi_2}\beta \big) \bold{1}_{{\mathcal{G}}^{1} \times {\mathcal{G}}^{-1}}(g^*_{\pmb{\varepsilon}_S\setminus {\varepsilon}_{(S)}}, \hat{g}^*_{\pmb{\varepsilon}_S\setminus {\varepsilon}_{(S)}})  + \eta/S.
		\end{aligned}
		\label{equation-proof: property of Ac-3}
	\end{equation}
	By combining Eqs.~\eqref{equation-proof: property of Ac-1}–\eqref{equation-proof: property of Ac-3}, we get
	\begin{align}
		R \leq & \frac{1}{2} \mathbb{E}_{\pmb{\varepsilon}_{S-1}} \sup_{g, \hat{g} \in \mathcal{G}} \Big\{  \mathrm{Lip}_{\psi} ( g(\mathrm{z}_{(S)}) - \hat{g}(\mathrm{z}_{(S)}) ) + u_{\pmb{\varepsilon}_{S-1}}(g, \hat{g}) \Big\} \nonumber \\
		& - \frac{  \mathrm{Lip}_{\phi_1}\alpha +  \mathrm{Lip}_{\phi_2}\beta }{2^{S-1}} \cdot \sum_{\pmb{\varepsilon}_S\setminus {\varepsilon}_{(S)} \in \{-1,1\}^{S-1}} \bold{1}_{{\mathcal{G}}^{1} \times {\mathcal{G}}^{-1}}(g^*_{\pmb{\varepsilon}_S\setminus {\varepsilon}_{(S)}}, \hat{g}^*_{\pmb{\varepsilon}_S\setminus {\varepsilon}_{(S)}}) + \eta/S \nonumber  \\
		= & \mathbb{E}_{\pmb{\varepsilon}_{S}} \sup_{g \in \mathcal{G}} \Big\{  {\varepsilon}_{(S)}  \mathrm{Lip}_{\psi} g(\mathrm{z}_{(S)})  + \sum_{s=1}^{S-1} \varepsilon_{(s)} \psi(g(\mathrm{z}_{(s)})) \Big\} \nonumber  \\
		& - \frac{  \mathrm{Lip}_{\phi_1}\alpha +  \mathrm{Lip}_{\phi_2}\beta }{2^{S-1}} \cdot \sum_{\pmb{\varepsilon}_S \setminus {\varepsilon}_{(S)} \in \{-1,1\}^{S-1}} \bold{1}_{{\mathcal{G}}^{1} \times {\mathcal{G}}^{-1}}(g^*_{\pmb{\varepsilon}_S\setminus {\varepsilon}_{(S)}}, \hat{g}^*_{\pmb{\varepsilon}_S\setminus {\varepsilon}_{(S)}}) + \eta/S. \nonumber 
	\end{align}
	
	It follows by iterating the above process for $S-1$ more steps (i.e., for $s \in \{1, \ldots, S-1\}$) that 
	$$
	\begin{aligned}
		R \leq &  \mathrm{Lip}_{\psi} \cdot \mathbb{E}_{\pmb{\varepsilon}_{S}} \sup_{{g} \in \mathcal{G}} \Big\{ \sum_{s=1}^{S} \varepsilon_{(s)} g(\mathrm{z}_{(s)}) \Big\}   - \frac{  \mathrm{Lip}_{\phi_1}\alpha +  \mathrm{Lip}_{\phi_2}\beta }{2^{S-1}} \cdot \sum_{s=1}^{S} \sum_{\pmb{\varepsilon}_S \setminus {\varepsilon}_{(s)} \in \{-1,1\}^{S-1}} \bold{1}_{{\mathcal{G}}^{1} \times {\mathcal{G}}^{-1}}(g^*_{\pmb{\varepsilon}_S\setminus {\varepsilon}_{(s)}}, \hat{g}^*_{\pmb{\varepsilon}_S\setminus {\varepsilon}_{(s)}}) + \eta,
	\end{aligned}
	$$
	where $g^*_{\pmb{\varepsilon}_S\setminus {\varepsilon}_{(s)}}, \hat{g}^*_{\pmb{\varepsilon}_S\setminus {\varepsilon}_{(s)}} \in \mathcal{G}$.  Hence, we arrive at
	$$
	\mathscr{R}_{\mathcal{Z}}(\psi \circ \mathcal{G}) \leq   \mathrm{Lip}_{\psi} \cdot \mathscr{R}_{\mathcal{Z}}(\mathcal{G})- \frac{( \mathrm{Lip}_{\phi_1}\alpha +  \mathrm{Lip}_{\phi_2}\beta)}{S} \bar{C}_{\mathcal{Z}} + \eta/S,
	$$
	where $\bar{C}_{\mathcal{Z}} = \bar{C}_{\mathcal{Z}}(\eta)= \frac{1}{2^{S-1}}\sum_{s=1}^{S} \sum_{\pmb{\varepsilon}_S \setminus {\varepsilon}_{(s)} \in \{-1,1\}^{S-1}} \bold{1}_{{\mathcal{G}}^{1} \times {\mathcal{G}}^{-1}}(g^*_{\pmb{\varepsilon}_S\setminus {\varepsilon}_{(s)}}, \hat{g}^*_{\pmb{\varepsilon}_S\setminus {\varepsilon}_{(s)}})$ depends on $\mathcal{Z}$ and implicitly on $\eta$ . 
	Taking the limit as $\eta \to 0$, we obtain (along a subsequence) that $\bar{C}_{\mathcal{Z}}(\eta) \to C_{\mathcal{Z}}$ for some $0 \leq C_{\mathcal{Z}} \leq S$.  
	Thus, we obtain
	\begin{equation}
		\mathscr{R}_{\mathcal{Z}}(\psi \circ \mathcal{G}) \leq   \mathrm{Lip}_{\psi} \cdot \mathscr{R}_{\mathcal{Z}}(\mathcal{G})- \frac{ \mathrm{Lip}_{\phi_1}\alpha +  \mathrm{Lip}_{\phi_2}\beta}{S} C_{\mathcal{Z}}.
		\nonumber
	\end{equation}
	Since the Rademacher complexity is nonnegative, we get 
	$C_{\mathcal{Z}} \leq \frac{ \mathrm{Lip}_{\psi}S}{ \mathrm{Lip}_{\phi_1}\alpha +  \mathrm{Lip}_{\phi_2}\beta} \mathscr{R}_{\mathcal{Z}}(\mathcal{G})$, and complete the proof.
\end{proof}

\subsection{Proof of Theorem~\ref{theorem: Uniform generalization error bound for discrete-time ResNet}}
\label{subsec: Proof for Theorem-1}

In this subsection, we give a detailed proof of Theorem~\ref{theorem: Uniform generalization error bound for discrete-time ResNet}.
{Our proof exploits the dynamical system formulation and a Rademacher complexity analysis via an equivalence property of component hypothesis sets (Lemma~\ref{lemma: equivalence hypothesis set with scalar function}) and our refined contraction inequality (Proposition~\ref{proposition: Rademacher complexity property of ac function}). }
The following lemma provides the boundedness of state variables. 

\begin{lemma}  
	Assume assumptions $(A_1)-(A_3)$ hold. 
	Suppose
	that ${\Theta}^{\rm total}_L \in \Omega_{{\Theta}^{\rm total}_L}, \pmb{\Theta}^{\rm total} \in \Omega_{\pmb{\Theta}^{\rm total}}$ and the input $\mathrm{d}$ is the first argument of a sample from  $\mathfrak{D}$. 
	Then the state variables $\{\mathrm{x}^{l}(\mathrm{d}; {\Theta}^{\rm total}_L)\}_{l=0}^L$ of discrete-time ResNet (\ref{equation: discrete-time ResNet}) and $\{\bold{x}(t; \mathrm{d}; \pmb{\Theta}^{\rm total}): 0\leq t \leq T\}$ of continuous-time ResNet (\ref{equation: continuous-time ResNet}) exist uniquely and satisfy
	\begin{equation}
		\begin{aligned}
			\|\mathrm{x}^{l}(\mathrm{d}; {\Theta}^{\rm total}_L)\|_{\infty} \leq & \Big[ 
			\mathrm{Lip}_{\psi} B_{\pmb{\Theta}} ( B_{{\mathfrak{D}}} + 1)\! + \! \tau_L l ( \mathrm{Lip}_{\psi}B_{\pmb{\Theta}}^2 \! + \!B_{\pmb{\Theta}})
			\Big] \exp(\tau_L l  \mathrm{Lip}_{\psi}B_{\pmb{\Theta}}^2) =: B_{\rm out}^l, \\
			\|\bold{x}(t; \mathrm{d}; \pmb{\Theta}^{\rm total}) \|_{\infty} \leq & \Big[ 
			\mathrm{Lip}_{\psi} B_{\pmb{\Theta}} ( B_{{\mathfrak{D}}} + 1) + t( \mathrm{Lip}_{\psi}B_{\pmb{\Theta}}^2 + B_{\pmb{\Theta}})
			\Big] \exp(t  \mathrm{Lip}_{\psi}B_{\pmb{\Theta}}^2) =: B_{\rm out}^t, 
		\end{aligned}
		\nonumber
	\end{equation}
	respectively. Moreover, let $B_{\rm out} = \Big[ 
	\mathrm{Lip}_{\psi} B_{\pmb{\Theta}} ( B_{{\mathfrak{D}}} + 1) + T ( \mathrm{Lip}_{\psi}B_{\pmb{\Theta}}^2 + B_{\pmb{\Theta}}) \Big] \exp(T  \mathrm{Lip}_{\psi}B_{\pmb{\Theta}}^2)$. We then have 
	$$\max \{ \max_{0 \leq l \leq L}\|\mathrm{x}^{l}(\mathrm{d}, {\Theta}^{\rm total}_L)\|_{\infty} ,  \max_{0 \leq t \leq T} \|\bold{x}(t; \mathrm{d}; \pmb{\Theta}^{\rm total}) \|_{\infty}    \}\leq  B_{\rm out}.$$
	\label{lemma: dis-ResNet-state-bound}
\end{lemma}
\begin{proof}
	The proof is standard and is supplied in Appendix~\ref{sec: appendix-2}.
\end{proof}

The uniform bound $B_{\rm out}$ for the states of discrete- and continuous-time ResNets implies a uniform bound $B_{\ell}$ and a uniform local Lipschitz coefficient $B_{\kappa}$ for the loss function $\ell$ used in both discrete- and continuous-time settings. These constants depend only on $T, B_{{\mathfrak{D}}}, B_{\pmb{\Theta}},  \mathrm{Lip}_{\psi}$, and will be used in proving Theorems~\ref{theorem: Uniform generalization error bound for discrete-time ResNet} and \ref{theorem: Uniform generalization error bound for continuous-time ResNet}.

The next lemma establishes an equivalence property of hypothesis sets ${\mathcal{F}}^{L}_{i}$ ($1\leq i \leq n$) under assumptions ($A_1$) ($A_2$), which plays a key role in proving Theorem~\ref{theorem: Uniform generalization error bound for discrete-time ResNet}.
\begin{lemma}
	Let assumptions ($A_1$)-($A_2$) hold. 
	For any layer $0\le l\le L$, the hypothesis sets ${\mathcal{F}}^{l}_{i}$ $(1\leq i \leq n)$ defined in (\ref{def-equa: hypothesis set scalar function}) satisfy
	${\mathcal{F}}^{l}_{i_1} = {\mathcal{F}}^{l}_{i_2}, i_1,i_2\in\{1,\dots,n\}.$ 
	\label{lemma: equivalence hypothesis set with scalar function}
\end{lemma}
\begin{proof}
	Let $\mathrm{P}_{(ij)}$ be the $n\times n$ permutation matrix swapping coordinates $i$ and $j$, i.e.,
	$\mathrm{P}_{(ij)} :=  
	\mathrm{I} + (\mathrm{e}_{(i)} - \mathrm{e}_{(j)})(\mathrm{e}_{(j)} - \mathrm{e}_{(i)})^{\top},
	$
	where $\mathrm{e}_i$ represents the $i$-th standard basis vector in an $n$-dimensional space, with 1 at the $i$-th position and 0 elsewhere, $\mathrm{I} \in \mathbb{R}^{n\times n}$ is the identity matrix. 
	Denote ${\Theta}_l := \left({\Theta}^{0}_L, {\Theta}^{1}_L, \ldots, {\Theta}^{l-1}_L \right) \in \mathcal{E}^{l}, {\Theta}^{\rm total}_l := ({\Theta}^{\rm pre}, {\Theta}_l) \in \mathcal{E}^{\rm pre} \times \mathcal{E}^{l}$.
	We show by induction on $l$ $(0\leq l \leq L)$ that, for each choice of parameters $\Theta^{\rm total}_l \in \Omega_{{\Theta}^{\rm total}_{l}}$ and indices $i_1,i_2\in\{1,\dots,n\}$, there is a modified parameter $\widetilde\Theta^{\rm total}_l \in \Omega_{{\Theta}^{\rm total}_{l}}$ such that
	\begin{equation}
		\mathrm{x}^{l}(\mathrm{d};\tilde{\Theta}^{\rm total}_{l}) = \mathrm{P}_{(i_1i_2)} \mathrm{x}^{l}(\mathrm{d};{\Theta}^{\rm total}_{l}).
		\label{equation: class-eqv}
	\end{equation}
	Once (\ref{equation: class-eqv}) holds, this establishes a bijection between the $i_1$‑th and $i_2$‑th components of the states, i.e., $\mathcal{F}^l_{i_1}=\mathcal{F}^l_{i_2}$, given that $\mathrm{P}_{(i_1i_2)}^2=\mathrm{I}$.
	
	For the preprocessing layer, if ${\Theta}^{\rm pre} = (\mathrm{U},\mathrm{a}) \in \mathcal{E}^{\rm pre}$ and $i_1, i_2 \in \{1,\ldots, n\}$, let $\tilde{\Theta}^{\rm pre} = (\mathrm{P}_{(i_1i_2)}\mathrm{U}, \mathrm{P}_{(i_1i_2)}\mathrm{a}) \in \mathcal{E}^{\rm pre}$.
	Then it is easy to see
	$$
	\mathrm{x}^{0}(\mathrm{d};\tilde{\Theta}^{\rm pre}) \!=\! \psi \circ (\mathrm{P}_{(i_1i_2)}\mathrm{U}\mathrm{d}+ \mathrm{P}_{(i_1i_2)}\mathrm{a}) \!= \! \mathrm{P}_{(i_1i_2)}\psi \circ (\mathrm{U}\mathrm{d}+\mathrm{a}) \!= \! \mathrm{P}_{(i_1i_2)}\mathrm{x}^{0}(\mathrm{d};{\Theta}^{\rm pre}), \ \forall \mathrm{d} \in \mathbb{R}^{n_\mathrm{d}},
	$$
	by the element-wise action of $\psi$. 
	For $l=0$ and given parameter ${\Theta}^{\rm total}_0 = ((\mathrm{U},\mathrm{a}),(\mathrm{V}^0,\mathrm{W}^0,\mathrm{b}^0,\mathrm{a}^0)) \in {\Theta}^{\rm total}_0$ and $i_1, i_2 \in \{1,\ldots, n\}$, let $\tilde{{\Theta}}_0^{\rm total} = ((\tilde{\mathrm{U}},\tilde{\mathrm{a}}),(\tilde{\mathrm{V}}^0,\tilde{\mathrm{W}}^0,\mathrm{b}^0,\tilde{\mathrm{c}}^0)) \in \Omega_{{\Theta}^{\rm total}_{0}}$, where
	$$
	\tilde{\mathrm{U}} = \mathrm{P}_{(i_1i_2)}\mathrm{U}, \ \tilde{\mathrm{a}} = \mathrm{P}_{(i_1i_2)}\mathrm{a}, \ \tilde{\mathrm{V}}^0 = \mathrm{V}^0\mathrm{P}_{(i_1i_2)}, \ \tilde{\mathrm{W}}^0 = \mathrm{P}_{(i_1i_2)}\mathrm{W}^0, \ \tilde{\mathrm{c}}^0 = \mathrm{P}_{(i_1i_2)}\mathrm{c}^0.
	$$
	Then we have 
	$$
	\mathrm{x}^{0}(\mathrm{d};\tilde{\Theta}^{\rm pre}) = \psi \circ (\mathrm{P}_{(i_1i_2)}\mathrm{U}\mathrm{d}+ \mathrm{P}_{(i_1i_2)}\mathrm{a}) = \mathrm{P}_{(i_1i_2)}\psi \circ (\mathrm{U}\mathrm{d}+\mathrm{a}) = \mathrm{P}_{(i_1i_2)}\mathrm{x}^{0}(\mathrm{d};{\Theta}^{\rm pre})
	$$
	and 
	$$
	\begin{aligned}
		\mathrm{x}^{1}_{i_1}(\mathrm{d};\tilde{\Theta}^{\rm total}_0)
		= & \mathrm{x}^{0}_{i_1}(\mathrm{d};\tilde{\Theta}^{\rm pre})+ \tau_L\Big(\tilde{\mathrm{W}}^0_{i_1,:}\psi \circ (\tilde{\mathrm{V}}^0\mathrm{x}^{0}(\mathrm{d};\tilde{\Theta}^{\rm pre})+\mathrm{b}^0)+\tilde{\mathrm{c}}^{0}_{i_1}\Big) \\
		= & \mathrm{x}^{0}_{i_2}(\mathrm{d};{\Theta}^{\rm pre}) + \tau_L\Big({\mathrm{W}}^0_{i_2,:}\psi \circ (\mathrm{V}^0\mathrm{x}^{0}(\mathrm{d};{\Theta}^{\rm pre})+\mathrm{b}^0)+\mathrm{c}^{0}_{i_2}\Big) \\
		= & \mathrm{x}^{1}_{i_2}(\mathrm{d};{\Theta}^{\rm total}_0), \ \forall \mathrm{d} \in \mathbb{R}^{n_\mathrm{d}},
	\end{aligned}
	$$
	owing to the definition of ${\mathcal{F}}^{l}_i$ (for $l=0$) in (\ref{def-equa: hypothesis set scalar function}), $\mathrm{P}_{(i_1i_2)}\mathrm{P}_{(i_1i_2)} = \mathrm{I}$ and the element-wise action of $\psi$. Through a similar procedure as described above, we can derive that $\mathrm{x}^{1}_{i_2}(\mathrm{d};\tilde{\Theta}^{\rm total}_0)
	= \mathrm{x}^{1}_{i_1}(\mathrm{d};{\Theta}^{\rm total}_0)$ and $\mathrm{x}^{1}_{i}(\mathrm{d};\tilde{\Theta}^{\rm total}_0)=\mathrm{x}^{1}_{i}(\mathrm{d};{\Theta}^{\rm total}_0), i \neq i_1,i_2$, i.e., $\mathrm{x}^{1}(\mathrm{d};\tilde{\Theta}^{\rm total}_{0}) = \mathrm{P}_{(i_1i_2)} \mathrm{x}^{1}(\mathrm{d};{\Theta}^{\rm total}_{0}).$
	
	Suppose the claim hold for layer $l\leq K$, where $K\ge 2$ and $K \in \mathbb{N}$. 
	Consider parameter ${\Theta}^{\rm total}_{K+1} = ({\Theta}^{\rm pre}_{K+1}, {\Theta}_{K+1}^0, \ldots,{\Theta}_{K+1}^{K}) \in  \Omega_{{\Theta}^{\rm total}_{K+1}}$ and $i_1, i_2 \in \{1,\ldots, n\}$.
	By the inductive hypothesis, we can find parameters
	$(\tilde{{\Theta}}^{\rm pre}_{K}, \tilde{{\Theta}}_{K}^0, \ldots,\tilde{{\Theta}}_{K}^{K-1})  \in \Omega_{{\Theta}^{\rm total}_{K}}$ such that for $\tilde{{\Theta}}_{K+1}^{\rm total} = (\tilde{{\Theta}}^{\rm pre}_{K}, \tilde{{\Theta}}_{K}^0, \ldots,\tilde{{\Theta}}_{K}^{K-1}, \tilde{{\Theta}}_{K+1}^{K})$ $=(\tilde{{\Theta}}^{\rm pre}_{K+1}, \tilde{{\Theta}}_{K+1}^0, \ldots, \tilde{{\Theta}}_{K+1}^{K-1}, \tilde{{\Theta}}_{K+1}^{K}) \in \Omega_{{\Theta}^{\rm total}_{K+1}}$, there is
	$$
	\mathrm{x}^{K}(\mathrm{d};\tilde{\Theta}^{\rm total}_{K+1}) = \mathrm{P}_{(i_1i_2)} \mathrm{x}^{K}(\mathrm{d};{\Theta}^{\rm total}_{K+1}).
	$$
	Let parameter $\tilde{{\Theta}}_{K+1}^{K} = (\tilde{\mathrm{V}}^{K},\tilde{\mathrm{W}}^{K},\mathrm{b}^{K},\tilde{\mathrm{c}}^{K})$, where
	$
	\tilde{\mathrm{V}}^{K} = \mathrm{V}^{K}\mathrm{P}_{(i_1i_2)},  \tilde{\mathrm{W}}^{K} = \mathrm{P}_{(i_1i_2)}\mathrm{W}^{K},  \tilde{\mathrm{c}}^{K} = \mathrm{P}_{(i_1i_2)}\mathrm{c}^{K}.
	$
	We therefore obtain 
	$$
	\begin{aligned}
		\mathrm{x}^{K+1}_{i_1}(\mathrm{d};\tilde{\Theta}^{\rm total}_{K+1})
		= & \mathrm{x}^{K}_{i_1}(\mathrm{d};\tilde{\Theta}^{\rm total}_{K+1})+ \tau_L\Big(\tilde{\mathrm{W}}^{K}_{i_1,:}\psi \circ (\tilde{\mathrm{V}}^{K}\mathrm{x}^{K}(\mathrm{d};\tilde{\Theta}^{\rm total}_{K+1})+\mathrm{b}^{K})+\tilde{\mathrm{c}}^{K}_{i_1})\Big) \\
		= & \mathrm{x}^{K}_{i_2}(\mathrm{d};{\Theta}^{\rm total}_{K+1}) + \tau_L\Big({\mathrm{W}}^{K}_{i_2,:}\psi \circ (\mathrm{V}^{K}\mathrm{x}^{K}(\mathrm{d};{\Theta}^{\rm total}_{K+1})+\mathrm{b}^{K})+\mathrm{c}^{K}_{i_2})\Big) \\
		= & \mathrm{x}^{K+1}_{i_2}(\mathrm{d};{\Theta}^{\rm total}_{K+1}),
	\end{aligned}
	$$
	due to the definition of ${\mathcal{F}}^{l}_i$ ($l\!=\!K\!+\!1$) in (\ref{def-equa: hypothesis set scalar function}), $\mathrm{P}_{(i_1i_2)}\mathrm{P}_{(i_1i_2)}\! = \! \mathrm{I}$ and the element-wise action of $\psi$. Besides, we can similarly get $\mathrm{x}^{K+1}(\mathrm{d};\tilde{\Theta}^{\rm total}_{K+1}) = \mathrm{P}_{(i_1i_2)} \mathrm{x}^{K+1}(\mathrm{d};{\Theta}^{\rm total}_{K+1})$. These complete the proof.
\end{proof}


We now prove Theorem~\ref{theorem: Uniform generalization error bound for discrete-time ResNet}.
\begin{proof}
	({Proof of Theorem~\ref{theorem: Uniform generalization error bound for discrete-time ResNet}})
	We divide the proof into three steps. The proof strategy is to first reduce the generalization gap to a Rademacher complexity estimate,
	and then carefully propagate this complexity through the residual layers,
	tracking both the standard Lipschitz growth and a layer-wise non-positive structural refinement term induced by the activation structure.
	
	Step 1. We first establish the boundedness of the loss function, which allows us to apply the result in \cite[Theorem 26.5]{shalev2014understanding} to derive an upper bound on the generalization gap via Rademacher complexity.
	By assumptions $(A_1)-(A_3)$ and Lemma~\ref{lemma: dis-ResNet-state-bound}, 
	the sets of state variable $\mathcal{S}_{\rm state}^L:=\{\mathrm{x}^{L}(\mathrm{d}; {\Theta}^{\rm total}_L): \  {\Theta}^{\rm total}_L \in \Omega_{{\Theta}^{\rm total}_L}, (\mathrm{d},\mathrm{g}) \sim \mathfrak{D}\}, L\ge 1$ are uniformly bounded by a constant $B_{\rm out}$ dependent on $T, B_{{\mathfrak{D}}}, B_{\pmb{\Theta}},  \mathrm{Lip}_{\psi}$.
	Therefore, the sets $\{{\ell}(\mathrm{x}^{L}(\mathrm{d}; {\Theta}^{\rm total}_L), \mathrm{g}): {\Theta}^{\rm total}_L \in \Omega_{{\Theta}^{\rm total}_L}, (\mathrm{d},\mathrm{g}) \sim \mathfrak{D} \}, L\ge 1$ are bounded by a constant $B_{\ell}$ dependent on $T, B_{{\mathfrak{D}}}, B_{\pmb{\Theta}},  \mathrm{Lip}_{\psi}$ due to the continuity of loss function.
	Hence, for any $0 < \delta < 1$, with probability $1-\delta$,
	\begin{equation}
		GE_{\mathcal{S}}(\mathrm{x}^{L}(\cdot; {\Theta}^{\rm total}_L)) \leq 2\mathscr{R}_{\mathcal{S}}({\ell}\circ {\mathcal{F}}^{L}) + 4 B_{\ell}\sqrt{\frac{2\log(4/\delta)}{S}}, \ \forall \mathrm{x}^{L}(\cdot; {\Theta}^{\rm total}_L) \in \mathcal{F}^{L},
		\label{equation: generalization error bound for dis ResNet networks}
	\end{equation}
	due to \cite[Theorem 26.5]{shalev2014understanding}. Here ${\ell}\circ {\mathcal{F}}^{L} = \{{f}^L(\cdot, \cdot) = {\ell} (\mathrm{x}^{L}(\cdot; {\Theta}^{\rm total}_L), \cdot): \ \mathrm{x}^{L}(\cdot; {\Theta}^{\rm total}_L) \in {\mathcal{F}}^{L}\}$.
	
	Step~2. We estimate the complexity $\mathscr{R}_{\mathcal{S}}({\ell}\circ {\mathcal{F}}^{L})$. Here, we exploit the Lipschitz property of the loss with respect to the network output
	to reduce the vector-valued Rademacher complexity to a sum of scalar-valued ones,
	thereby decoupling the dependence across output coordinates.

	
	By assumption ($A_4$) and the discussion in Step~1, 
	the sets $\{ {\kappa}(\mathrm{x}, \tilde{\mathrm{x}}, \mathrm{g}): \mathrm{x}, \tilde{\mathrm{x}} \in \mathcal{S}_{\rm state}^L; (\mathrm{d},\mathrm{g}) \sim \mathfrak{D} \}, L \ge 1$ are uniformly bounded by a constant $B_{\kappa}$ dependent on $T, B_{{\mathfrak{D}}}, B_{\pmb{\Theta}},  \mathrm{Lip}_{\psi}$. Hence, $\ell$ is $B_{\kappa}$-Lipschitz with respect to the first argument in $\{(\mathrm{x}^{L}(\mathrm{d}; {\Theta}^{\rm total}_L), \mathrm{g}): {\Theta}^{\rm total}_L \in \Omega_{{\Theta}^{\rm total}_L}, (\mathrm{d},\mathrm{g}) \sim \mathfrak{D}\}$, and we have, 
	\begin{align}
		\mathscr{R}_{\mathcal{S}}({\ell}\circ {\mathcal{F}}^{L}) 
		= &  \mathbb{E}_{\pmb{\varepsilon}} \sup _{ \ell \circ \mathrm{x}^{L}(\cdot; {\Theta}^{\rm total}_L) \in 	\ell \circ {\mathcal{F}}^{L}} \frac{1}{S} \sum_{s=1}^S  \varepsilon_{(s)} \cdot {\ell}(\mathrm{x}^{L}(\mathrm{d}_{(s)}; {\Theta}^{\rm total}_L), \mathrm{g}_{(s)}) \nonumber \\	
		= &  \mathbb{E}_{\pmb{\varepsilon}} \sup _{\mathrm{x}^{L}(\cdot; {\Theta}^{\rm total}_L) \in 	{\mathcal{F}}^{L}} \frac{1}{S} \sum_{s=1}^S  \varepsilon_{(s)} \cdot {\ell}(\mathrm{x}^{L}(\mathrm{d}_{(s)}; {\Theta}^{\rm total}_L), \mathrm{g}_{(s)}) \nonumber \\	
		\leq & \sqrt{2} B_{\kappa}	\mathbb{E}_{\pmb{\epsilon}} \sup _{\substack{\mathrm{x}^{L}_k(\cdot; {\Theta}^{\rm total}_L) \in 	{\mathcal{F}}^{L}_k }, 1\leq k \leq n} \frac{1}{S} \sum_{s=1}^S \sum_{k=1}^{n}  \epsilon_{(sk)} \mathrm{x}^{L}_k(\mathrm{d}_{(s)}; {\Theta}^{\rm total}_L) \nonumber \\
		\leq & \sqrt{2} B_{\kappa} \sum_{k=1}^{n}	\mathbb{E}_{\pmb{\epsilon}_{k,:}} \sup _{\substack{\mathrm{x}^{L}_k(\cdot; {\Theta}^{\rm total}_L) \in 	{\mathcal{F}}^{L}_k }} \frac{1}{S} \sum_{s=1}^S   \epsilon_{(sk)} \mathrm{x}^{L}_k(\mathrm{d}_{(s)}; {\Theta}^{\rm total}_L) \nonumber \\
		= & \sqrt{2} B_{\kappa} \sum_{k=1}^{n} \mathscr{R}_{\mathcal{S}} ({\mathcal{F}}^{L}_{k}), \label{equation: proof-GE-bound-1}
	\end{align}
	where $\pmb{\epsilon}_{k,:}=\{\epsilon_{1k}, \ldots,  \epsilon_{Sk}\}$, and the first inequality follows from \cite[Corollary 1]{maurer2016vector}.

	Step 3. We estimate the Rademacher complexity 
	$\mathscr{R}_{\mathcal{S}} ({\mathcal{F}}^{L}_{k})$ for $1\leq k \leq n$. This step is the core of the analysis.
	
	By Lemma~\ref{lemma: equivalence hypothesis set with scalar function}, it suffices to estimate $\mathscr{R}_{\mathcal{S}} ({\mathcal{F}}^{L}_{1})$.
	For simplicity, define
	$$
	\begin{aligned} 
		\mathrm{z}^{l+1}(\mathrm{d}; {\Theta}^{\rm total}_L) & =  \mathrm{V}^{l}\mathrm{x}^{l}(\mathrm{d}; {\Theta}^{\rm total}_L) + \mathrm{b}^l, 0\leq l \leq L-1.
	\end{aligned}
	$$
	By Lemma~\ref{lemma: equivalence hypothesis set with scalar function} and its proof, we can see each component of $\mathrm{z}^{l+1}(\cdot; {\Theta}^{\rm total}_L)$ (i.e., $\mathrm{z}^{l+1}_k(\cdot; {\Theta}^{\rm total}_L)$, $1\leq k \leq m$) is in the same function class
	$$
	\begin{aligned}
		\mathcal{Z}^{l+1}:= \Big\{  
		\mathrm{v}^{\top}\mathrm{x}^{l}(\cdot; {\Theta}^{\rm total}_L) + b | \   \mathrm{v} \in \mathbb{R}^{n}, b \in \mathbb{R}, 
		\max\{ \|\mathrm{v}\|_1, |b| \} \! \leq \! B_{\pmb{\Theta}}; \mathrm{x}^{l}(\cdot; {\Theta}^{\rm total}_L) \in \mathcal{F}^l  \Big \}.
	\end{aligned}
	$$
	Therefore, we can derive that
	\begin{align}
		&\mathscr{R}_{\mathcal{S}}({\mathcal{F}}^{l+1}_1) \nonumber \\
		=& \frac{1}{S}\mathbb{E}_{\pmb{\varepsilon}} \sup _{\mathrm{x}^{l+1}_1(\cdot; {\Theta}^{\rm total}_L) \in {\mathcal{F}}^{l+1}_{1}} \sum_{s=1}^S\! \varepsilon_{(s)}   \mathrm{x}^{l+1}_1(\mathrm{d}_{(s)}; {\Theta}^{\rm total}_L)\nonumber \\
		=& \frac{1}{S}\mathbb{E}_{\pmb{\varepsilon}} \sup _{\substack{\mathrm{x}^{l}(\cdot; {\Theta}^{\rm total}_L) \in {\mathcal{F}}^{l}_, \\
				\max \{|\mathrm{c}^{l}_{1}|, \|\mathrm{W}_{1,:}^{l}\|_1 \} \leq B_{\pmb{\Theta}},\\
				\max \{ \|\mathrm{V}^{l}\|_{\infty}, \|\mathrm{b}^l\|_{\infty} \} \leq B_{\pmb{\Theta}} }}
		\!\sum_{s=1}^S\! \varepsilon_{(s)} \! \Big[ \mathrm{x}^{l}_1(\mathrm{d}_{(s)}; {\Theta}^{\rm total}_L) \!+\! \tau_L \Big( \mathrm{W}_{1,:}^{l}\psi \circ   \big(\mathrm{V}^{l}\mathrm{x}^{l}(\mathrm{d}_{(s)}; {\Theta}^{\rm total}_L)\! +\! \mathrm{b}^l \big) \!+ \!\mathrm{c}^{l}_{1} \Big) \Big]  \nonumber \\
		\leq & \frac{1}{S}\mathbb{E}_{\pmb{\varepsilon}} \sup _{\substack{\mathrm{x}^{l}_1(\cdot; {\Theta}^{\rm total}_L) \in {\mathcal{F}}^{l}_{1}, \\
				\max \{|\mathrm{c}^{l}_{1}|, \|\mathrm{W}_{1,:}^{l}\|_1 \} \leq B_{\pmb{\Theta}},\\
				\mathrm{z}^{l+1}(\cdot; {\Theta}^{\rm total}_L) \in (\mathcal{Z}^{l+1})^m }} 
		\sum_{s=1}^S \varepsilon_{(s)} \Big[ \mathrm{x}^{l}_1(\mathrm{d}_{(s)}; {\Theta}^{\rm total}_L) + \tau_L \Big( \mathrm{W}_{1,:}^{l}\psi \circ   (\mathrm{z}^{l+1}(\mathrm{d}_{(s)}; {\Theta}^{\rm total}_L)) + \mathrm{c}^{l}_{1} \Big) \Big]  \nonumber \\
		\leq & \mathscr{R}_{\mathcal{S}}({\mathcal{F}}^{l}_{1}) + \tau_L\frac{B_{\pmb{\Theta}}}{\sqrt{S}} +  \tau_L 2 B_{\pmb{\Theta}} \frac{1}{S}\mathbb{E}_{\pmb{\varepsilon}} \sup _{ \mathrm{z}^{l+1}_1(\cdot; {\Theta}^{\rm total}_L)\in \mathcal{Z}^{l+1}}
		\sum_{s=1}^S  \varepsilon_{(s)}  \psi \big( \mathrm{z}_1^{l+1}(\mathrm{d}_{(s)}; {\Theta}^{\rm total}_L) \big) \nonumber \\
		\leq & \mathscr{R}_{\mathcal{S}}({\mathcal{F}}^{l}_{1}) + \tau_L\frac{B_{\pmb{\Theta}}}{\sqrt{S}} +  \tau_L  \mathrm{Lip}_{\psi} 2 B_{\pmb{\Theta}} \frac{1}{S} \mathbb{E}_{\pmb{\varepsilon}} \! \sup _{\substack{ \mathrm{x}^{l}(\cdot; {\Theta}^{\rm total}_L) \in {\mathcal{F}}^{l}, \\
				\max \{ \|\mathrm{V}^{l}_{1,:}\|_1, |\mathrm{b}^l_1| \} \leq B_{\pmb{\Theta}} }}
		\! \sum_{s=1}^S  \varepsilon_{(s)} \Big[ \mathrm{V}^{l}_{1,:}\mathrm{x}^{l}(\mathrm{d}_{(s)}; {\Theta}^{\rm total}_L) + \mathrm{b}^l_1 \Big] \nonumber \\
		& -  \tau_L\frac{2B_{\pmb{\Theta}}( \mathrm{Lip}_{\phi_1}\alpha\! +\!  \mathrm{Lip}_{\phi_2}\beta)}{S} \bar{C}^{l+1}_{\mathcal{S}}  \nonumber \\ 
		\leq &  \mathscr{R}_{\mathcal{S}}({\mathcal{F}}^{l}_{1}) + \tau_L\frac{B_{\pmb{\Theta}}}{\sqrt{S}} +  \tau_L  2 \mathrm{Lip}_{\psi}  B_{\pmb{\Theta}} \Big( \frac{B_{\pmb{\Theta}}}{\sqrt{S}} +   B_{\pmb{\Theta}} \mathscr{R}_{\mathcal{S}}({\mathcal{F}}^{l}_{1}) \Big)  -  \tau_L\frac{2 B_{\pmb{\Theta}}( \mathrm{Lip}_{\phi_1}\alpha\! +\!  \mathrm{Lip}_{\phi_2}\beta)}{S} \bar{C}^{l+1}_{\mathcal{S}} \nonumber \\
		\leq &  \mathscr{R}_{\mathcal{S}}({\mathcal{F}}^{l}_{1}) \!+ \! \tau_L  2 \mathrm{Lip}_{\psi}  B_{\pmb{\Theta}}^2 \mathscr{R}_{\mathcal{S}}({\mathcal{F}}^{l}_{1})  \!+\!  \max \Big\{\tau_L\frac{B_{\pmb{\Theta}}}{\sqrt{S}} \Big( 1 \!+ \! 2 \mathrm{Lip}_{\psi} B_{\pmb{\Theta}} \Big) \!-\! \tau_L 2 B_{\pmb{\Theta}}\frac{( \mathrm{Lip}_{\phi_1}\alpha\! +\!  \mathrm{Lip}_{\phi_2}\beta) \bar{C}^{l+1}_{\mathcal{S}}}{S}, 0  \Big\}, \nonumber
	\end{align} 
	where the second equality follows from Eq.(\ref{equation: discrete-time ResNet}); the second inequality uses a similar proof of \cite[Theorem 12]{bartlett2002rademacher} and \cite[Lemma 4]{shultzman2023generalization}; the third inequality follows from Proposition~\ref{proposition: Rademacher complexity property of ac function} and $0 \leq  \bar{C}_{\mathcal{S}}^{l+1} \leq S$ is dependent on data ${\mathcal{S}}$; and the forth inequality follows from Lemma~\ref{lemma: equivalence hypothesis set with scalar function}, \cite[Lemma 4]{shultzman2023generalization} together with the symmetry of the class $\{\mathrm{V}^{l}_{1,:}\mathrm{x}^{l}(\mathrm{d}_{(s)}; {\Theta}^{\rm total}_L): \ \mathrm{x}^{l}(\cdot; {\Theta}^{\rm total}_L) \in {\mathcal{F}}^{l}, \|\mathrm{V}^{l}_{1,:}\|_1 \leq B_{\pmb{\Theta}} \}$.

	Let ${C}_{\mathcal{S}}^{l+1} = \min \{ \sqrt{S} \frac{1 +  2\mathrm{Lip}_{\psi} B_{\pmb{\Theta}}}{2 (\mathrm{Lip}_{\phi_1}\alpha +  \mathrm{Lip}_{\phi_2}\beta)}, \bar{C}_{\mathcal{S}}^{l+1} \}$, $0 \leq l \leq L-1$. This truncation is introduced to ensure that the coefficients in the above inequality remain nonnegative, enabling us to use discrete Gr{\"o}nwall inequality. Then we get
	\begin{equation}
		\mathscr{R}_{\mathcal{S}}({\mathcal{F}}^{l+1}_1) \! \leq \! \mathscr{R}_{\mathcal{S}}({\mathcal{F}}^{l}_{1})\! +\!  \tau_L 2 \mathrm{Lip}_{\psi} B_{\pmb{\Theta}}^2 \mathscr{R}_{\mathcal{S}}({\mathcal{F}}^{l}_{1}) \! + \! \tau_L\frac{B_{\pmb{\Theta}}}{\sqrt{S}} \Big( 1 \!+ \! 2 \mathrm{Lip}_{\psi} B_{\pmb{\Theta}} \Big) \!-\! \tau_L B_{\pmb{\Theta}} \frac{( \mathrm{Lip}_{\phi_1}\alpha\! +\!  \mathrm{Lip}_{\phi_2}\beta) {C}_{\mathcal{S}}^{l+1}}{S},
		\nonumber
	\end{equation}
	where $0 \leq l \leq L-1$ and $0 \leq {C}_{\mathcal{S}}^{l+1} \leq \min \{ \sqrt{S} \frac{1 + 2\mathrm{Lip}_{\psi} B_{\pmb{\Theta}}}{2( \mathrm{Lip}_{\phi_1}\alpha +  \mathrm{Lip}_{\phi_2}\beta)}, S \}$. By applying Gr\"onwall's inequality \cite{emmrich1999discrete} to the inequality above, together with $\tau_L \cdot L = T$, we get
	\begin{equation}
		\begin{aligned}
			\mathscr{R}_{\mathcal{S}}({\mathcal{F}}^{L}_1) 
			\leq & \Big( \mathscr{R}_{\mathcal{S}}({\mathcal{F}}^{0}_1) + \frac{\tau_L L \cdot B_{\pmb{\Theta}}}{\sqrt{S}}  (1 + 2 \mathrm{Lip}_{\psi} B_{\pmb{\Theta}}) \Big)\exp(\tau_L L \cdot 2 \mathrm{Lip}_{\psi} B_{\pmb{\Theta}}^2) \\
			& \ -  \frac{B_{\pmb{\Theta}}}{\sqrt{S}} \exp(\tau_L L \cdot 2 \mathrm{Lip}_{\psi} B_{\pmb{\Theta}}^2) \tau_L  \sum_{l=1}^L \big(\frac{( \mathrm{Lip}_{\phi_1}\alpha\! +\!  \mathrm{Lip}_{\phi_2}\beta) {C}_{\mathcal{S}}^{l}}{\sqrt{S}} \big) \\
			= & \Big( \mathscr{R}_{\mathcal{S}}({\mathcal{F}}^{0}_1) + \frac{TB_{\pmb{\Theta}}}{\sqrt{S}}  (1 +  2 \mathrm{Lip}_{\psi} B_{\pmb{\Theta}}) \Big)\exp(2T  \mathrm{Lip}_{\psi} B_{\pmb{\Theta}}^2) \\
			& \ -  \frac{B_{\pmb{\Theta}}}{\sqrt{S}} \exp(2T  \mathrm{Lip}_{\psi} B_{\pmb{\Theta}}^2) \tau_L  \sum_{l=1}^L \big(\frac{( \mathrm{Lip}_{\phi_1}\alpha\! +\!  \mathrm{Lip}_{\phi_2}\beta) {C}_{\mathcal{S}}^{l}}{\sqrt{S}} \big) .
		\end{aligned}
		\nonumber
	\end{equation}
	Due to the lemma in \cite[Lemma~26.11]{shalev2014understanding} and Eq.\eqref{equation: lip property of Rademacher of activation old}, we see
	$$
	\begin{aligned}
		\mathscr{R}_{\mathcal{S}}({\mathcal{F}}^{0}_1) = \mathbb{E}_{\pmb{\varepsilon}} \sup _{ \max\{ \|\mathrm{U}_{1,:}\|_{1}, |\mathrm{a}_1| \} \leq  B_{\pmb{\Theta}} } \frac{1}{S} \sum_{s=1}^S  \varepsilon_{(s)} \cdot \psi \!\circ \!(\mathrm{U}_{1,:}\mathrm{d}_{(s)} \!+\! \mathrm{a}_1) 
		\! \leq \!   \mathrm{Lip}_{\psi} B_{\pmb{\Theta}} \frac{B_{{\mathfrak{D}}}\sqrt{2\log(2n_{\mathrm{d}})} \!+ \!1}{\sqrt{S}}.
	\end{aligned}
	$$
	It then follows from the above two inequalities that
	\begin{equation}
		\begin{aligned}
			\mathscr{R}_{\mathcal{S}}({\mathcal{F}}^{L}_1) 
			\leq & \frac{B_{\pmb{\Theta}}}{\sqrt{S}} \Big( \mathrm{Lip}_{\psi} B_{{\mathfrak{D}}}\sqrt{2\log(2n_{\mathrm{d}})} + 1 + T \cdot (1 +  2\mathrm{Lip}_{\psi} B_{\pmb{\Theta}}) \Big)\exp(2T \mathrm{Lip}_{\psi} B_{\pmb{\Theta}}^2) \\
			& -  \frac{B_{\pmb{\Theta}}}{\sqrt{S}} \exp(2T  \mathrm{Lip}_{\psi} B_{\pmb{\Theta}}^2) \tau_L  \sum_{l=1}^L \frac{( \mathrm{Lip}_{\phi_1}\alpha\! +\!  \mathrm{Lip}_{\phi_2}\beta) {C}_{\mathcal{S}}^{l}}{\sqrt{S}} \\
			\leq & \frac{B_{\pmb{\Theta}}}{\sqrt{S}} \Big( \mathrm{Lip}_{\psi} B_{{\mathfrak{D}}}\sqrt{2\log(2n_{\mathrm{d}})} + 1 + T \cdot (1 +  2\mathrm{Lip}_{\psi} B_{\pmb{\Theta}}) \Big)\exp(2T \mathrm{Lip}_{\psi} B_{\pmb{\Theta}}^2) \\
			& -  \frac{B_{\pmb{\Theta}}( \mathrm{Lip}_{\phi_1}\alpha\! +\!  \mathrm{Lip}_{\phi_2}\beta) \exp(2T  \mathrm{Lip}_{\psi} B_{\pmb{\Theta}}^2)}{S} T {C}_{\mathcal{S},L} ,
		\end{aligned}
		\label{equation: proof-GE-bound-2}
	\end{equation}
	where ${C}_{\mathcal{S},L} = \frac{1}{L} \sum_{l=1}^L {C}_{\mathcal{S}}^l$. Combining Eq.(\ref{equation: generalization error bound for dis ResNet networks}), Eq.(\ref{equation: proof-GE-bound-1}), and Eq.(\ref{equation: proof-GE-bound-2}) and taking the supremum over the set ${\mathcal{F}}^{L}$, we complete the proof.  
\end{proof}

\subsection{Proof of Theorem~\ref{theorem: Uniform generalization error bound for continuous-time ResNet}}
\label{subsec: Proof for Theorem-2}
In this subsection, we prove Theorem~\ref{theorem: Uniform generalization error bound for continuous-time ResNet}. 
Our approach is to transform the generalization problem of continuous-time ResNet into estimating the Rademacher complexity of discrete-time ResNet and to apply a discrete-to-continuum convergence property. 
As in \cite{huang2024on}, we use the following definition to connect the discrete- and continuous-time learnable parameters.
\begin{definition} 
	\label{definition: extension operator}
	Let $\mathcal{X}$ be a vector space. Partition the interval $[0,T]$ into $L$ subintervals $\{[t_L^{l-1}, t_L^l)\}_{l=1}^{L-1} \cup [t_L^{L-1}, t_L^L]$, where $t_L^l := l \cdot T / L$.
	We define a sampling operator $\boldsymbol{\mathcal{T}}_L : \mathcal{C}([0,T]; \mathcal{X}) \rightarrow \mathcal{X}^L$ for $\boldsymbol{\Xi} \in \mathcal{C}([0,T]; \mathcal{X})$ by
	$$
	\boldsymbol{\mathcal{T}}_L(\boldsymbol{\Xi}) = \big( \boldsymbol{\Xi}(t_L^0), \ldots, \boldsymbol{\Xi}(t_L^{L-1}) \big),
	$$
	and  define a piecewise constant extension operator $\bar{\boldsymbol{\mathcal{I}}}_L : \mathcal{X}^L \rightarrow \mathcal{M}_L([0,T]; \mathcal{X})$ for any $\mathrm{\Xi}_L = (\mathrm{\xi}_L^0, \mathrm{\xi}_L^1, \ldots, \mathrm{\xi}_L^{L-1}) \in \mathcal{X}^L$ by
	$$
	(\bar{\boldsymbol{\mathcal{I}}}_L \mathrm{\Xi}_L)(t) =  \sum_{l=0}^{L-2} \mathrm{\xi}_L^l \boldsymbol{1}_{[t_L^l, t_L^{l+1})}(t) + \mathrm{\xi}_L^{L-1} \boldsymbol{1}_{[t_L^{L-1}, t_L^L]}(t),
	$$
	where $\mathcal{M}_L([0,T]; \mathcal{X})$ denotes the space of piecewise constant functions on $[0,T]$ with partition $\{[t_L^{l-1}, t_L^l)\}_{l=1}^{L-1} \cup [t_L^{L-1}, t_L^L]$.
\end{definition}
Using the above definition, we derive a convergence property between the state variables of discrete- and continuous-time ResNets.

\begin{lemma} 
	\label{lemma: state convergence}
	Let assumptions $(A_1)-(A_3)$ hold and denote $t_L^l = l \tau_L = l \cdot \frac{T}{L},0 \leq l \leq L$. Suppose
	that $ \pmb{\Theta}^{\rm total} =({\Theta}^{\rm pre}, {\pmb{\Theta}}) \in \Omega_{\pmb{\Theta}^{\rm total}}$ and the input $\mathrm{d}$ is the first argument of a sample from  $\mathfrak{D}$.
	Then there exists a parameter ${\Theta}^{\rm total}_L=({\Theta}^{\rm pre}, {{\Theta}}_L) \in \Omega_{{\Theta}^{\rm total}_L}$ and a constant $C_1>0$ such that
	$$\| \bar{\boldsymbol{\mathcal{I}}}_{L} {{\Theta}}_L - \pmb{\Theta} \|_{\mathcal{L}^{2}([0,T]; \mathcal{E})} \leq C_1 \tau_L^{1/2} \rightarrow 0 \text{ as } L \rightarrow \infty. $$
	Moreover, for this ${\Theta}^{\rm total}_L$, there exists a constant $C_2 > 0$ such that
	\begin{equation}
		\sup_{ 1 \leq l \leq L } \sup_{t\in [t_{L}^{l-1}, t_{L}^{l}]} \|\mathrm{x}^{l}(\mathrm{d}; {\Theta}^{\rm total}_L) - \bold{x}(t; \mathrm{d}; \pmb{\Theta}^{\rm total}) \|_{\infty} \leq C_2 \tau_L^{1/2} \rightarrow 0 \ {\rm as } \ L\rightarrow \infty.
		\label{equation: forward convergence}
	\end{equation}
\end{lemma}
\begin{proof} The proof primarily follows the approach in \cite{thorpe2018deep, huang2024on}.
	For every $ \pmb{\Theta}^{\rm total} =({\Theta}^{\rm pre}, {\pmb{\Theta}}) \in \Omega_{\pmb{\Theta}^{\rm total}}$, let ${\Theta}^{\rm total}_L= ({\Theta}^{\rm pre}, {{\Theta}}_L)$, where ${{\Theta}}_L=({{\Theta}}_L^{0}, {{\Theta}}_L^{1}, \ldots, {{\Theta}}_L^{L-1})  := \boldsymbol{\mathcal{T}}_L(\pmb{\Theta})$. By the definition of $\Omega_{\pmb{\Theta}^{\rm total}}$ and a similar proof of Morrey's inequality \cite[Remark 11.35]{leoni2017first}, we have for every $t, \hat{t} \in [0,T]$, $\|\pmb{\Theta}(t) - \pmb{\Theta}(\hat{t})\|_{\infty} \leq \bar{C}_1 |t - \hat{t}|^{1/2} \|D(\pmb{\Theta})\|_{\mathcal{L}^2([0,T]; \mathcal{E})} \leq \bar{C}_1 B_{\pmb{\Theta}}|t - \hat{t}|^{1/2}$.
	Then
	\begin{equation}
		\begin{aligned}
			\| \bar{\boldsymbol{\mathcal{I}}}_{L}\pmb{\Theta}_{L} - \pmb{\Theta} \|_{\mathcal{L}^{2}([0,T];\mathcal{E})}^{2} & = \sum_{l=0}^{L-1} \int_{t_{L}^{l}}^{t_{L}^{l+1}}  \|\pmb{\Theta}(t_{L}^{l}) - \pmb{\Theta}(t)\|_{\infty}^{2} dt  \\
			& \leq \sum_{l=0}^{L-1} \int_{t_{L}^{l}}^{t_{L}^{l+1}} \bar{C}_1^2 B_{\pmb{\Theta}}^2 |t_L^l - {t}| dt   \leq C_1^2 \tau_L  \rightarrow 0,\ L \rightarrow \infty,
		\end{aligned}
		\label{inequality: state convergence proof-1}
	\end{equation}
	where $C_1 = \sqrt{\frac{1}{2} \bar{C}_1^2 B_{\pmb{\Theta}}^2 T}$.
	
	We next prove (\ref{equation: forward convergence}) for a such ${\Theta}^{\rm total}_L$. 
	Let $e_L^l = \|\mathrm{x}^{l}(\mathrm{d}; {\Theta}^{\rm total}_L) - \bold{x}(t_L^l; \mathrm{d}; \pmb{\Theta}^{\rm total}) \|_{\infty}$.
	Since the ResNets defined in Eq.(\ref{equation: discrete-time ResNet}) and Eq.(\ref{equation: continuous-time ResNet}) are special cases of the ResNet architectures considered in \cite[Eq.(5) and Eq.(7)]{huang2024on}, it follows from the proof of \cite{thorpe2018deep, huang2024on} that 
	\begin{equation}
		\begin{aligned}
			\sup_{0 \leq l \leq L} e^l_L \leq & \frac{\gamma_{L} M_{2}}{2}T \Big(1+\frac{M_{1}T}{L} \Big)^{2L} + \frac{M_{2}}{2\gamma_{L}}\|  \pmb{\Theta} - \bar{\boldsymbol{\mathcal{I}}}_{L}\pmb{\Theta}_{L}\|^{2}_{\mathcal{L}^{2}([0,T];\mathcal{E})} \! + \! M_3 \frac{T^{2}}{L}\Big(1+\frac{M_{1}T}{L}\Big)^{L},
			\label{inequality: state convergence proof-2}
		\end{aligned}
	\end{equation}
	where $M_1, M_2, M_3$ are constants depending on $B_{\pmb{\Theta}}, B_{{\mathfrak{D}}}, T,  \mathrm{Lip}_{\psi}$, and $\gamma_L > 0$ can be any constant. Taking $\gamma_{L} = \| \bar{\boldsymbol{\mathcal{I}}}_{L}\pmb{\Theta}_{L} - \pmb{\Theta} \|_{\mathcal{L}^{2}([0,T];\mathcal{E})}$ and by Eq.\eqref{inequality: state convergence proof-1}, we get 
	$$
	\sup_{0 \leq l \leq L} \| \mathrm{x}^{l}(\mathrm{d}; {\Theta}^{\rm total}_L) - \bold{x}(t_L^l; \mathrm{d}; \pmb{\Theta}^{\rm total})  \|_{\infty} \leq M_4 \tau_L^{1/2},
	$$
	where $M_4$ is a constant dependent on $B_{\pmb{\Theta}}, B_{{\mathfrak{D}}}, T,  \mathrm{Lip}_{\psi}$.
	
	On the other hand, 
	\begin{equation}
		\begin{aligned}
			\sup_{t \in [0,T]} \Big\| \frac{\mathrm{d}\bold{x}(t; \mathrm{d}; \pmb{\Theta}^{\rm total})}{\mathrm{d} t}   \Big\|_{\infty} 
			=& \sup_{t \in [0,T]} \|\bold{W}(t)\psi \circ (\bold{V}(t)\bold{x}(t; \mathrm{d}; \pmb{\Theta}^{\rm total}) + \bold{b}(t)) + \bold{c}(t) \|_{\infty} \\
			\leq &  \mathrm{Lip}_{\psi} B_{\pmb{\Theta}}^2(B_{\rm out} +1) + B_{\pmb{\Theta}}=: \mathrm{Lip}_{\bold{x}},
		\end{aligned}
		\label{inequality: state convergence proof-3}
	\end{equation}
	where $B_{\rm out}$ is a constant given in Lemma~\ref{lemma: dis-ResNet-state-bound}.
	Therefore, combining Eq.(\ref{inequality: state convergence proof-2}) and Eq.(\ref{inequality: state convergence proof-3}) gives
	\begin{equation}
		\begin{aligned}
			& \sup_{t\in [t_{L}^{l-1}, t_{L}^{l}]} \| \bold{x}(t; \mathrm{d}; \pmb{\Theta}^{\rm total}) - \mathrm{x}^{l}(\mathrm{d}; {\Theta}^{\rm total}_L) \|_{\infty} \\
			\leq & \sup_{t\in [t_{L}^{l-1}, t_{L}^{l}]} \big( \| \bold{x}(t; \mathrm{d}; \pmb{\Theta}^{\rm total}) - \bold{x}(t_L^l; \mathrm{d}; \pmb{\Theta}^{\rm total}) \|_{\infty} + \| 
			\bold{x}(t_L^l; \mathrm{d}; \pmb{\Theta}^{\rm total})
			-\mathrm{x}^{l}(\mathrm{d}; {\Theta}^{\rm total}_L) \|_{\infty} \big)  \\
			\leq & \mathrm{Lip}_{\bold{x}} \tau_L + M_4 \tau_L^{1/2}.
		\end{aligned}
	\end{equation}
	Letting $C_2 = \mathrm{Lip}_{\bold{x}} \sqrt{T} + M_4$ completes the proof.
\end{proof}

We now prove Theorem~\ref{theorem: Uniform generalization error bound for continuous-time ResNet}.
\begin{proof} ({Proof of Theorem~\ref{theorem: Uniform generalization error bound for continuous-time ResNet}})
	We divide the proof into two steps.
	
	Step 1. We show the boundedness of the loss function and apply the theorem in \cite[Theorem 26.5]{shalev2014understanding} to get a Rademacher complexity upper bound of the generalization gap. 
	Similar to the discussion in the proof of Theorem~\ref{theorem: Uniform generalization error bound for discrete-time ResNet},
	the loss function $\ell$ is bounded by a constant $B_{\ell}$ dependent on $T, B_{{\mathfrak{D}}}, B_{\pmb{\Theta}},  \mathrm{Lip}_{\psi}$. Therefore, for any $0 < \delta < 1$, with probability $1-\delta$ that
	\begin{equation}
		GE_{\mathcal{S}}(\bold{x}(T;\cdot; \pmb{\Theta}^{\rm total})) \leq 2\mathscr{R}_{\mathcal{S}}({\ell}\circ {\mathcal{F}}^{T}) + 4B_{\ell} \sqrt{\frac{2\log(4/\delta)}{S}}, \ \forall \ \bold{x}(T;\cdot; \pmb{\Theta}^{\rm total}) \in \mathcal{F}^T,
		\label{equation: generalization error bound for ResNet networks}
	\end{equation}
	due to \cite[Theorem 26.5]{shalev2014understanding}. Here ${\ell}\circ {\mathcal{F}}^{T} = \{{f}^T(\cdot, \cdot) = {\ell} (\bold{x}(T; \cdot; \pmb{\Theta}^{\rm total}), \cdot): \ \bold{x}(T;\cdot; \pmb{\Theta}^{\rm total}) \in 	{\mathcal{F}}^{T}\}$.

	Step 2. We estimate the Rademacher complexity $\mathscr{R}_{\mathcal{S}}({\ell}\circ {\mathcal{F}}^{T})$. Similar to the discussion in Step~2 of the proof of Theorem~\ref{theorem: Uniform generalization error bound for discrete-time ResNet}, we have
	\begin{equation}
		\begin{aligned}
			\mathscr{R}_{\mathcal{S}}({\ell}\circ {\mathcal{F}}^{T}) = &  \frac{1}{S} \mathbb{E}_{\pmb{\varepsilon}} \sup _{\ell \circ \bold{x}(T;\cdot; \pmb{\Theta}^{\rm total}) \in 	\ell \circ{\mathcal{F}}^{T}}  \sum_{s=1}^S  \varepsilon_{(s)} \cdot {\ell}(\bold{x}(T; \mathrm{d}_{(s)};\pmb{\Theta}^{\rm total}), \mathrm{g}_{(s)})  \\
			= &  \frac{1}{S} \mathbb{E}_{\pmb{\varepsilon}} \sup _{\bold{x}(T;\cdot; \pmb{\Theta}^{\rm total}) \in 	{\mathcal{F}}^{T}}  \sum_{s=1}^S  \varepsilon_{(s)} \cdot {\ell}(\bold{x}(T; \mathrm{d}_{(s)};\pmb{\Theta}^{\rm total}), \mathrm{g}_{(s)})  \\	
			\leq & \sqrt{2}B_{\kappa}  \frac{1}{S} \mathbb{E}_{\pmb{\epsilon}} \sup _{\substack{\bold{x}_k(T; \cdot; \pmb{\Theta}^{\rm total}) \in 	{\mathcal{F}}^{T}_{k}}, 1\leq k \leq n} \sum_{s=1}^S\sum_{k=1}^{n}  \epsilon_{(sk)} \cdot \bold{x}_k(T;\mathrm{d}_{(s)};\pmb{\Theta}^{\rm total}) \\	  
			\leq & \sqrt{2}B_{\kappa} \sum_{k=1}^{n} \mathbb{E}_{\pmb{\epsilon}_{k,:}} \sup _{\substack{\bold{x}_k(T; \cdot; \pmb{\Theta}^{\rm total}) \in 	{\mathcal{F}}^{T}_{k}}} \frac{1}{S} \sum_{s=1}^S  \epsilon_{(sk)} \cdot \bold{x}_k(T;\mathrm{d}_{(s)};\pmb{\Theta}^{\rm total}) \\	  
			= & \sqrt{2}B_{\kappa} \sum_{k=1}^{n}	\mathscr{R}_{\mathcal{S}}(\mathcal{F}^{T}_{k}) ,  
		\end{aligned}
		\label{equation: proof-continuous-time-GE-bound-1}
	\end{equation}
	with $\pmb{\epsilon}_{k,:}=\{\epsilon_{1k}, \ldots,  \epsilon_{Sk}\}$,  where the first inequality is due to \cite[Corollary 1]{maurer2016vector}.
	
	To apply the result in Theorem~\ref{theorem: Uniform generalization error bound for discrete-time ResNet}, we let $t_L^l = l \tau_L = l \cdot \frac{T}{L}$, $0 \leq l \leq L$.
	Then there are constants $c_1, c_2$ independent of $L$ such that 
	$$
	\|\bar{\boldsymbol{\mathcal{I}}}_L \boldsymbol{\mathcal{T}}_L(\pmb{\Theta})-\pmb{\Theta}\|^2_{\mathcal{L}^{2}([0,T];\mathcal{E})} \leq c_1 \tau_L, \forall \,  \pmb{\Theta} \in \Omega_{\pmb{\Theta}},
	$$
	and also
	\begin{equation}
		\begin{aligned}
			\|\bold{x}(T; \mathrm{d}; \pmb{\Theta}^{\rm total})  - \mathrm{x}^{L}(\mathrm{d}; ({\Theta}^{\rm pre},\boldsymbol{\mathcal{T}}_L(\pmb{\Theta})))
			\|_{\infty} \leq c_2 \tau^{\frac{1}{2}}_L, \forall \mathrm{d} \in \mathbb{R}^{n_\mathrm{d}}, \forall \, \pmb{\Theta} \in \Omega_{\pmb{\Theta}}, 
		\end{aligned}
		\label{inequality: continuous-time proof-1}
	\end{equation}
	by using Lemma~\ref{lemma: state convergence}. Hence, noting $({\Theta}^{\rm pre},\boldsymbol{\mathcal{T}}_L(\pmb{\Theta})) \in \Omega_{{\Theta}^{\rm total}_L}$ for all $({\Theta}^{\rm pre}, \pmb{\Theta}) \in \Omega_{\pmb{\Theta}}$ and applying Eq.(\ref{equation: proof-GE-bound-2}) and Eq.(\ref{inequality: continuous-time proof-1}), we get
	\begin{align}
		\mathscr{R}_{\mathcal{S}} ({\mathcal{F}}^{T}_{k}) 
		=& \frac{1}{S}\mathbb{E}_{\pmb{\varepsilon}} \sup _{
			\bold{x}_k(T; \cdot; \pmb{\Theta}^{\rm total}) \in {\mathcal{F}}^{T}_{k} } \sum_{s=1}^{S} \varepsilon_{(s)} \bold{x}_k(T;\mathrm{d}_{(s)};\pmb{\Theta}^{\rm total})  \nonumber \\
		\leq & \frac{1}{S}\mathbb{E}_{\pmb{\varepsilon}} \Big[ \sup _{ \pmb{\Theta}^{\rm total} \in \Omega_{\pmb{\Theta}}^{\rm total} } \sum_{s=1}^{S} \varepsilon_{(s)} \Big(\bold{x}_k 
		\big(T;\mathrm{d}_{(s)};\pmb{\Theta}^{\rm total} \big) - \mathrm{x}_k^{L}\big(\mathrm{d}_{(s)}; ({\Theta}^{\rm pre},\boldsymbol{\mathcal{T}}_L(\pmb{\Theta}))\big)\Big) \Big]  \nonumber \\
		&  + \frac{1}{S}\mathbb{E}_{\pmb{\varepsilon}} \sup _{
			\pmb{\Theta}^{\rm total} \in \Omega_{\pmb{\Theta}}^{\rm total} } \sum_{s=1}^{S} \varepsilon_{(s)} \mathrm{x}_k^{L}(\mathrm{d}_{(s)}; ({\Theta}^{\rm pre},\boldsymbol{\mathcal{T}}_L(\pmb{\Theta})) )  \nonumber \\
		\leq & {c_2\sqrt{\tau_L}}  + \frac{1}{S}\mathbb{E}_{\pmb{\varepsilon}} \sup _{
			\mathrm{x}^{L}_k(\cdot; {\Theta}^{\rm total}_L) \in \mathcal{F}_k^L } \sum_{s=1}^{S} \varepsilon_{(s)} \mathrm{x}_k^{L}(\mathrm{d}_{(s)}; {\Theta}^{\rm total}_L)  \nonumber \\
		= & {c_2\sqrt{\tau_L}}  + \mathscr{R}_{\mathcal{S}} ({\mathcal{F}}^{L}_{k})  \nonumber \\ 
		\leq &  {c_2\sqrt{\tau_L}}  + \frac{B_{\pmb{\Theta}}}{\sqrt{S}} \Big(\mathrm{Lip}_{\psi}B_{{\mathfrak{D}}}\sqrt{2\log(2n_{\mathrm{d}})} + 1 + T  (1 +  2\mathrm{Lip}_{\psi} B_{\pmb{\Theta}}) \Big)\exp(2T  \mathrm{Lip}_{\psi} B_{\pmb{\Theta}}^2)  \nonumber \\
		& -  \frac{B_{\pmb{\Theta}}( \mathrm{Lip}_{\phi_1}\alpha\! +\!  \mathrm{Lip}_{\phi_2}\beta) \exp(2T  \mathrm{Lip}_{\psi} B_{\pmb{\Theta}}^2)}{S} T {C}_{\mathcal{S},L} , \nonumber
	\end{align}
	where $1\leq k \leq n$, $0 \leq {C}_{\mathcal{S},L} \leq \min \{ \sqrt{S} \frac{1 +  2\mathrm{Lip}_{\psi} B_{\pmb{\Theta}}}{2 (\mathrm{Lip}_{\phi_1}\alpha +  \mathrm{Lip}_{\phi_2}\beta)}, S \}$.

	Since $\tau_L=\frac{T}{L}$, letting $L \to \infty$, possibly along a subsequence, we conclude that there exists a constant $C_{\mathcal{S}}$ dependent on data ${\mathcal{S}}$ such that
	$$
	\begin{aligned}
		\mathscr{R}_{\mathcal{S}} ({\mathcal{F}}^{T}_{k}) \leq & \frac{B_{\pmb{\Theta}}}{\sqrt{S}} \Big(\mathrm{Lip}_{\psi}B_{{\mathfrak{D}}}\sqrt{2\log(2n_{\mathrm{d}})} + 1 + T (1 +  \mathrm{Lip}_{\psi} B_{\pmb{\Theta}}) \Big)\exp(T  \mathrm{Lip}_{\psi} B_{\pmb{\Theta}}^2) \\
		& -  \frac{B_{\pmb{\Theta}}( \mathrm{Lip}_{\phi_1}\alpha\! +\!  \mathrm{Lip}_{\phi_2}\beta)\exp(T  \mathrm{Lip}_{\psi} B_{\pmb{\Theta}}^2) }{S} TC_{\mathcal{S}}, 
	\end{aligned}   
	$$
	where $1 \leq k \leq n$, $0 \leq C_{\mathcal{S}} \leq \min \big\{ \sqrt{S} \frac{1 +  \mathrm{Lip}_{\psi} B_{\pmb{\Theta}}}{ \mathrm{Lip}_{\phi_1}\alpha +  \mathrm{Lip}_{\phi_2}\beta}, S \big \}$.
	Combing the above inequality with Eq.(\ref{equation: generalization error bound for ResNet networks}), Eq.(\ref{equation: proof-continuous-time-GE-bound-1}), and taking the supremum over the set $\mathcal{F}^T$, we complete the proof.
\end{proof}

\section{Experiments}
\label{Sec: 5}
In this section, we conduct numerical experiments on image classification tasks to evaluate the generalization ability of ResNets within the dynamical system modeling framework.  Specifically, we examine how the generalization behaves by varying the number of neural network layers and the number of training samples. We also investigate the benefits of activation functions with the decomposition structure on the generalization ability of ResNets.
Like lots of image classification applications of ResNet architectures, here we adopt the common setting with four blocks like \cite{he2016deep}, where each block consists of $L$ layers of (\ref{equation: discrete-time ResNet}).
Since the expected risk in \eqref{equation: discrete-time generalization error} is incomputable, we adopt the following approximation scheme. Assume we have $S_{\mathfrak{D}}$ i.i.d. samples from (data, label) distribution $\mathfrak{D}$, where $S_{\mathfrak{D}} \gg S$. Let the training and testing samples in experiments be $\mathcal{S} = \{ (\mathrm{d}_{(s)}, \mathrm{g}_{(s)} )\}_{s=1}^{S}$ and $\bar{\mathcal{S}} = \{ (\mathrm{d}_{(s)}, \mathrm{g}_{(s)} )\}_{s=S+1}^{S + \bar{S}}$, respectively. 
Following the definition of the generalization gap in (\ref{equation: discrete-time generalization error}), we derive
\begin{align}
	GE_{\mathcal{S}}\big(\mathrm{x}^{L}(\cdot; {\Theta}^{\rm total}_L)\big) = & \mathfrak{R}_{{\mathfrak{D}}}\big(\mathrm{x}^{L}(\cdot; {\Theta}^{\rm total}_L)\big) - \hat{\mathfrak{R}}_{\mathcal{S}}\big(\mathrm{x}^{L}(\cdot; {\Theta}^{\rm total}_L)\big)  \nonumber \\
	{\approx} & \frac{1}{S_{\mathfrak{D}}} \sum_{s=1}^{S_{\mathfrak{D}}}{\ell} \big(\mathrm{x}^{L}(\mathrm{d}_{(s)};\pmb{\Theta}^{\rm total}),  \mathrm{g}_{(s)}\big) - \frac{1}{S} \sum_{s=1}^{S}{\ell} \big(\mathrm{x}^{L}(\mathrm{d}_{(s)};\pmb{\Theta}^{\rm total}),  \mathrm{g}_{(s)}\big) \nonumber  \\
	= & - \!\frac{S_{\mathfrak{D}} \! - \! S}{S_{\mathfrak{D}}} \frac{1}{S}  \! \sum_{s=1}^{S} \! {\ell} \big(\mathrm{x}^{L}(\mathrm{d}_{(s)}; \! \pmb{\Theta}^{\rm total}),  \mathrm{g}_{(s)}\big)  \! +  \! \frac{S_{\mathfrak{D}}  \!- \! S}{S_{\mathfrak{D}}}  \frac{1}{S_{\mathfrak{D}}  \!-  \!S}  \!\sum_{s=S \!+ \!1}^{S_{\mathfrak{D}}}  \! {\ell} \big(\mathrm{x}^{L}(\mathrm{d}_{(s)}; \! \pmb{\Theta}^{\rm total}),  \mathrm{g}_{(s)}\big) \nonumber \\
	\approx & \frac{S_{\mathfrak{D}} - S}{S_{\mathfrak{D}}} \big(\hat{\mathfrak{R}}_{\rm test} - \hat{\mathfrak{R}}_{\rm train} \big) \\
	\approx & \hat{\mathfrak{R}}_{\rm test} - \hat{\mathfrak{R}}_{\rm train}, \nonumber
\end{align}
where the testing loss $\hat{\mathfrak{R}}_{\rm test}  \!= \! \frac{1}{\bar{S}} \sum_{s=S+1}^{\bar{S} +S}\! {\ell}(\mathrm{x}^{L}(\mathrm{d}_{(s)};\pmb{\Theta}^{\rm total}), \! \mathrm{g}_{(s)}) \approx \frac{1}{S_{\mathfrak{D}} \!-\! S} \sum_{s=S+1}^{S_{\mathfrak{D}}}{\ell} \big(\mathrm{x}^{L}(\mathrm{d}_{(s)};\!\pmb{\Theta}^{\rm total}),\!  \mathrm{g}_{(s)}\big)$ and  training loss $\hat{\mathfrak{R}}_{\rm train} \! = \!\hat{\mathfrak{R}}_{\mathcal{S}}\big(\mathrm{x}^{L}(\cdot; {\Theta}^{\rm total}_L)\big) \! = \! \frac{1}{S} \sum_{s=1}^{S}{\ell} (\mathrm{x}^{L}(\mathrm{d}_{(s)};\pmb{\Theta}^{\rm total}),  \mathrm{g}_{(s)})$, by using the assumption of i.i.d. samples.
This derivation shows that the difference between the testing and training losses can approximate the generalization gap (\ref{equation: discrete-time generalization error}). Therefore, we use such empirical generalization gap $\hat{\mathfrak{R}}_{\rm test} \! - \! \hat{\mathfrak{R}}_{\rm train}$ to numerically validate the results in Theorem~\ref{theorem: Uniform generalization error bound for discrete-time ResNet} and Theorem~\ref{theorem: Uniform generalization error bound for continuous-time ResNet}.

All ResNets are trained for 40 epochs on the SVHN dataset and 120 epochs on the CIFAR10 and CIFAR100 datasets using an SGD optimizer with an initial learning rate of 0.01 and a momentum of 0.9.
All training and testing procedures are implemented using PyTorch on NVIDIA GeForce RTX 3090 GPUs. Our code is available at ``\href{https://github.com/Huangjsh15/Dynamical-ResNet.git}{https://github.com/Huangjsh15/Dynamical-ResNet.git}''.

\subsection{A test on the asymptotic behavior of the generalization ability of ResNets along the number of layers} 
We first investigate how the generalization ability of ResNets behaves under the framework of dynamical system modeling. In particular, we conduct experiments by progressively increasing the layer number $L$ of the ResNet. The empirical generalization gaps and top-1 image classification accuracies versus training epochs on the SVHN, CIFAR10, and CIFAR100 datasets are plotted in Figure~\ref{fig: layer refine train} and Figure~\ref{fig: layer refine test}, respectively.
The results show that as $L$ increases, both empirical generalization gap and testing accuracy of the discrete-time ResNet tend to converge, which is consistent with our theoretical predictions. We also observe that the empirical generalization gap curves and classification accuracy curves for ResNets with $(T,L)=(6,12),(6,24)$ largely overlap, as well as those with $(T,L)=(8,16),(8,32)$. This phenomenon suggests that the networks are already approaching their limiting regime as the depth increases. Therefore, we can also regard the experiments of ResNets with $(T,L)=(6,24)$ and $(T,L)=(8,32)$ as a good approximation to the generalization behavior of the corresponding continuous-time model. 

\begin{figure*}[h]
	\centering
	\begin{tabular}{ccc}
		\begin{minipage}{0.25\textwidth}
			\centering
			\includegraphics[width=\linewidth]{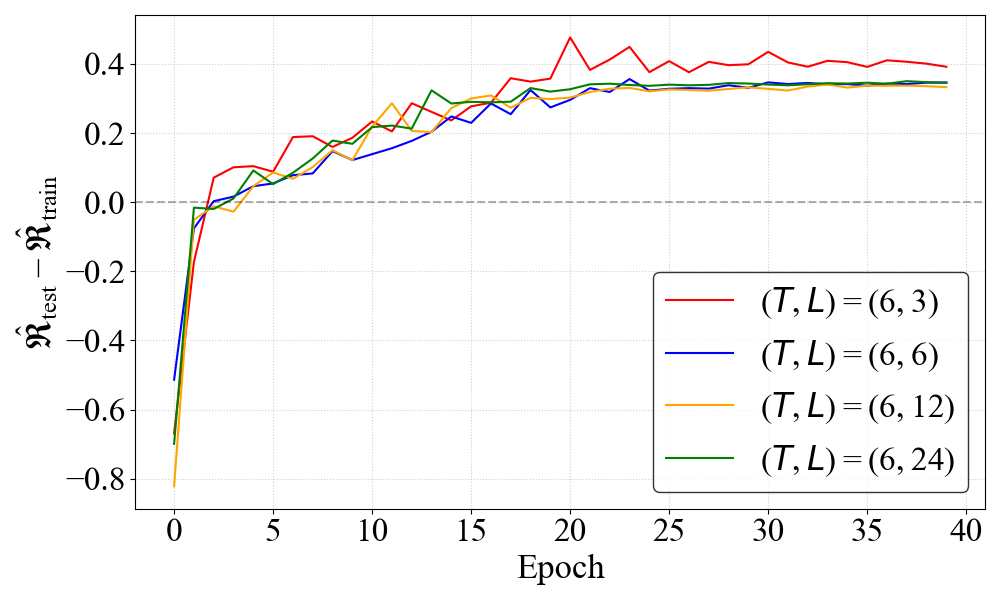}
			\\ \small (a1) SVHN
		\end{minipage} &
		\begin{minipage}{0.25\textwidth}
			\centering
			\includegraphics[width=\linewidth]{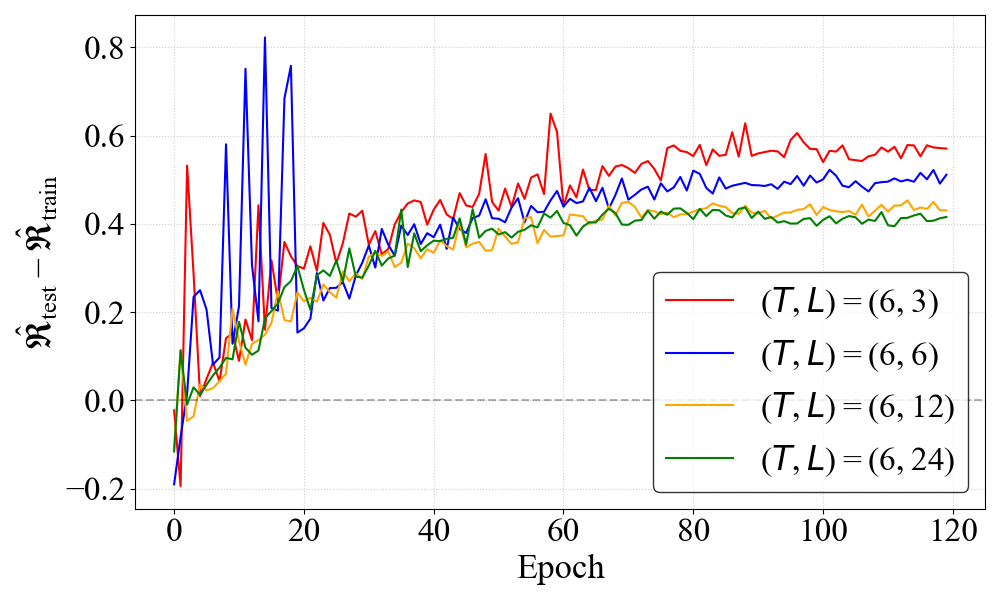}
			\\ \small (b1) CIFAR10
		\end{minipage} &
		\begin{minipage}{0.25\textwidth}
			\centering
			\includegraphics[width=\linewidth]{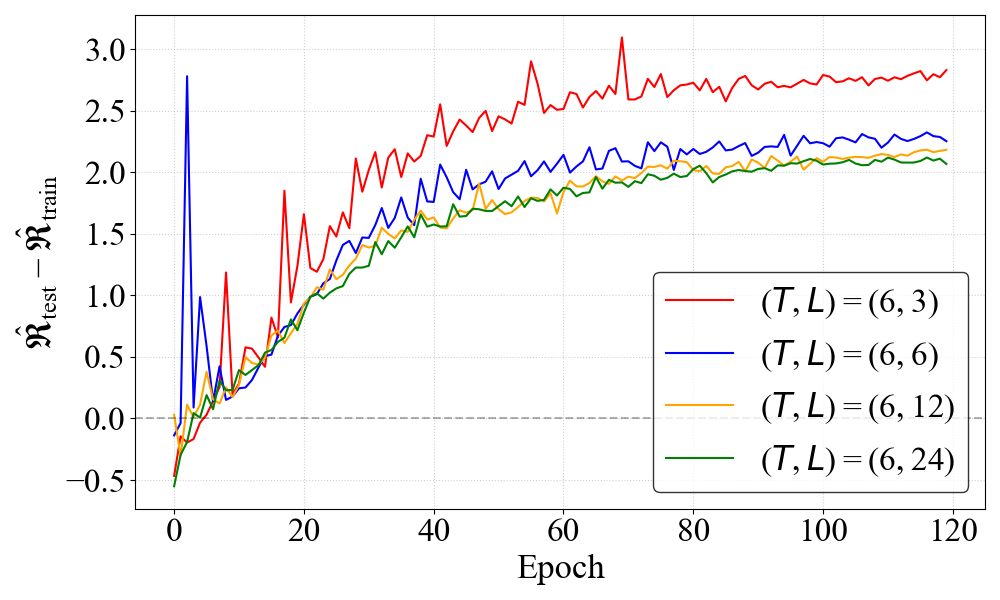}
			\\ \small (c1) CIFAR100
		\end{minipage} \\[7pt] 
		\begin{minipage}{0.25\textwidth}
			\centering
			\includegraphics[width=\linewidth]{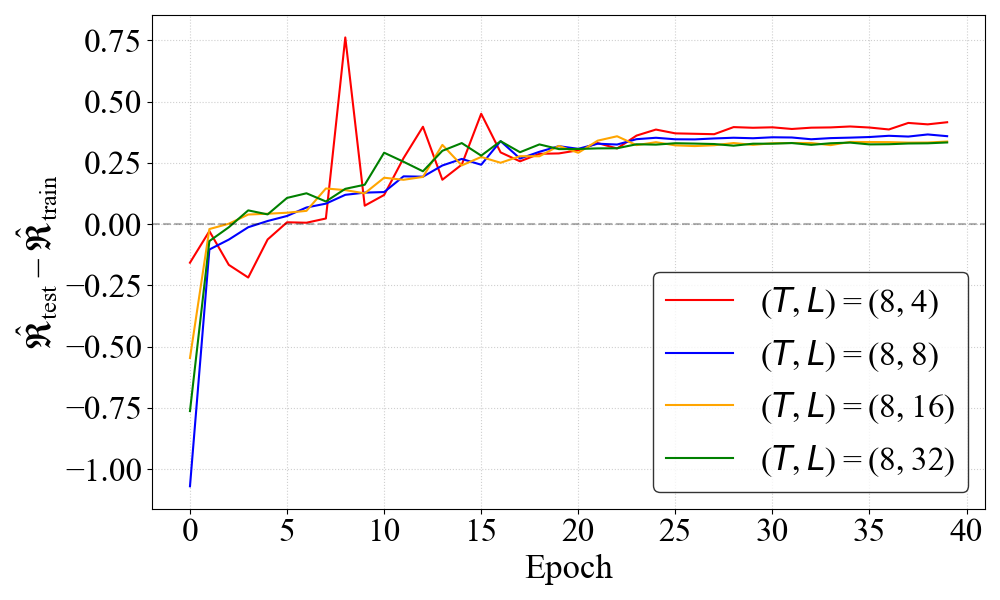}
			\\ \small (a2) SVHN
		\end{minipage} &
		\begin{minipage}{0.25\textwidth}
			\centering
			\includegraphics[width=\linewidth]{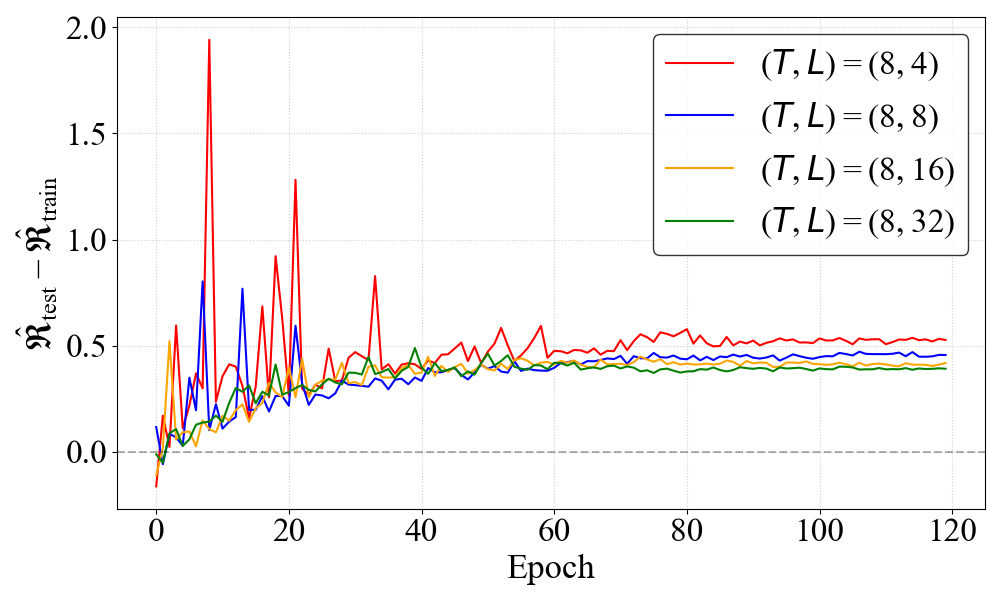}
			\\ \small (b2) CIFAR10
		\end{minipage} &
		\begin{minipage}{0.25\textwidth}
			\centering
			\includegraphics[width=\linewidth]{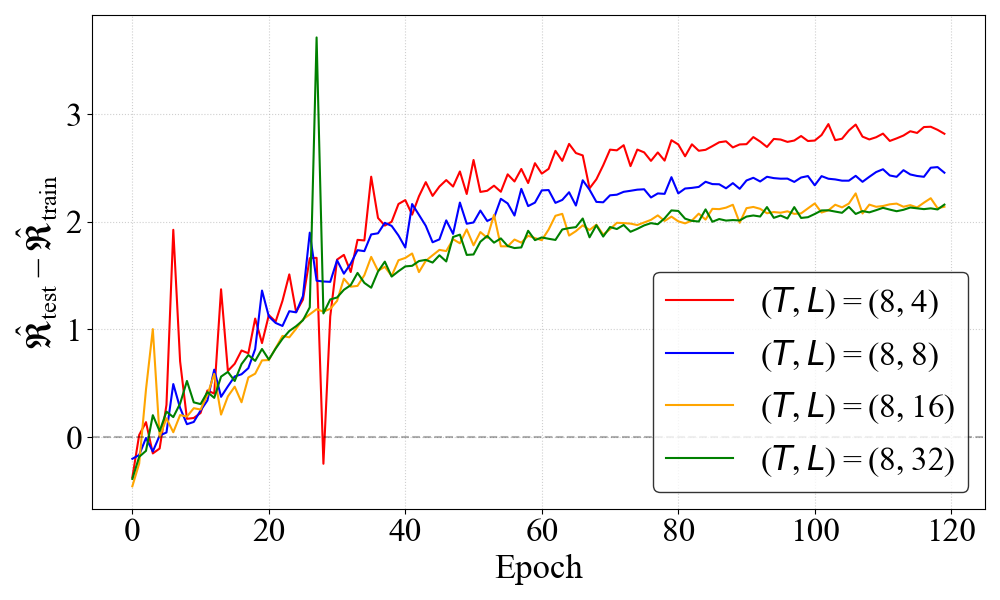}
			\\ \small (c2) CIFAR100
		\end{minipage} \\
	\end{tabular}
	\caption{ Empirical generalization gap $\hat{\mathfrak{R}}_{\rm test}-\hat{\mathfrak{R}}_{\rm train}$ versus training epoch for ResNets with $T=6, 8$ and increasing number of layers.
		The first row corresponds to $T=6$ with $L = 3, 6, 12, 24$, and the second row corresponds to $T=8$ with $L = 4, 8, 16, 32$.
		Over the tested datasets SVHN, CIFAR10 and CIFAR100, the generalization gap of ResNets does not increase with the number of layers and shows a convergent trend, which aligns with the theoretical results.}
	\label{fig: layer refine train}
\end{figure*}

\begin{figure}[htbp]
	\centering
	\begin{tabular}{ccc}
		\begin{minipage}{0.24\textwidth}
			\centering
			\includegraphics[width=\linewidth]{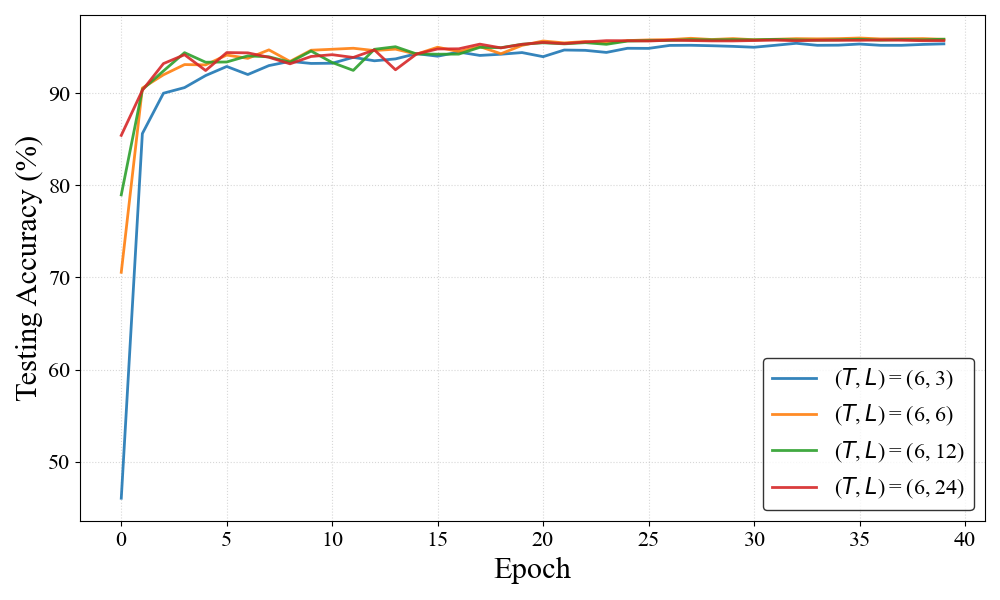}
			\\ \small (a1) SVHN
		\end{minipage} &
		\begin{minipage}{0.24\textwidth}
			\centering
			\includegraphics[width=\linewidth]{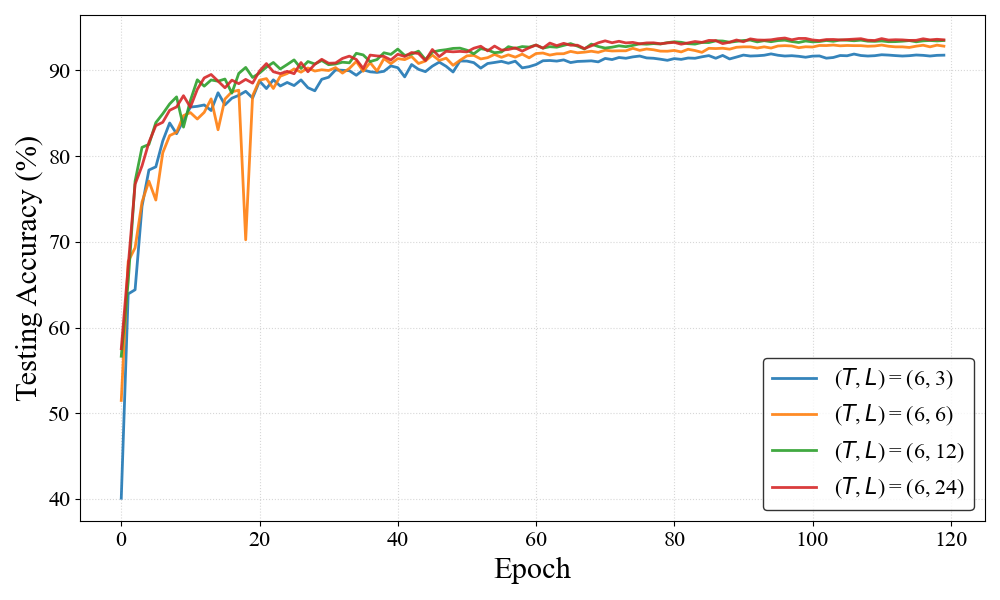}
			\\ \small (b1) CIFAR10
		\end{minipage} &
		\begin{minipage}{0.24\textwidth}
			\centering
			\includegraphics[width=\linewidth]{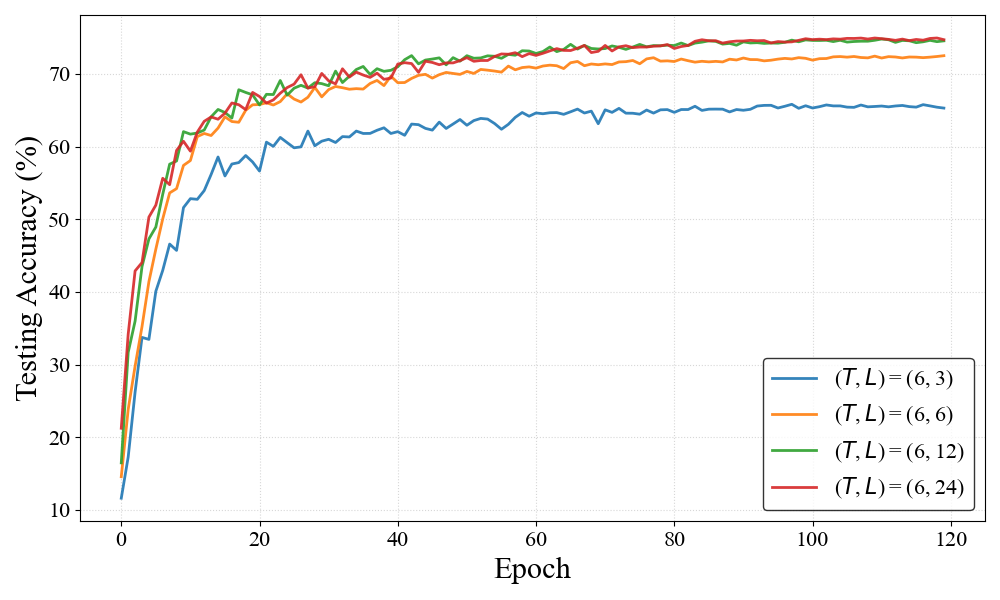}
			\\ \small (c1) CIFAR100
		\end{minipage} \\[7pt] 
		\begin{minipage}{0.24\textwidth}
			\centering
			\includegraphics[width=\linewidth]{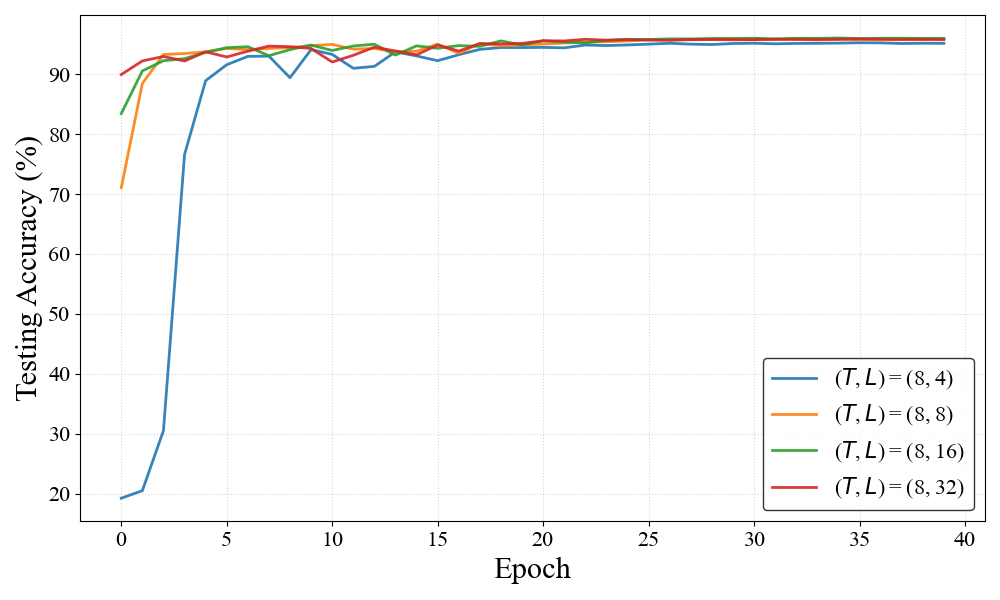}
			\\ \small (a2) SVHN
		\end{minipage} &
		\begin{minipage}{0.24\textwidth}
			\centering
			\includegraphics[width=\linewidth]{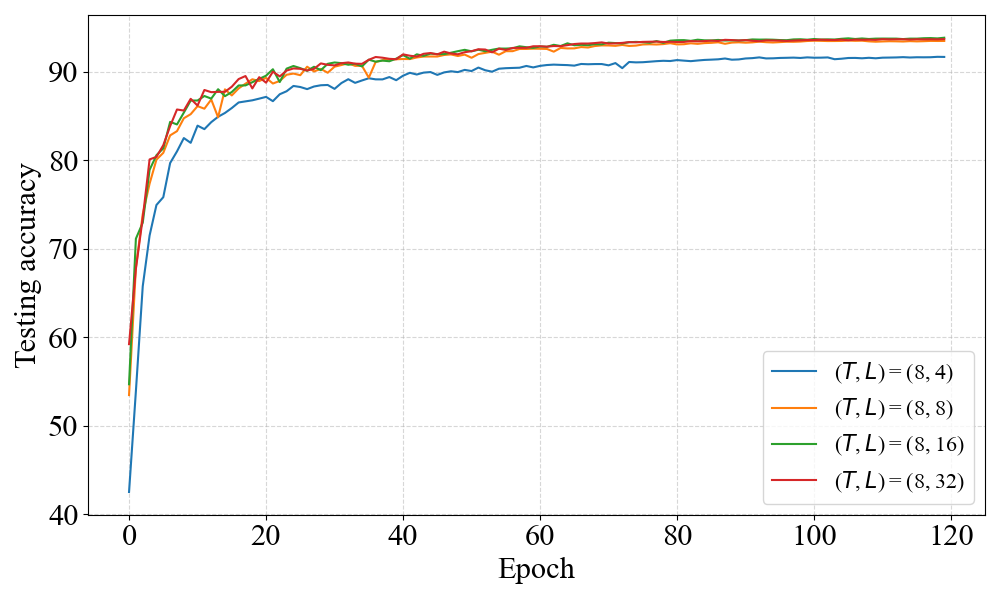}
			\\ \small (b2) CIFAR10
		\end{minipage} &
		\begin{minipage}{0.24\textwidth}
			\centering
			\includegraphics[width=\linewidth]{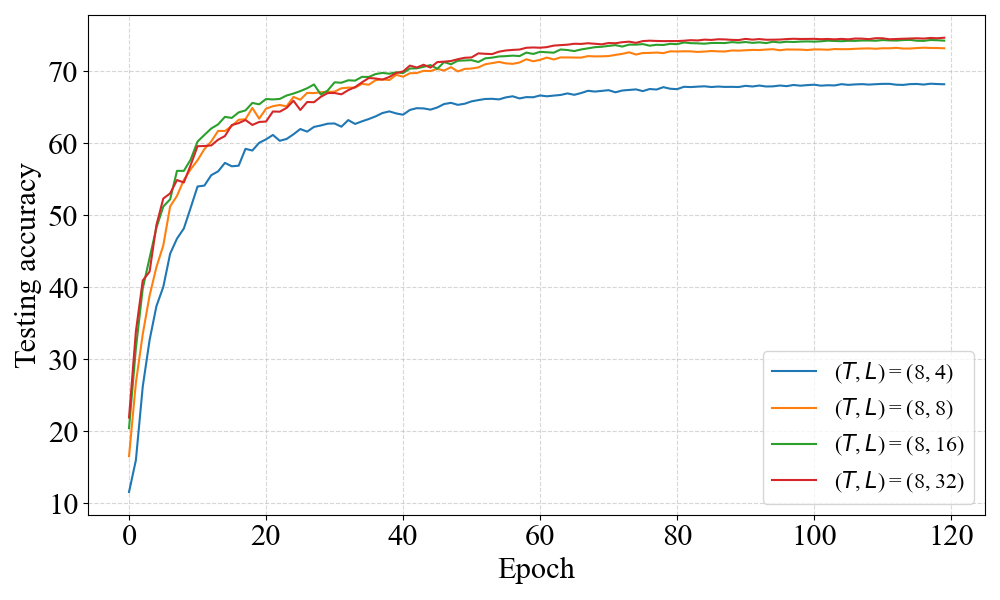}
			\\ \small (c2) CIFAR100
		\end{minipage} \\
	\end{tabular}
	\caption{
		Top-1 image classification accuracy of ResNets with $T=6, 8$ and increasing number of layers.
		The first row corresponds to $T=6$ with $L = 3, 6, 12, 24$, and the second row corresponds to $T=8$ with $L = 4, 8, 16, 32$.
		Over the tested datasets SVHN, CIFAR10 and CIFAR100, the testing accuracy exhibits a convergent trend as the number of layers increases, supporting the depth-stability predicted by our theoretical analysis.}
	\label{fig: layer refine test}
\end{figure}

\begin{table}[htbp]
	\begin{minipage}{\textwidth}
		\centering
		\caption{Top-1 image classification testing accuracy by trained ResNets with $(T, L) = (6, 6), (6,24)$ and $(T, L) = (8, 8), (8, 32)$ on SVHN, CIFAR10 and CIFAR100 datasets with different training sample sizes \( S \).}
		\begin{tabular}{m{1.2cm}<{\centering}m{1.8cm}<{\centering}m{1.5cm}<{\centering}m{1.5cm}<{\centering}m{1.5cm}<{\centering}m{1.5cm}<{\centering}m{1.5cm}<{\centering}}
			\toprule 
			\multicolumn{7}{c}{Testing accuracy}  \\ 
			\midrule 
			$(T, L)$ & $S$ & $10^4$ & $2\cdot10^4$ & $3\cdot10^4$ & $4\cdot10^4$ & $5\cdot10^4$ \\    
			\midrule 
			\multirow{3}{*}{(6, 6)} 
			& \text{SVHN}  &  92.935 & 94.473 & 95.026 & 95.33 & 95.849  \\
			& \text{CIFAR10}  & 84.375 & 89.083 & 90.865 & 92.348 & 92.839  \\
			& $\text{CIFAR100}$ & 49.189 & 59.165 & 66.216 & 68.139 & 72.516   \\
			\hline
			\multirow{3}{*}{(6, 24)} 
			& \text{SVHN}  & 92.881 & 94.557 & 94.899 & 95.28 & 95.706  \\
			& \text{CIFAR10}  & 84.896 & 89.694 & 91.526 & 92.728 & 93.58  \\
			& $\text{CIFAR100}$ & 53.405 & 64.323 & 69.591 & 72.596 & 74.72  \\
			\hline
			\multirow{3}{*}{(8, 8)} 
			& \text{SVHN}  & 92.847 & 94.238 & 94.892 & 95.449 & 95.806  \\
			& \text{CIFAR10}  & 84.143 & 89.329 & 90.995 & 92.474 & 93.5  \\
			& $\text{CIFAR100}$ & 49.211 & 61.793 & 66.743 & 69.455 & 73.177  \\
			\hline
			\multirow{3}{*}{(8, 32)} 
			& \text{SVHN}  & 92.589 & 94.492 & 95.191 & 95.418 & 95.81  \\
			& \text{CIFAR10}  & 85.116 & 89.463 & 91.797 & 92.768 & 93.82   \\
			& $\text{CIFAR100}$ & 53.486 & 64.193 & 69.321 & 71.855 & 74.509 \\
			\bottomrule
		\end{tabular}
		\label{tab: test acc}
	\end{minipage}
\end{table}

\begin{figure}[t]
	\centering
	\begin{tabular}{cccc}
		
		\begin{minipage}{0.22\textwidth}
			\centering
			\includegraphics[width=\linewidth]{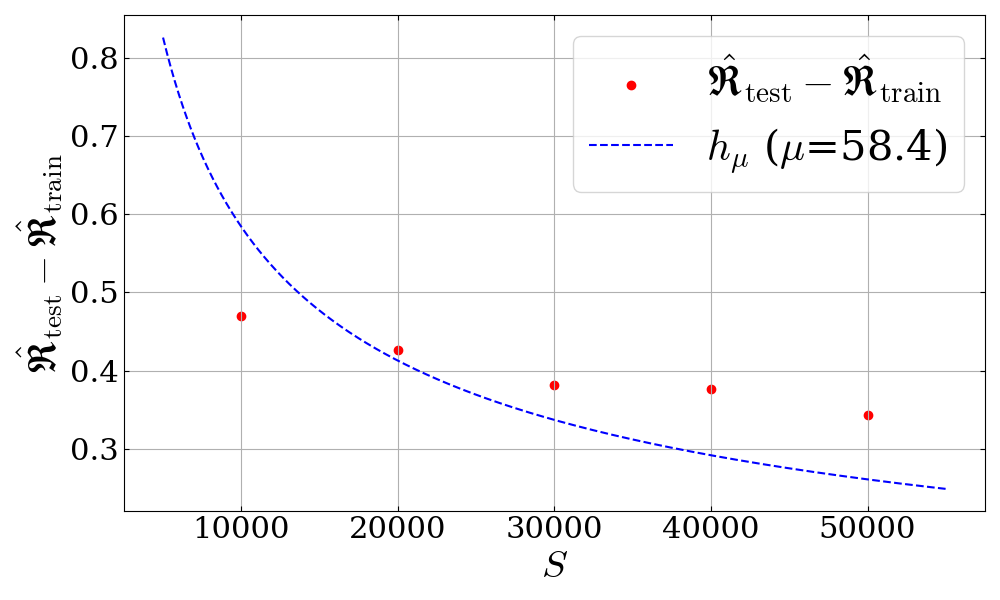}
			\\ \small (a1) $(T, L)$=(6, 6)
		\end{minipage} &
		
		\begin{minipage}{0.22\textwidth}
			\centering
			\includegraphics[width=\linewidth]{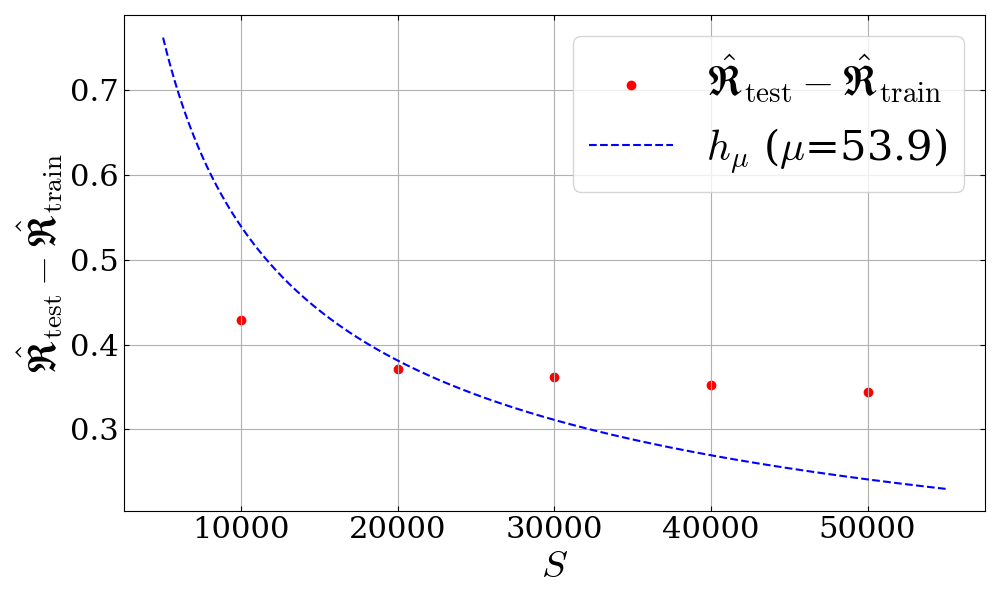}
			\\ \small  (a2)  $(T, L)$=(6, 24)
		\end{minipage} 
		
		\begin{minipage}{0.22\textwidth}
			\centering
			\includegraphics[width=\linewidth]{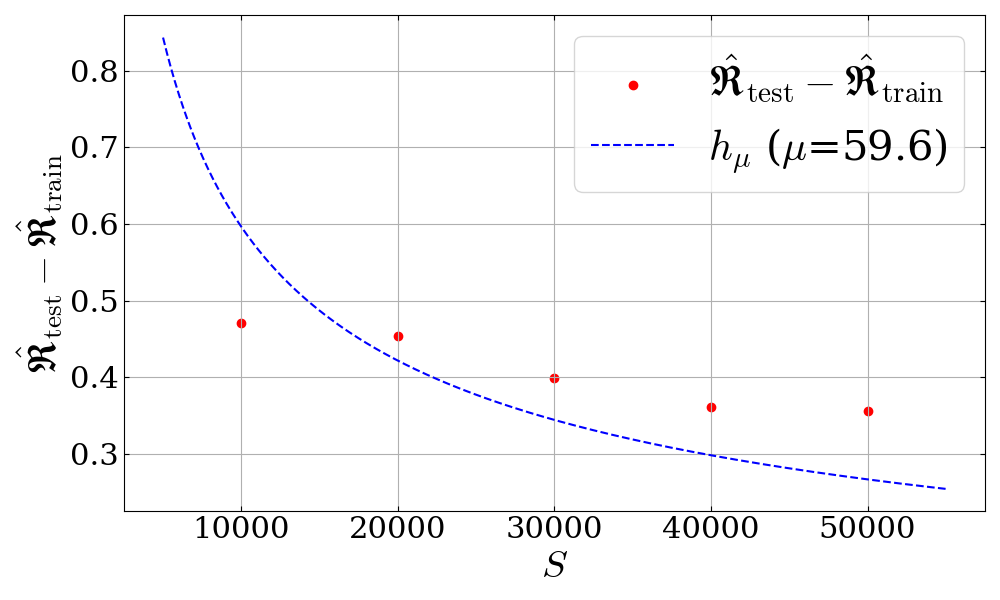}
			\\ \small  (a3) $(T, L)$=(8, 8)
		\end{minipage} &
		
		\begin{minipage}{0.22\textwidth}
			\centering
			\includegraphics[width=\linewidth]{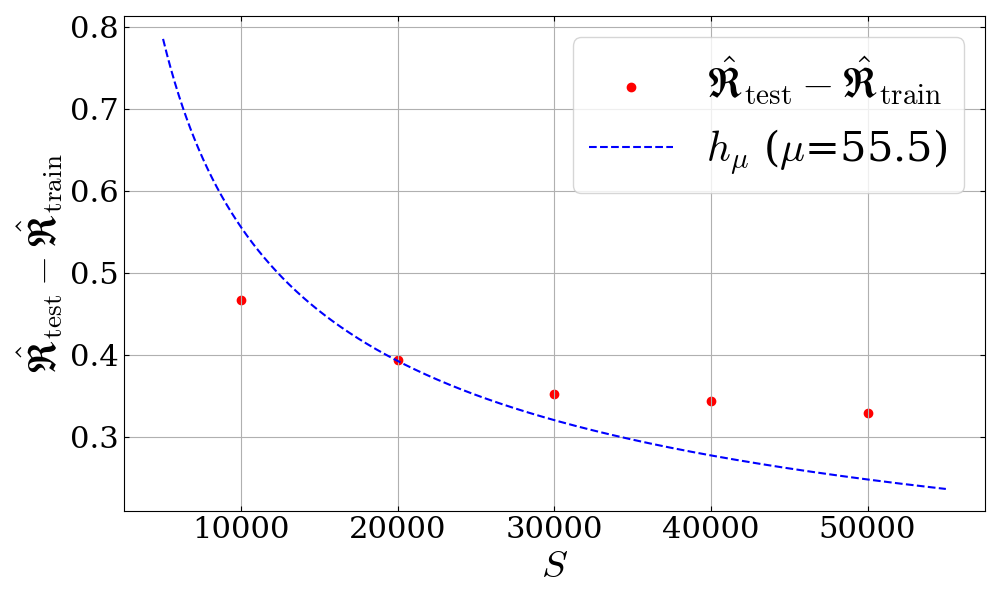}
			\\ \small  (a4)  $(T, L)$=(8, 32)
		\end{minipage} &
		
		\\
		& & & \\ 
		
		\begin{minipage}{0.22\textwidth}
			\centering
			\includegraphics[width=\linewidth]{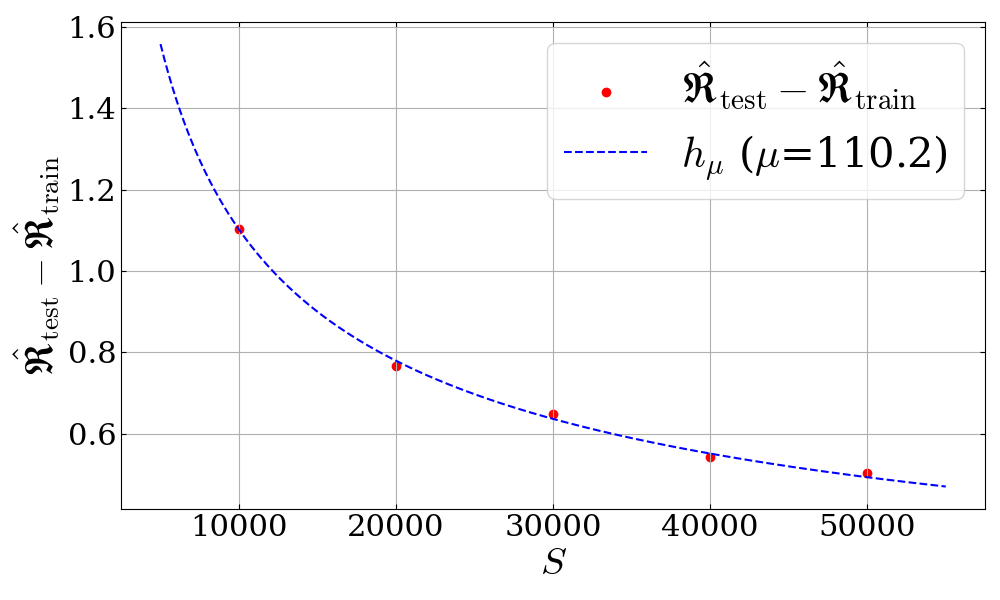}
			\\ \small (b1)  $(T, L)$=(6, 6)
		\end{minipage} &
		
		\begin{minipage}{0.22\textwidth}
			\centering
			\includegraphics[width=\linewidth]{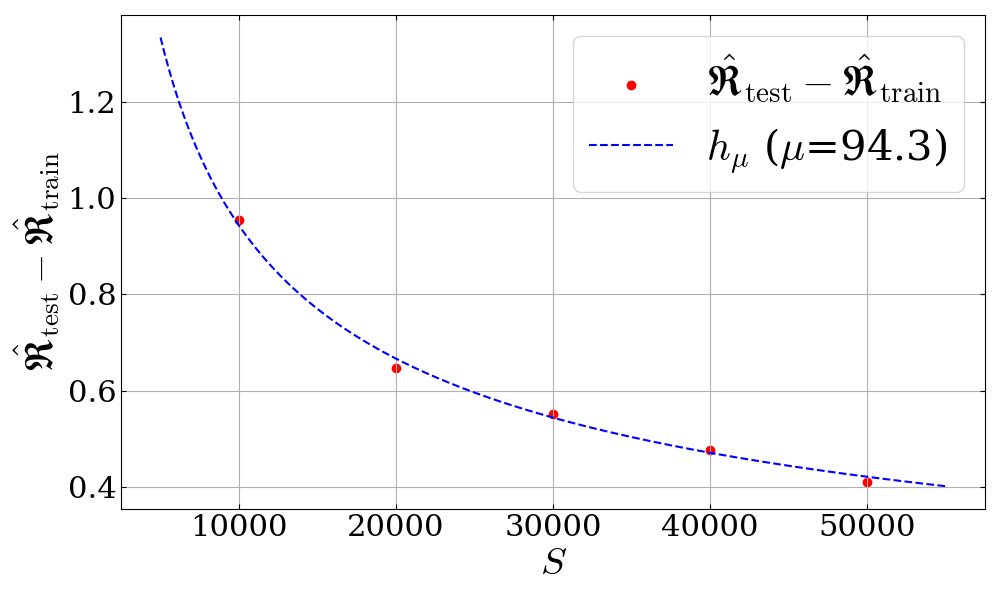}
			\\ \small  (b2) $(T, L)$=(6, 24)
		\end{minipage} 
		
		\begin{minipage}{0.22\textwidth}
			\centering
			\includegraphics[width=\linewidth]{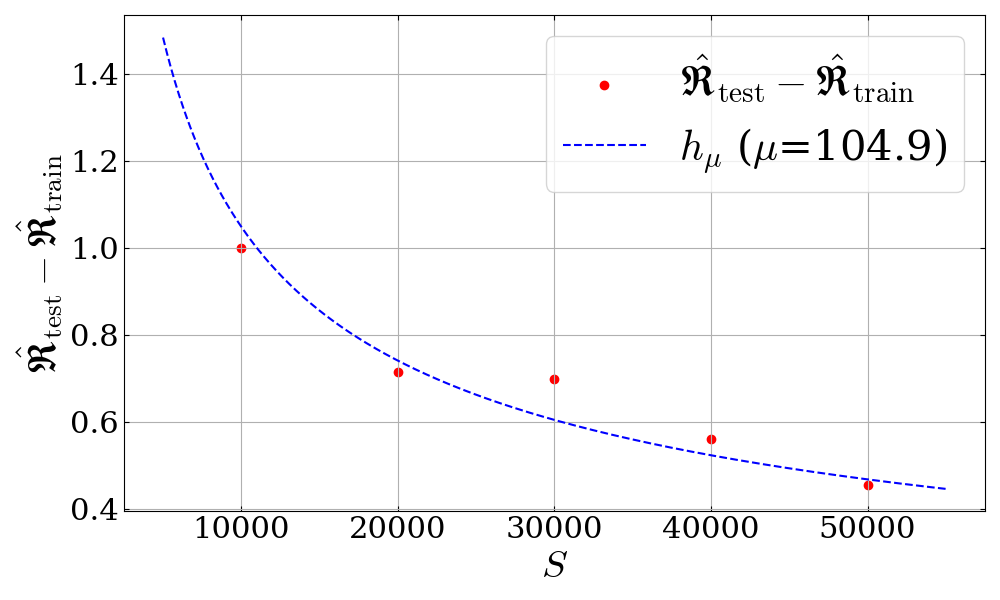}
			\\ \small (b3)  $(T, L)$=(8, 8)
		\end{minipage} &
		
		\begin{minipage}{0.22\textwidth}
			\centering
			\includegraphics[width=\linewidth]{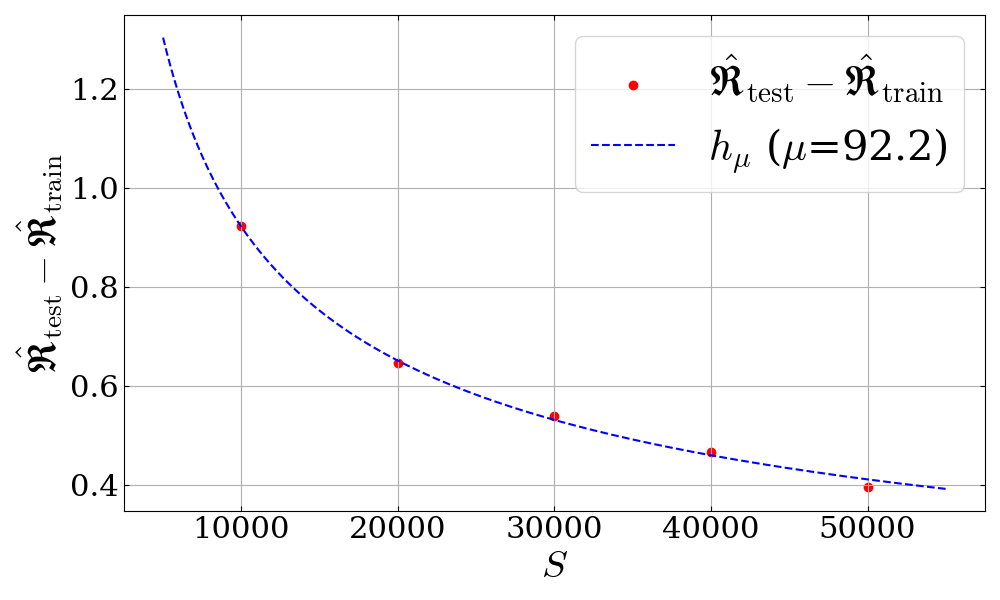}
			\\ \small (b4)  $(T, L)$=(8, 32)
		\end{minipage} &
		
		\\
		& & & \\ 
		
		\begin{minipage}{0.22\textwidth}
			\centering
			\includegraphics[width=\linewidth]{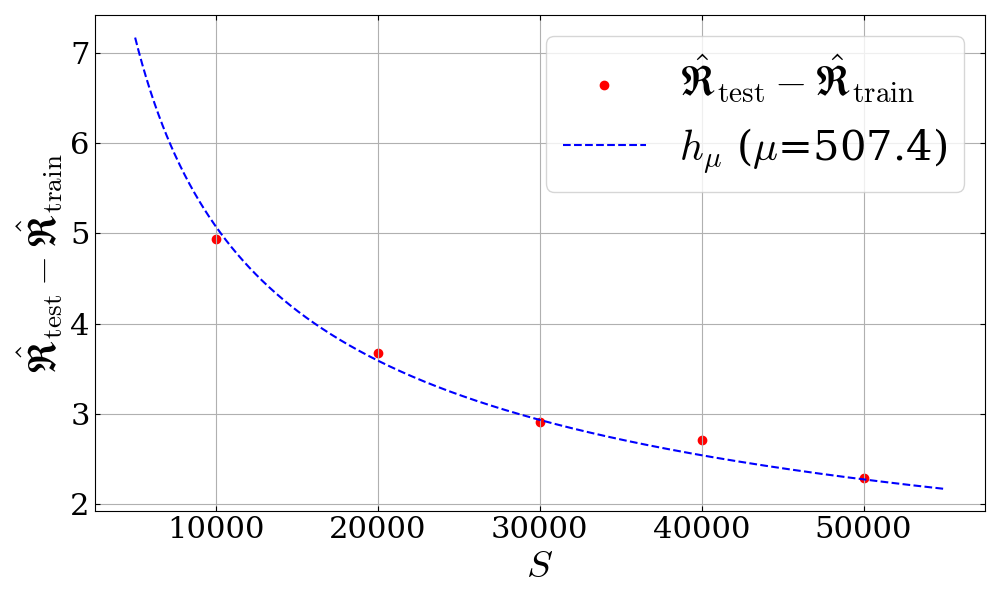}
			\\ \small (c1)  $(T, L)$=(6, 6)
		\end{minipage} &
		
		\begin{minipage}{0.22\textwidth}
			\centering
			\includegraphics[width=\linewidth]{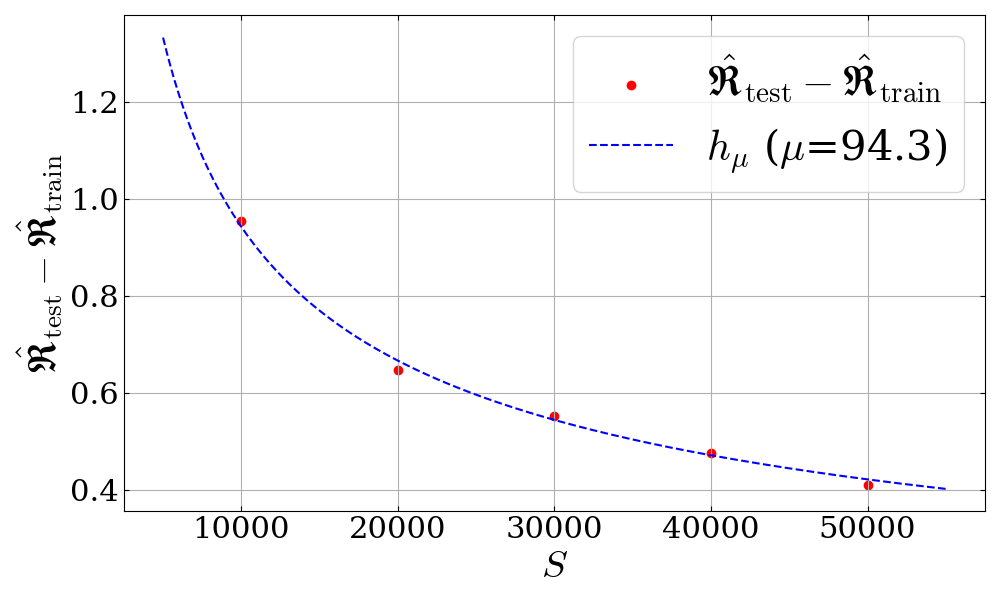}
			\\ \small  (c2)  $(T, L)$=(6, 24)
		\end{minipage} 
		
		\begin{minipage}{0.22\textwidth}
			\centering
			\includegraphics[width=\linewidth]{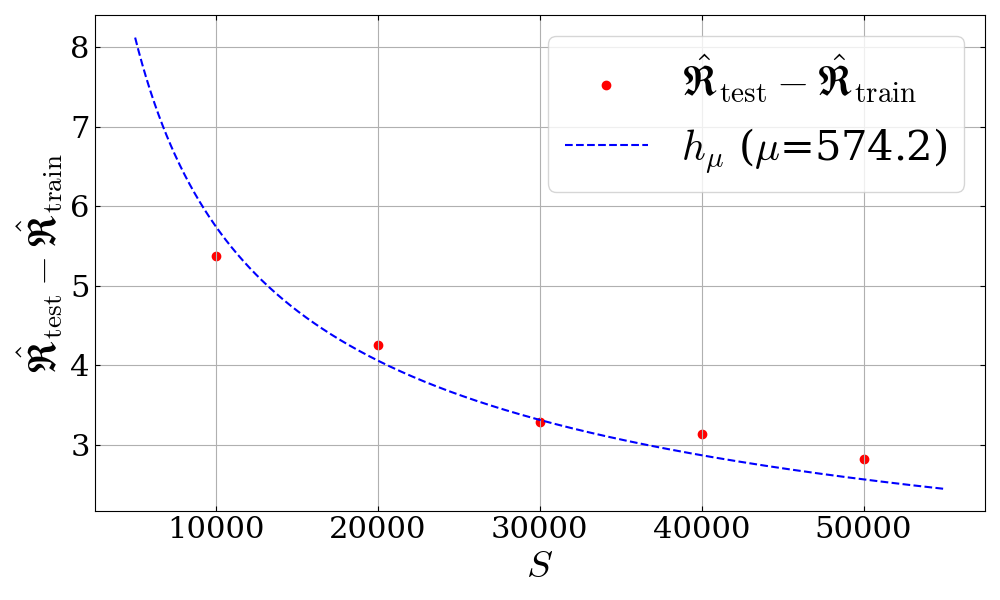}
			\\ \small (c3)  $(T, L)$=(8, 8)
		\end{minipage} &
		
		\begin{minipage}{0.22\textwidth}
			\centering
			\includegraphics[width=\linewidth]{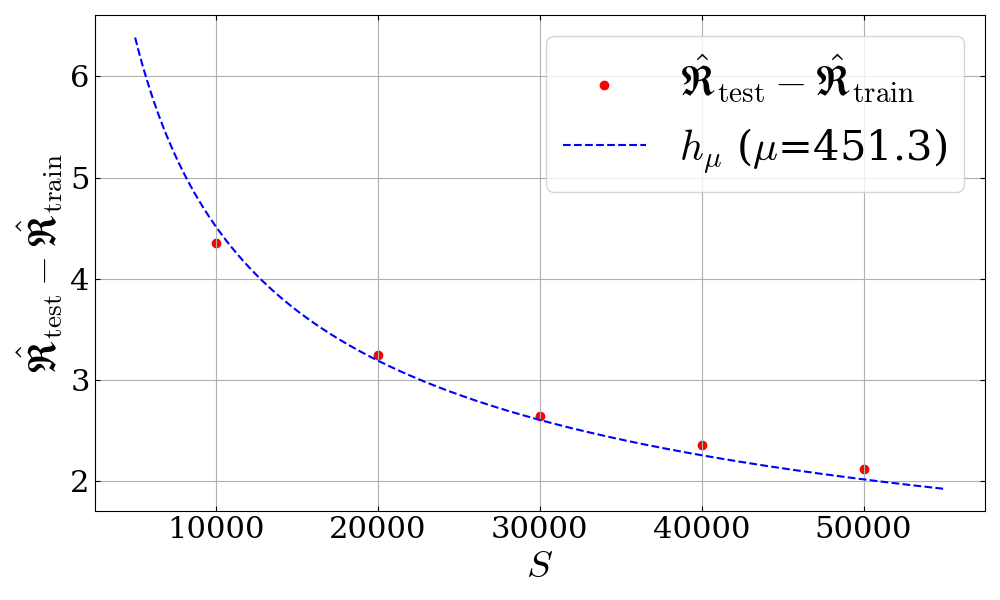}
			\\ \small (c4)  $(T, L)$=(8, 32)
		\end{minipage} &
		
	\end{tabular}
	
	\caption{Average empirical generalization gap of ResNets under multiple $(T, L)$ at last 10 epochs versus training sample size $S$, along with least-squares fitting using $h_{\mu}(S) = \mu / \sqrt{S}$.
		(a1)-(a4) show results of ResNets on SVHN.
		(b1)-(b4) show results of ResNets on CIFAR10.
		(c1)-(c4) show results of ResNets on CIFAR100.
		The empirical generalization gap decreases with increasing $S$, closely matching the theoretical order of $O(1/\sqrt{S})$, especially on the CIFAR10 and CIFAR100 datasets.
	}
	\label{fig: test loss and predicted func}
\end{figure}

\subsection{A test on the sample complexity order} 
We now test the sample complexity order of the generalization error of ResNets. 
Figure~\ref{fig: test loss and predicted func} shows the average empirical generalization gap of ResNets with different $(T, L)$ configurations at the last 10 epochs for them trained on the SVHN, CIFAR10, and CIFAR100 datasets.
As suggested by our $O(1/\sqrt{S})$ generalization bound order, we fit the empirical gap using the least squares method with function ${h}_{\mu}(S) = \mu/\sqrt{S}$. The fitting curves exhibit close agreement with our theoretical bound, and confirm that the generalization error decreases via $O(1/\sqrt{S})$ rate as the number of training samples increases.
This brings the increasing final test accuracies of trained ResNets with increasing $S$ on three datasets, as summarized in Table~\ref{tab: test acc}.

\subsection{A test on the effect of activation functions with structured decomposition}
We now investigate the effect of the decompositional structure of activation functions on the generalization gap of ResNets.
We consider a class of activation functions that admit Definition \ref{definition: activation func} as follows:
\begin{equation}
	\psi_{a,b; \alpha, \beta}(x) =
	\begin{cases}
		a(x - \alpha), & x \ge \alpha, \\
		0,  &  -\beta < x < \alpha,         \\
		b (x + \beta), & x \leq -\beta,
	\end{cases}
	\label{equation: ac_compare}
\end{equation}
where $\alpha, \beta\ge 0$. Specifically, we choose two representative instances of \eqref{equation: ac_compare} for our test: $\psi_{1,0.05;0,0}(\cdot)$ which is a standard leaky ReLU, and $\psi_{1,0.05;\alpha,\beta}(\cdot)$ with $\alpha$ and $\beta$ learnable as suggested in Remark~\ref{remark: ac func}. To enforce the constraint $\alpha,\beta\ge 0$, we parameterize them as
$\alpha=\operatorname{softplus}(\tilde{\alpha}),
\beta=\operatorname{softplus}(\tilde{\beta})$,
where $\tilde{\alpha}$ and $\tilde{\beta}$ are unconstrained trainable parameters.
To remain consistent with the theoretical setting, these parameters $\alpha$ and $\beta$ are shared across all layers, i.e., the same activation function is used throughout the network.


\begin{figure}[htbp]
	\centering
	
	\begin{tabular}{cccc}
		
		\begin{minipage}{0.22\textwidth}
			\centering
			\includegraphics[width=\linewidth]{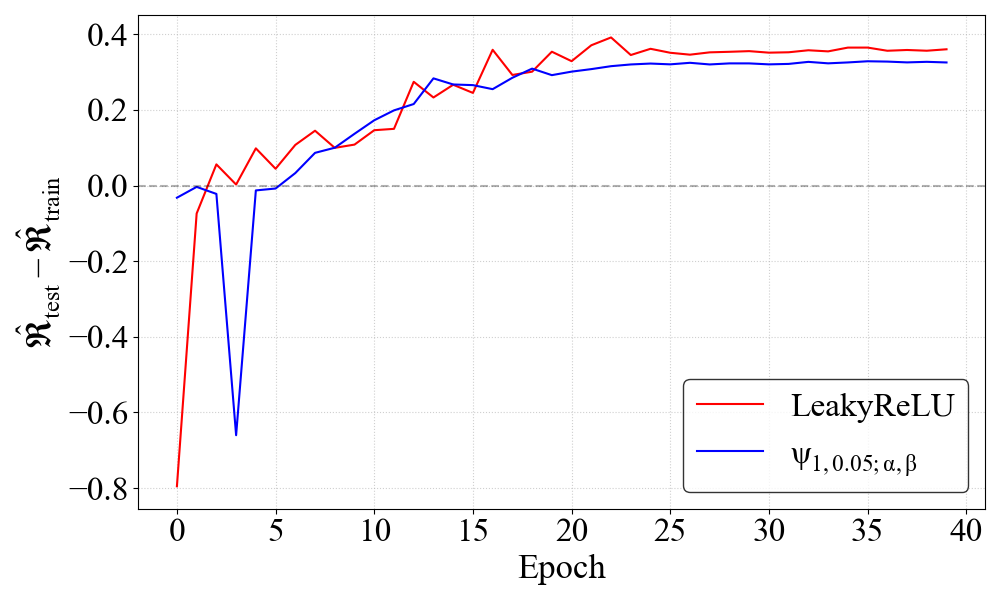}
			\\ \small (a1) $(T, L)$=(6, 6)
		\end{minipage} &
		
		\begin{minipage}{0.22\textwidth}
			\centering
			\includegraphics[width=\linewidth]{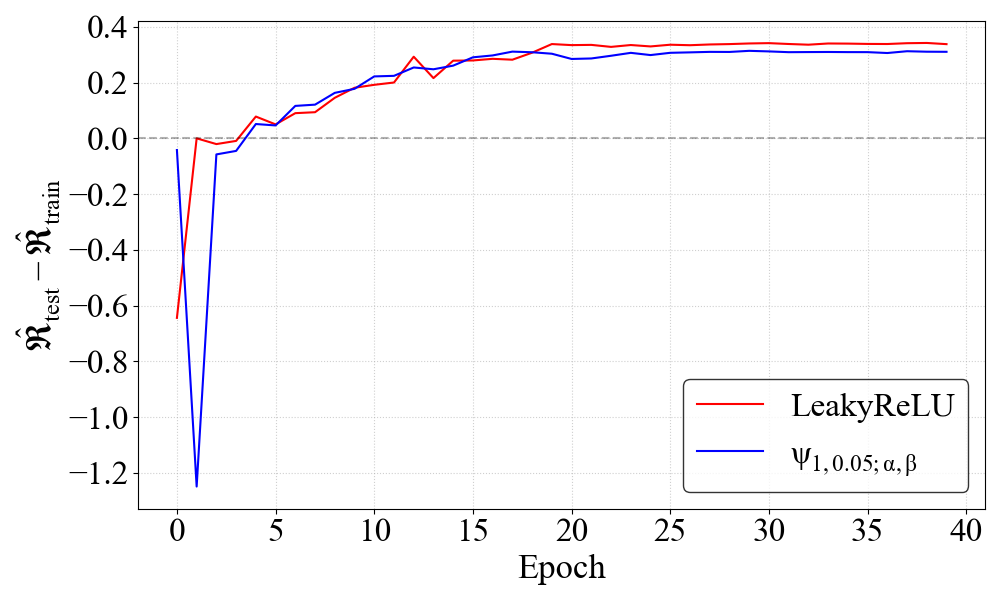}
			\\ \small  (a2)  $(T, L)$=(6, 24)
		\end{minipage} 
		
		\begin{minipage}{0.22\textwidth}
			\centering
			\includegraphics[width=\linewidth]{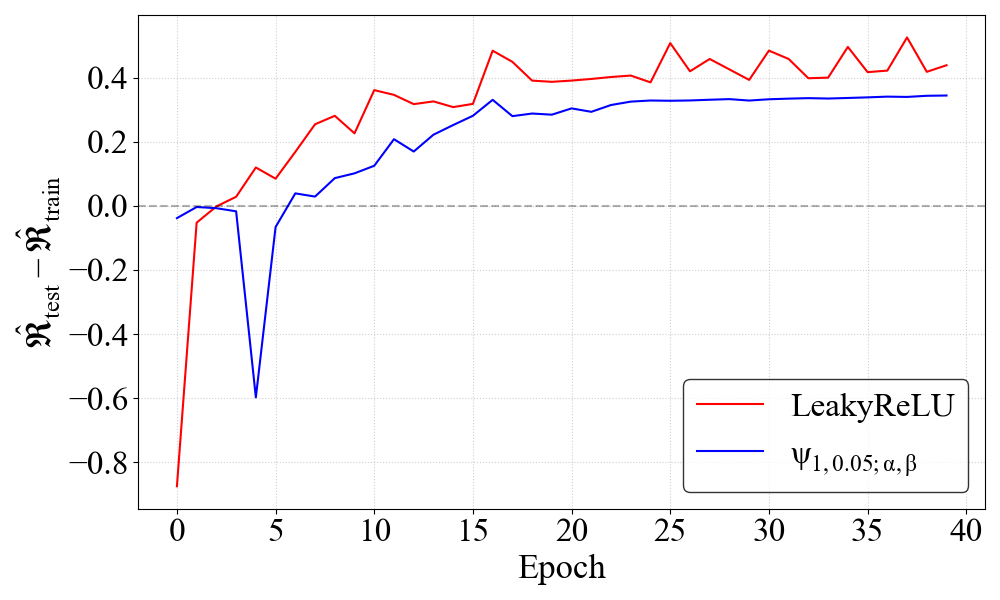}
			\\ \small  (a3) $(T, L)$=(8, 8)
		\end{minipage} &
		
		\begin{minipage}{0.22\textwidth}
			\centering
			\includegraphics[width=\linewidth]{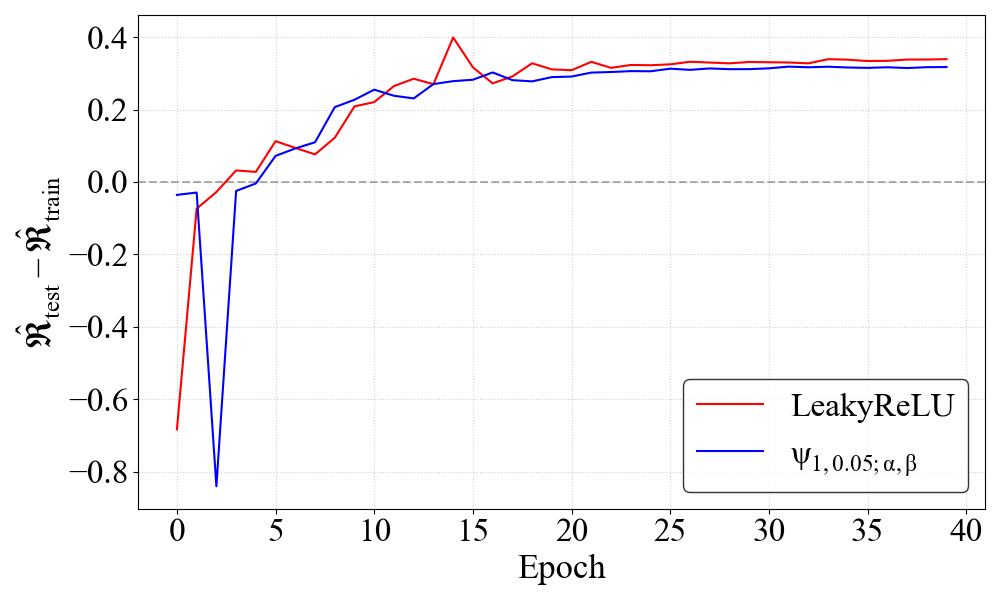}
			\\ \small  (a4)  $(T, L)$=(8, 32)
		\end{minipage} &
		
		\\
		& & & \\ 
		
		\begin{minipage}{0.22\textwidth}
			\centering
			\includegraphics[width=\linewidth]{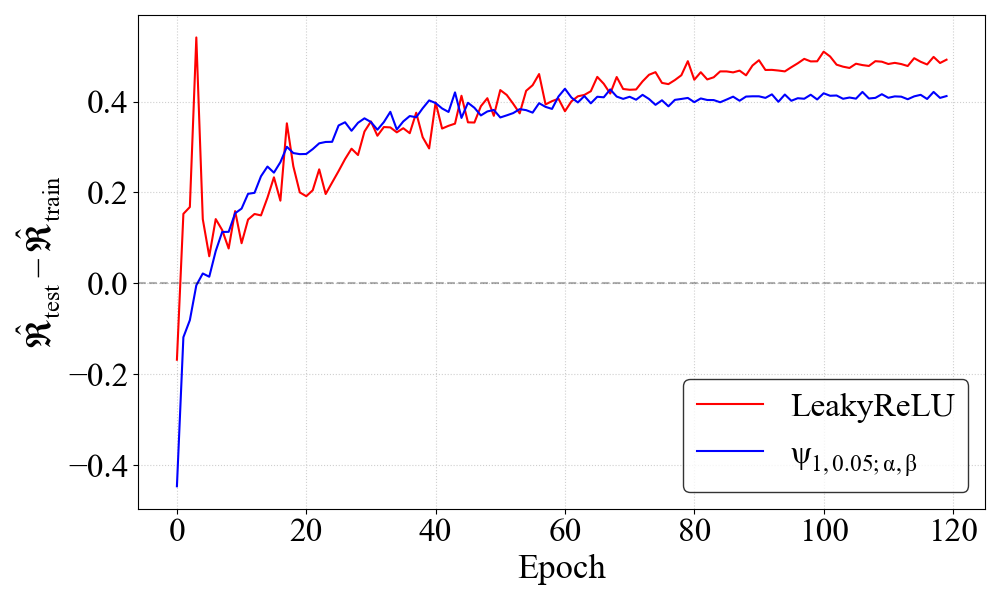}
			\\ \small (b1)  $(T, L)$=(6, 6)
		\end{minipage} &
		
		\begin{minipage}{0.22\textwidth}
			\centering
			\includegraphics[width=\linewidth]{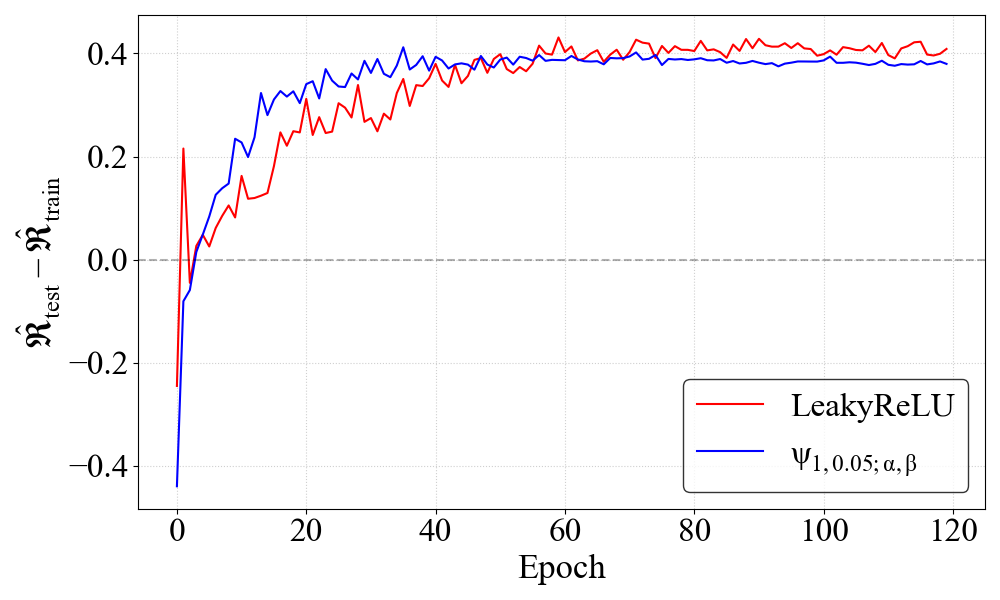}
			\\ \small  (b2) $(T, L)$=(6, 24)
		\end{minipage} 
		
		\begin{minipage}{0.22\textwidth}
			\centering
			\includegraphics[width=\linewidth]{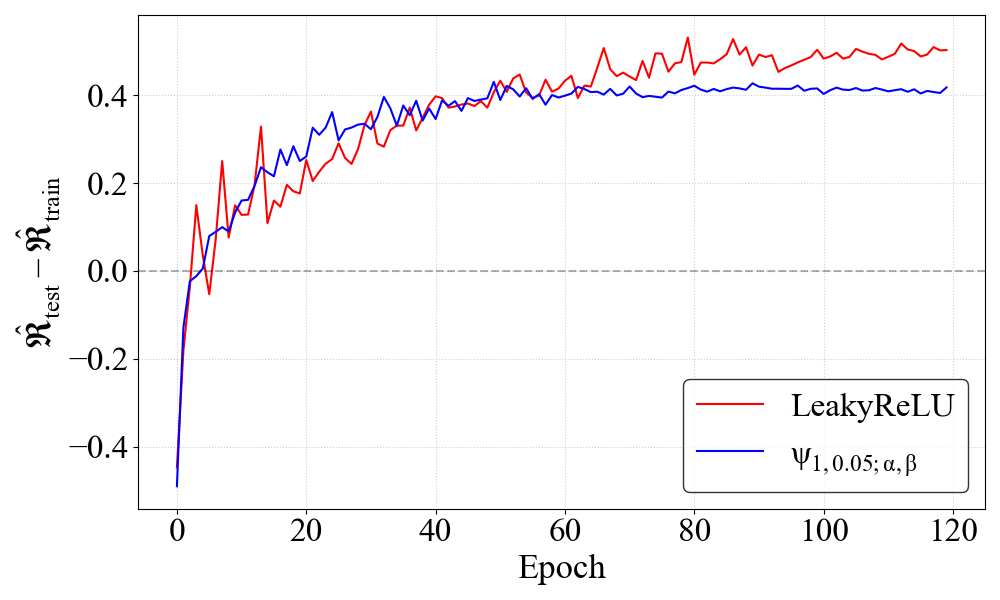}
			\\ \small (b3)  $(T, L)$=(8, 8)
		\end{minipage} &
		
		\begin{minipage}{0.22\textwidth}
			\centering
			\includegraphics[width=\linewidth]{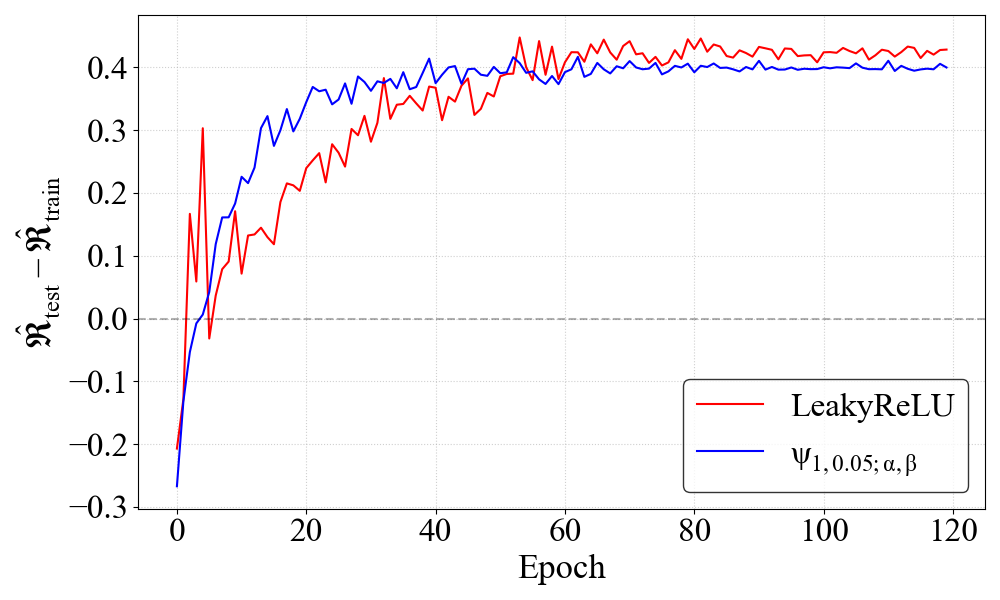}
			\\ \small (b4)  $(T, L)$=(8, 32)
		\end{minipage} &
		
		\\
		& & & \\ 
		
		\begin{minipage}{0.22\textwidth}
			\centering
			\includegraphics[width=\linewidth]{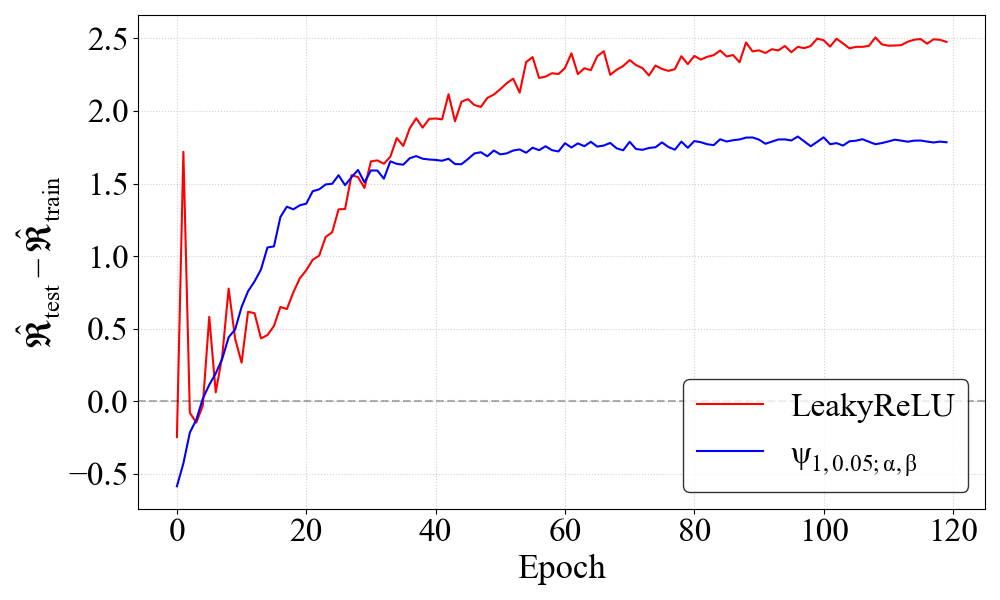}
			\\ \small (c1)  $(T, L)$=(6, 6)
		\end{minipage} &
		
		\begin{minipage}{0.22\textwidth}
			\centering
			\includegraphics[width=\linewidth]{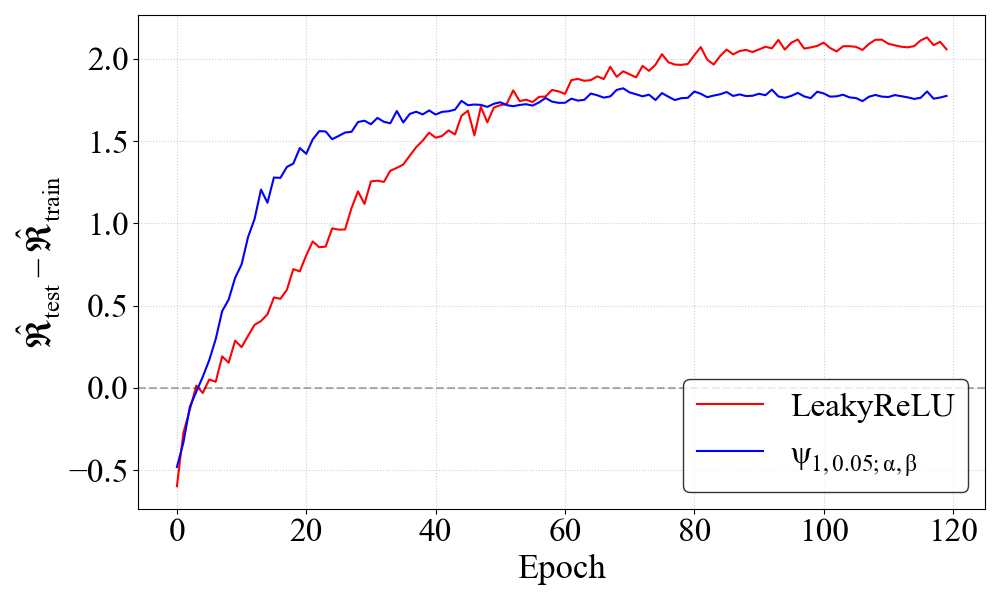}
			\\ \small  (c2)  $(T, L)$=(6, 24)
		\end{minipage} 
		
		\begin{minipage}{0.22\textwidth}
			\centering
			\includegraphics[width=\linewidth]{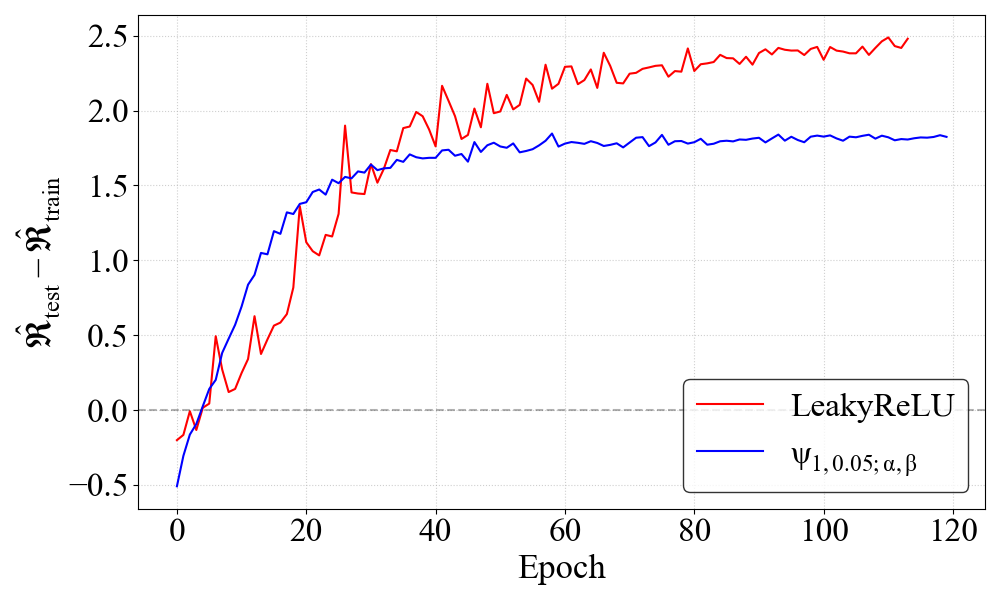}
			\\ \small (c3)  $(T, L)$=(8, 8)
		\end{minipage} &
		
		\begin{minipage}{0.22\textwidth}
			\centering
			\includegraphics[width=\linewidth]{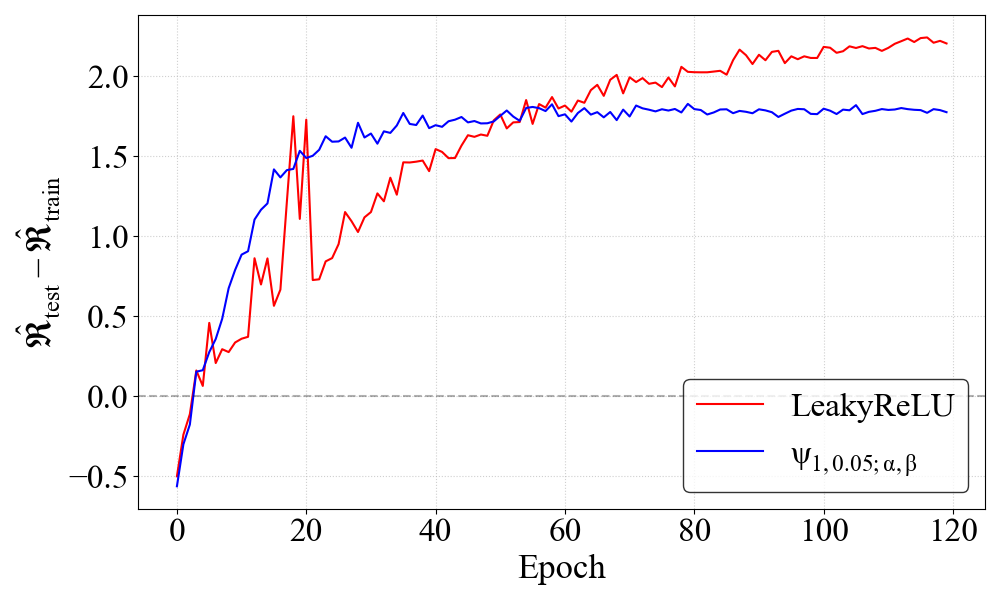}
			\\ \small (c4)  $(T, L)$=(8, 32)
		\end{minipage} &
		
	\end{tabular}
	
	\caption{A comparison on empirical generalization gap $\hat{\mathfrak{R}}_{\rm test}-\hat{\mathfrak{R}}_{\rm train}$ versus training epoch for ResNets with $\psi_{1,0.05;0,0}(\cdot)$ and $\psi_{1,0.05;\alpha,\beta}(\cdot)$ activation functions under multiple $(T,L)$ configurations. The first, second, and third rows show curves of ResNets on SVHN, CIFAR10, and CIFAR100, respectively. The ResNets equipped with activation functions admitting a structured decomposition exhibit smaller generalization gaps on three datasets during the relatively stable period of training.}
	\label{fig:activation_generalization_gap}
\end{figure}

	
	

Figure~\ref{fig:activation_generalization_gap} reports the empirical generalization gaps for ResNets with $\psi_{1,0.05;0,0}(\cdot)$ and $\psi_{1,0.05;\alpha,\beta}(\cdot)$ activation functions under multiple $(T,L)$ configurations on SVHN, CIFAR10, CIFAR100 datasets.
We observe that, during the relatively stable period of training, ResNets equipped with the activation function $\psi_{1,0.05;\alpha,\beta}(\cdot)$ tend to yield smaller generalization gaps compared to the standard leaky ReLU on three datasets.
This verifies empirically the existence of the non-positive structural refinement term in our generalization error bounds.
Our structural analysis on activation functions indeed captures their more sophisticated features in generalization error estimation than the existing well-known contraction inequality method based on the Lipschitz constant.

\section{Conclusion}
\label{Sec: 6}
In this work, we investigated the generalization properties of deep residual networks with a broad class of activation functions from a dynamical system perspective. By developing an equivalence of component hypothesis sets, 
a refined contraction inequality for Rademacher complexity of composed function classes,
and a forward convergence property,
we established compatible uniform generalization error bounds for discrete-time ResNet and its continuous-time counterpart arising in the deep-layer limit, respectively. 
They both consist of a common layer-number independent leading term, and similar non-positive refinement terms from the structure of the activation function, the latter of which are benefited from our contraction inequality derived for the broad class of activation functions $\mathscr{A}(\mathbb{R})$. {In contrast to prior studies that also consider both discrete- and continuous-time ResNets, our unified dynamical system approach enables us to achieve better compatible generalization bounds for both time settings under milder learnable parameter conditions.}
Numerical experiments verified our theoretical results.
Our analysis thus provides new insights into the consistency of the generalization behaviors of ResNets in the deep-layer limit.

\appendix
\section{Proof of Example~\ref{proposition: example of RC} }
\label{sec: appendix-1}

\begin{proof}[Proof of Example~\ref{proposition: example of RC}]
	By the definition of Rademacher complexity and the property of samples $\mathcal{Z}$, we have
	\begin{equation}
		\mathscr{R}_{\mathcal{Z}}(\mathcal{G})  
		=  \frac{1}{S}\mathbb{E}_{\boldsymbol{\varepsilon}} \sup_{g \in \mathcal{G}} \Big[ 
		\sum_{s=1}^{S} \varepsilon_{(s)} g (\mathrm{z}_{(s)})
		\Big] 
		=  \frac{1}{S}\mathbb{E}_{\boldsymbol{\varepsilon}} \sup_{(c_1,c_2)} \Big[(c_1+c_2)
		\sum_{s=1}^{S} \varepsilon_{(s)}\Big].\\
		\nonumber
	\end{equation}
	For any given $\boldsymbol{\varepsilon}$, we estimate 
	$\sup_{(c_1,c_2)} \Big[(c_1+c_2)
	\sum_{s=1}^{S} \varepsilon_{(s)}\Big]$. Assume that $\boldsymbol{\varepsilon} = \{\varepsilon_{(1)}, \ldots,  \varepsilon_{(S)}\}$ contains $i$ positive and $S-i$ negative elements. If $i<S-i$, we have  $$ \sup_{(c_1,c_2)} \Big[(c_1+c_2)\sum_{s=1}^{S} \varepsilon_{(s)}\Big] = \Big[ (c_1+c_2)\sum_{s=1}^{S} \varepsilon_{(s)} \Big] \Big |_{(c_1, c_2) = (- \eta, - \gamma)}= (S-2i)(\eta+\gamma).$$ 
	If $i\geq S-i$, we similarly get $$ \sup_{(c_1,c_2)} \Big[(c_1+c_2)\sum_{s=1}^{S} \varepsilon_{(s)}\Big]  = \Big[(c_1+c_2)\sum_{s=1}^{S} \varepsilon_{(s)} \Big] \Big |_{(c_1, c_2) = (\eta, \gamma)}= (2i-S)(\eta+\gamma).$$ 
	Thus, a direct calculation yields
	\begin{equation}
		\begin{aligned}
			& \mathscr{R}_{\mathcal{Z}}(\mathcal{G}) \\
			= & \frac{1}{S} \frac{1}{2^S} (\eta+\gamma) \sum_{i=0}^S\binom{S}{i}\big|S - 2i\big| \\
			= & \frac{1}{S} \frac{1}{2^S} (\eta+\gamma)  \left\{ \begin{aligned}
				& \sum_{i=0}^k \binom{2k}{i}(2k-2i) + \sum_{i=k+1}^{2k} \binom{2k}{i}(2i-2k),  \text{ if } S = 2k, \\
				& \sum_{i=0}^{k} \binom{2k+1}{i} (2k+1-2i) \! + \!\sum_{i=k+1}^{2k+1} \binom{2k+1}{i} (2i-2k-1), \text{ if } S = 2k+1,
			\end{aligned}\right. \\
			= & \frac{1}{S} \frac{1}{2^S} (\eta+\gamma)  \left\{ \begin{aligned}
				& 2 \sum_{i=0}^{k-1} \binom{2k}{i} (2k - 2i) , \text{ if } S = 2k, \\
				& 2 \sum_{i=0}^k \binom{2k+1}{i} (2k+1 - 2i),\text{ if } S = 2k+1,
			\end{aligned}\right. \\
			= & \frac{1}{S} \frac{1}{2^S} (\eta+\gamma) \cdot 2S \cdot \binom{S - 1}{\lfloor \frac{S}{2} \rfloor} = \frac{1}{2^{S-1}} (\eta+\gamma)\binom{S - 1}{\lfloor \frac{S}{2} \rfloor}.
		\end{aligned}
		\nonumber
	\end{equation}
	
	For the activation function $\psi$, we similarly derive
	$$
	\mathscr{R}_{\mathcal{Z}}(\psi \circ \mathcal{G})  
	=  \frac{1}{S}\mathbb{E}_{\boldsymbol{\varepsilon}} \sup_{g \in \mathcal{G}} \Big[ 
	\sum_{s=1}^{S} \varepsilon_{(s)} \psi \circ g (\mathrm{z}_{(s)})
	\Big] 
	=  \frac{1}{S}\mathbb{E}_{\boldsymbol{\varepsilon}} \sup_{(c_1,c_2)} \Big[\psi(c_1+c_2)
	\sum_{s=1}^{S} \varepsilon_{(s)}\Big].
	$$
	We consider the case that $\boldsymbol{\varepsilon}$ contains $i$ positive and $S-i$ negative elements. If $i<S-i$, we have, 
	$$ \sup_{(c_1,c_2)} \Big[\psi(c_1+c_2)\sum_{s=1}^{S} \varepsilon_{(s)}\Big] = \Big[ \psi(c_1+c_2)\sum_{s=1}^{S} \varepsilon_{(s)}\Big] \Big|_{(c_1 ,c_2)=(- \eta, - \gamma )} = (S-2i)(\eta+\gamma-\beta).$$
	While $i\geq S-i$, we have 
	$$ \sup_{(c_1,c_2)} \Big[\psi(c_1+c_2)\sum_{s=1}^{S} \varepsilon_{(s)}\Big] = \Big[\psi(c_1+c_2)\sum_{s=1}^{S} \varepsilon_{(s)}\Big] \Big|_{(c_1 ,c_2)=( \eta, \gamma )} = (2i-S)(\eta+\gamma-\alpha).$$ Thus, a direct calculation gives
	\begin{equation}
		\begin{aligned}
			& \mathscr{R}_{\mathcal{Z}}(\psi \circ\mathcal{G}) \\
			= & \frac{1}{S} \frac{1}{2^S} \cdot \left\{ \begin{aligned}
				& (\eta + \gamma - \beta) \sum_{i=0}^k \binom{2k}{i}(2k-2i) \!+ \!(\eta \!+\! \gamma \!- \!\alpha) \sum_{i=k+1}^{2k} \binom{2k}{i}(2i-2k), \text{ if } S = 2k \\
				& (\eta + \gamma - \beta) \sum_{i=0}^{k} \binom{2k+1}{i} (2k+1-2i) \\
				& \qquad \qquad \qquad+ (\eta + \gamma - \alpha) \sum_{i=k+1}^{2k+1} \binom{2k+1}{i} (2i-2k-1), \text{ if } S = 2k+1
			\end{aligned}\right. \\
			= & \frac{1}{S} \frac{1}{2^S} (2\eta+2\gamma-\beta-\alpha) \cdot \left\{ \begin{aligned}
				&  \sum_{i=0}^{k-1} \binom{2k}{i} (2k - 2i), \text{ if } S = 2k \\
				&  \sum_{i=0}^k \binom{2k+1}{i} (2k+1 - 2i), \text{ if } S = 2k+1
			\end{aligned}\right. \\
			= & \frac{1}{2^{S}} (2\eta+2\gamma-\beta-\alpha)\binom{S - 1}{\lfloor \frac{S}{2} \rfloor} ,
		\end{aligned}
		\nonumber
	\end{equation}
	which completes the proof.
\end{proof}

\section{Proof of Lemma~\ref{lemma: dis-ResNet-state-bound}}
\label{sec: appendix-2}

\begin{proof}[Proof of Lemma~\ref{lemma: dis-ResNet-state-bound}]
	The existence and uniqueness of the solution of discrete-time and continuous-time ResNets are standard results. We therefore omit the proof.
	
	Next, we estimate the bounds of state variables. For ${\Theta}^{\rm total}_L \in \Omega_{{\Theta}^{\rm total}_L}, \pmb{\Theta}^{\rm total} \in \Omega_{\pmb{\Theta}^{\rm total}}$ and $\mathrm{d} \in \mathbb{R}^{n_{\rm d}}$, let $\mathrm{x}_L^{l} := \mathrm{x}^{l}(\mathrm{d}; {\Theta}^{\rm total}_L), \bold{x}(t) := \bold{x}(t; \mathrm{d}; \pmb{\Theta}^{\rm total})$. By Eq.(\ref{equation: discrete-time ResNet}), Eq.(\ref{equation: continuous-time ResNet}) and the assumptions $(A_1)-(A_3)$, we get
	\begin{equation}
		\begin{aligned}
			\|\mathrm{x}^l_L\|_{\infty} \leq &
			\|\mathrm{x}^{l-1}_L\|_{\infty} + \tau_L \Big[ \|{\Theta}^{\rm total}_L\|_{\infty}  \mathrm{Lip}_{\psi} (\|{\Theta}^{\rm total}_L\|_{\infty}\|\mathrm{x}^{l-1}_L\|_{\infty} + \|{\Theta}^{\rm total}_L\|_{\infty}) + \|{\Theta}^{\rm total}_L\|_{\infty} \Big]
			\\
			\leq  & ( 1+ \tau_L  \mathrm{Lip}_{\psi} \|{\Theta}^{\rm total}_L\|_{\infty}^2 ) \|\mathrm{x}^{l-1}_L\|_{\infty} + \tau_L (  \mathrm{Lip}_{\psi} \|{\Theta}^{\rm total}_L\|_{\infty}^2 +  \|{\Theta}^{\rm total}_L\|_{\infty}) \\
			\leq  & ( 1+ \tau_L  \mathrm{Lip}_{\psi} B_{\pmb{\Theta}}^2 ) \|\mathrm{x}^{l-1}_L\|_{\infty} + \tau_L (  \mathrm{Lip}_{\psi} B_{\pmb{\Theta}}^2 +  B_{\pmb{\Theta}}),  \\
			\|\bold{x}(t)\|_{\infty} \leq & \|\bold{x}(0)\|_{\infty} + \int_{0}^{t} \| \bold{W}(s)\psi \circ (\bold{V}(s)\bold{x}(s) + \bold{b}(s)) + \bold{c}(s) \|_{\infty} \mathrm{d}s \\
			\leq  & \|\bold{x}(0)\|_{\infty} + T( \mathrm{Lip}_{\psi}B_{\pmb{\Theta}}^2 + B_{\pmb{\Theta}}) +   \mathrm{Lip}_{\psi}B_{\pmb{\Theta}}^2 \int_{0}^{t}  \|\bold{x}(s)\|_{\infty}  \mathrm{d}s ,
		\end{aligned}
		\nonumber
	\end{equation}
	where $ \mathrm{Lip}_{\psi}$ is the Lipschitz constant of ${\psi}$. Noting that $\|\mathrm{x}^0_L\|_{\infty} = \|\bold{x}(0)\|_{\infty} \leq  \|\psi \circ (\mathrm{U}\mathrm{d} + \mathrm{a})\|_{\infty} \leq  \mathrm{Lip}_{\psi} B_{\pmb{\Theta}} ( B_{\mathfrak{D}} + 1)$, we then apply Gronwall's inequalities \cite{emmrich1999discrete, dragomir2002some} to obtain the desired bounds.
\end{proof}


\section*{Acknowledgments}
This work was partially supported by the National Natural Science Foundation of China (grant 12271273) and the Key Program (21JCZDJC00220) of the Natural Science Foundation of Tianjin, China.

\bibliographystyle{siamplain}
\bibliography{references}

@article{bartlett2002rademacher,
  title={Rademacher and Gaussian complexities: Risk bounds and structural results},
  author={Bartlett, Peter L and Mendelson, Shahar},
  journal={Journal of Machine Learning Research},
  volume={3},
  number={Nov},
  pages={463--482},
  year={2002}
}

@article{Zhou2025Tighter,
  title={A tighter generalization error bound for wide GCN based on loss landscape},
  author={Xianchen Zhou and Kun Hu and Hongxia Wang},
  journal={Applied and Computational Harmonic Analysis},
  year={2025},
  doi = {https://doi.org/10.1016/j.acha.2025.101777}
}

@article{huang2024on,
year = {2024},
month = {jun},
publisher = {IOP Publishing},
volume = {40},
number = {7},
pages = {075006},
author = {Jinshu Huang and Yiming Gao and Chunlin Wu},
title = {On dynamical system modeling of learned primal-dual with a linear operator $\mathcal{K}$: stability and convergence properties},
journal = {Inverse Problems}
}

@inproceedings{shultzman2023generalization,
  title={Generalization and estimation error bounds for model-based neural networks},
  author={Shultzman, Avner and Azar, Eyar and Rodrigues, Miguel RD and Eldar, Yonina C},
  booktitle = {International Conference on Learning Representations},
  year      = {2023}
}

@book{beck2017first,
  title={First-order methods in optimization},
  author={Beck, Amir},
  year={2017},
  publisher={SIAM}
}

@article{clevert2015fast,
  title={Fast and accurate deep network learning by exponential linear units (elus)},
  author={Clevert, Djork-Arn{\'e}},
  journal={arXiv preprint arXiv:1511.07289},
  year={2015}
}

@article{thorpe2018deep,
  title={Deep limits of residual neural networks},
  author={Thorpe, Matthew and van Gennip, Yves},
  journal={Research in the Mathematical Sciences},
  volume={10},
  number={1},
  pages={6},
  year={2023},
  publisher={Springer}
}

@book{shalev2014understanding,
  title={Understanding machine learning: From theory to algorithms},
  author={Shalev-Shwartz, Shai and Ben-David, Shai},
  year={2014},
  publisher={Cambridge University Press}
}

@inproceedings{chen2020understanding,
  title={Understanding deep architecture with reasoning layer},
  author={Chen, Xinshi and Zhang, Yufei and Reisinger, Christoph and Song, Le},
  booktitle={Advances in Neural Information Processing Systems},
  year={2020}
}

@inproceedings{kobler2020total,
  title={Total deep variation for linear inverse problems},
  author={Kobler, Erich and Effland, Alexander and Kunisch, Karl and Pock, Thomas},
  booktitle={IEEE Conference on Computer Vision and Pattern Recognition},
  pages={7549--7558},
  year={2020}
}

@book{leoni2017first,
  title={A first course in Sobolev spaces},
  author={Leoni, Giovanni},
  year={2017},
  publisher={American Mathematical Soc.}
}

@book{goodfellow2016deep,
  title={Deep learning},
  author={Goodfellow, Ian and Bengio, Yoshua and Courville, Aaron},
  year={2016},
  publisher={MIT press}
}

@article{lecun2015deep,
  title={Deep learning},
  author={LeCun, Yann and Bengio, Yoshua and Hinton, Geoffrey},
  journal={Nature},
  volume={521},
  number={7553},
  pages={436--444},
  year={2015},
  publisher={Nature Publishing Group}
}

@article{zhang2020forward,
  title={Forward stability of ResNet and its variants},
  author={Zhang, Linan and Schaeffer, Hayden},
  journal={Journal of Mathematical Imaging and Vision},
  volume={62},
  number={3},
  pages={328--351},
  year={2020},
  publisher={Springer}
}

@article{huang2024Mathematical,
  title={Mathematical Modeling and Convergence Analysis of Deep Neural Networks with Dense Layer Connectivities in Deep Learning},
  author={Huang, Jinshu and Su, Haibin and Tai, Xue-Cheng and Wu, Chunlin},
  journal={arXiv preprint arXiv:2510.02049},
  year={2025}
}

@article{chen2018neural,
  title={Neural ordinary differential equations},
  author={Chen, Ricky TQ and Rubanova, Yulia and Bettencourt, Jesse and Duvenaud, David K},
  journal={Advances in Neural Information Processing Systems},
  volume={31},
  year={2018}
}

@article{sherstinsky2020fundamentals,
  title={Fundamentals of recurrent neural network (RNN) and long short-term memory (LSTM) network},
  author={Sherstinsky, Alex},
  journal={Physica D: Nonlinear Phenomena},
  volume={404},
  pages={132306},
  year={2020},
  publisher={Elsevier}
}

@article{weinan2017proposal,
  title={A proposal on machine learning via dynamical systems},
  author={E, Weinan},
  journal={Communications in Mathematics and Statistics},
  volume={1},
  number={5},
  pages={1--11},
  year={2017}
}

@article{haber2017stable,
  title={Stable architectures for deep neural networks},
  author={Haber, Eldad and Ruthotto, Lars},
  journal={Inverse Problems},
  volume={34},
  number={1},
  pages={014004},
  year={2017},
  publisher={IOP Publishing}
}

@inproceedings{Lu18Beyond,
  author       = {Yiping Lu and
                  Aoxiao Zhong and
                  Quanzheng Li and
                  Bin Dong},
  title        = {Beyond finite layer neural networks: Bridging deep architectures and
                  numerical differential equations},
  booktitle       = {Proceedings of Machine Learning Research},
  volume       = {80},
  pages        = {3282--3291},
  publisher    = {{PMLR}},
  year         = {2018}
}

@article{ruthotto2020deep,
  title={Deep neural networks motivated by partial differential equations},
  author={Ruthotto, Lars and Haber, Eldad},
  journal={Journal of Mathematical Imaging and Vision},
  volume={62},
  number={3},
  pages={352--364},
  year={2020},
  publisher={Springer}
}

@article{bartlett1998almost,
  title={Almost linear VC dimension bounds for piecewise polynomial networks},
  author={Bartlett, Peter and Maiorov, Vitaly and Meir, Ron},
  journal={Advances in Neural Information Processing Systems},
  volume={11},
  year={1998}
}

@inproceedings{neyshabur2015norm,
  title={Norm-based capacity control in neural networks},
  author={Neyshabur, Behnam and Tomioka, Ryota and Srebro, Nathan},
  booktitle={Conference on Learning Theory},
  pages={1376--1401},
  year={2015},
  organization={PMLR}
}

@article{bartlett2017spectrally,
  title={Spectrally-normalized margin bounds for neural networks},
  author={Bartlett, Peter L and Foster, Dylan J and Telgarsky, Matus J},
  journal={Advances in Neural Information Processing Systems},
  volume={30},
  year={2017}
}

@article{schnoor2023generalization, 
  title={Generalization error bounds for iterative recovery algorithms unfolded as neural networks},
  author={Schnoor, Ekkehard and Behboodi, Arash and Rauhut, Holger},
  journal={Information and Inference: A Journal of the IMA},
  volume={12},
  number={3},
  pages={2267--2299},
  year={2023},
  publisher={Oxford University Press}
}

@inproceedings{kuzborskij2018data,
  title={Data-dependent stability of stochastic gradient descent},
  author={Kuzborskij, Ilja and Lampert, Christoph},
  booktitle={International Conference on Machine Learning},
  pages={2815--2824},
  year={2018},
  organization={PMLR}
}

@inproceedings{sokolic2017generalization,
  title={Generalization error of invariant classifiers},
  author={Sokolic, Jure and Giryes, Raja and Sapiro, Guillermo and Rodrigues, Miguel},
  booktitle={Artificial Intelligence and Statistics},
  pages={1094--1103},
  year={2017},
  organization={PMLR}
}

@inproceedings{he2016deep,
  title={Deep residual learning for image recognition},
  author={He, Kaiming and Zhang, Xiangyu and Ren, Shaoqing and Sun, Jian},
  booktitle={Proceedings of the IEEE Conference on Computer Vision and Pattern Recognition},
  pages={770--778},
  year={2016}
}

@book{vershynin2018high,
  title={High-dimensional probability: An introduction with applications in data science},
  author={Vershynin, Roman},
  volume={47},
  year={2018},
  publisher={Cambridge University Press}
}

@book{grohs2022mathematical,
  title={Mathematical aspects of deep learning},
  author={Grohs, Philipp and Kutyniok, Gitta},
  year={2022},
  publisher={Cambridge University Press}
}

@article{li2022approximation,
  title={Approximation and optimization theory for linear continuous-time recurrent neural networks},
  author={Li, Zhong and Han, Jiequn and Weinan, E and Li, Qianxiao},
  journal={Journal of Machine Learning Research},
  volume={23},
  number={42},
  pages={1--85},
  year={2022}
}

@article{li2022deep,
  title={Deep learning via dynamical systems: An approximation perspective},
  author={Li, Qianxiao and Lin, Ting and Shen, Zuowei},
  journal={Journal of the European Mathematical Society},
  volume={25},
  number={5},
  pages={1671--1709},
  year={2022}
}

@article{bousquet2002stability,
  title={Stability and generalization},
  author={Bousquet, Olivier and Elisseeff, Andr{\'e}},
  journal={Journal of Machine Learning Research},
  volume={2},
  number={Mar},
  pages={499--526},
  year={2002}
}

@article{he2020resnet,
  title={Why resnet works? residuals generalize},
  author={He, Fengxiang and Liu, Tongliang and Tao, Dacheng},
  journal={IEEE Transactions on Neural Networks and Learning Systems},
  volume={31},
  number={12},
  pages={5349--5362},
  year={2020},
  publisher={IEEE}
}

@article{li2018tighter,
  title={On tighter generalization bound for deep neural networks: Cnns, resnets, and beyond},
  author={Li, Xingguo and Lu, Junwei and Wang, Zhaoran and Haupt, Jarvis and Zhao, Tuo},
  journal={arXiv preprint arXiv:1806.05159},
  year={2018}
}

@article{gunther2020layer,
  title={Layer-parallel training of deep residual neural networks},
  author={Gunther, Stefanie and Ruthotto, Lars and Schroder, Jacob B and Cyr, Eric C and Gauger, Nicolas R},
  journal={SIAM Journal on Mathematics of Data Science},
  volume={2},
  number={1},
  pages={1--23},
  year={2020},
  publisher={SIAM}
}

@article{parpas2019predict,
author = {Parpas, Panos and Muir, Corey},
title = {Predict globally, correct locally: Parallel-in-time optimization of neural networks},
year = {2025},
issue_date = {Jan 2025},
publisher = {Pergamon Press, Inc.},
address = {USA},
volume = {171},
number = {C},
issn = {0005-1098},
doi = {10.1016/j.automatica.2024.111976},
journal = {Automatica},
month = jan,
numpages = {11}
}

@article{pandey2020overcoming,
  title={Overcoming overfitting and large weight update problem in linear rectifiers: Thresholded exponential rectified linear units},
  author={Pandey, Vijay},
  journal={arXiv preprint arXiv:2006.02797},
  year={2020}
}

@inproceedings{maurer2016vector,
  title={A vector-contraction inequality for rademacher complexities},
  author={Maurer, Andreas},
  booktitle={Algorithmic Learning Theory: 27th International Conference},
  pages={3--17},
  year={2016},
  organization={Springer}
}

@book{emmrich1999discrete,
  title={Discrete versions of Gronwall's lemma and their application to the numerical analysis of parabolic problems},
  author={Emmrich, Etienne},
  year={1999},
  publisher={Techn. Univ.}
}

@book{dragomir2002some,
  title={Some Gronwall Type Inequalities and Applications},
  author={Dragomir, Sever Silvestru},
  isbn={9781590338278},
  lccn={2003016169},
  year={2003},
  publisher={Nova Science Publishers}
}

@article{marion2023generalization,
  title={Generalization bounds for neural ordinary differential equations and deep residual networks},
  author={Marion, Pierre},
  journal={Advances in Neural Information Processing Systems},
  volume={36},
  pages={48918--48938},
  year={2023}
}

@article{benning2019deep,
  title={Deep learning as optimal control problems: Models and numerical methods},
  author={Benning, Martin and Celledoni, Elena and Ehrhardt, Matthias J and Owren, Brynjulf and Sch{\"o}nlieb, Carola-Bibiane},
  journal={Journal of Computational Dynamics},
  volume={6},
  number={2},
  pages={171--198},
  year={2019},
  publisher={American Institute of Mathematical Sciences (AIMS)}
}

@inproceedings{glorot2011deep,
  title={Deep sparse rectifier neural networks},
  author={Glorot, Xavier and Bordes, Antoine and Bengio, Yoshua},
  booktitle={Proceedings of the Fourteenth International Conference on Artificial Intelligence and Statistics},
  pages={315--323},
  year={2011},
  organization={JMLR Workshop and Conference Proceedings}
}

@article{e2020rademacher,
  title={Rademacher complexity and the generalization error of residual networks},
  author={E, Weinan and Ma, Chao and Wang, Qingcan},
  journal={Communications in Mathematical Sciences},
  volume={18},
  number={6},
  pages={1755--1774},
  year={2020},
  publisher={International Press of Boston}
}

@inproceedings{bleistein2023generalization,
  title={On the Generalization Capacities of Neural Controlled Differential Equations},
  author={Bleistein, Linus and Guilloux, Agathe},
  booktitle={Workshop on New Frontiers in Learning, Control, and Dynamical Systems at the International Conference on Machine Learning (ICML 2023)},
  year={2023}
}

@inproceedings{hanson2024rademacher,
  title={Rademacher complexity of neural odes via chen-fliess series},
  author={Hanson, Joshua and Raginsky, Maxim},
  booktitle={6th Annual Learning for Dynamics \& Control Conference},
  pages={758--769},
  year={2024},
  organization={PMLR}
}

@article{Lin2025DeeplayerLA,
  title={Deep-layer limit and stability analysis of the basic forward–backward-splitting induced network (I): feed-forward systems},
  author={Xuan Lin and Chunlin Wu},
  journal={IMA Journal of Numerical Analysis},
  year={2025}
}

@article{truong2022rademacher,
  title={On rademacher complexity-based generalization bounds for deep learning},
  author={Truong, Lan V},
  journal={arXiv preprint arXiv:2208.04284},
  year={2022}
}

@article{vapnik1971uniform,
  title={On the Uniform Convergence of Relative Frequencies of Events to Their Probabilities},
  author={Vapnik, VN and Chervonenkis, A Ya},
  journal={Theory of Probability \& Its Applications},
  volume={16},
  number={2},
  pages={264--280},
  year={1971},
  publisher={Society for Industrial \& Applied Mathematics (SIAM)}
}

@inproceedings{neyshabur2018pac,
  title={A PAC-Bayesian Approach to Spectrally-Normalized Margin Bounds for Neural Networks},
  author={Neyshabur, Behnam and Bhojanapalli, Srinadh and Srebro, Nathan},
  booktitle={International Conference on Learning Representations},
  year={2018}
}

@book{anthony2009neural,
  title={Neural network learning: Theoretical foundations},
  author={Anthony, Martin and Bartlett, Peter L},
  year={2009},
  publisher={Cambridge university press}
}

@inproceedings{mcallester1999pac,
  title={PAC-Bayesian model averaging},
  author={McAllester, David A},
  booktitle={Proceedings of the twelfth annual conference on Computational learning theory},
  pages={164--170},
  year={1999}
}

@article{langford2002pac,
  title={PAC-Bayes \& margins},
  author={Langford, John and Shawe-Taylor, John},
  journal={Advances in neural information processing systems},
  volume={15},
  year={2002}
}

@inproceedings{maas2013rectifier,
  title={Rectifier nonlinearities improve neural network acoustic models},
  author={Maas, Andrew L and Hannun, Awni Y and Ng, Andrew Y and others},
  booktitle={Proc. icml},
  volume={30},
  number={1},
  pages={3},
  year={2013},
  organization={Atlanta, GA}
}
\end{document}